\documentclass[11pt]{article}

\usepackage{natbib}

\title{A Blackbox Approach to Best of Both Worlds in Bandits and Beyond}

\setcitestyle{authoryear,round,citesep={;},aysep={,},yysep={;}}
\usepackage[english]{babel}

\usepackage{titletoc}
\usepackage[toc,page,header]{appendix}

\usepackage[algo2e,ruled,vlined]{algorithm2e}
\usepackage{amsmath}
\usepackage{amssymb}
\usepackage{amsthm}
\usepackage{graphicx}
\usepackage{xcolor}
\usepackage[colorlinks=true, linkcolor=blue, citecolor=blue,urlcolor=black,linktocpage=true]{hyperref}

\usepackage{prettyref}
\newcommand{\pref}[1]{\prettyref{#1}}

\newcommand{\savehyperref}[2]{\texorpdfstring{\hyperref[#1]{#2}}{#2}}
\newrefformat{eq}{\savehyperref{#1}{Eq.~\textup{(\ref*{#1})}}}
\newrefformat{eqn}{\savehyperref{#1}{Equation~\ref*{#1}}}
\newrefformat{lem}{\savehyperref{#1}{Lemma~\ref*{#1}}}
\newrefformat{lemma}{\savehyperref{#1}{Lemma~\ref*{#1}}}
\newrefformat{exp}{\savehyperref{#1}{Example~\ref*{#1}}}
\newrefformat{def}{\savehyperref{#1}{Definition~\ref*{#1}}}
\newrefformat{line}{\savehyperref{#1}{Line~\ref*{#1}}}
\newrefformat{thm}{\savehyperref{#1}{Theorem~\ref*{#1}}}
\newrefformat{corr}{\savehyperref{#1}{Corollary~\ref*{#1}}}
\newrefformat{cor}{\savehyperref{#1}{Corollary~\ref*{#1}}}
\newrefformat{sec}{\savehyperref{#1}{Section~\ref*{#1}}}
\newrefformat{subsec}{\savehyperref{#1}{Section~\ref*{#1}}}
\newrefformat{app}{\savehyperref{#1}{Appendix~\ref*{#1}}}
\newrefformat{assum}{\savehyperref{#1}{Assumption~\ref*{#1}}}
\newrefformat{ex}{\savehyperref{#1}{Example~\ref*{#1}}}
\newrefformat{fig}{\savehyperref{#1}{Figure~\ref*{#1}}}
\newrefformat{alg}{\savehyperref{#1}{Algorithm~\ref*{#1}}}
\newrefformat{rem}{\savehyperref{#1}{Remark~\ref*{#1}}}
\newrefformat{conj}{\savehyperref{#1}{Conjecture~\ref*{#1}}}
\newrefformat{prop}{\savehyperref{#1}{Proposition~\ref*{#1}}}
\newrefformat{proto}{\savehyperref{#1}{Protocol~\ref*{#1}}}
\newrefformat{prob}{\savehyperref{#1}{Problem~\ref*{#1}}}
\newrefformat{claim}{\savehyperref{#1}{Claim~\ref*{#1}}}
\newrefformat{que}{\savehyperref{#1}{Question~\ref*{#1}}}
\newrefformat{op}{\savehyperref{#1}{Open Problem~\ref*{#1}}}
\newrefformat{fn}{\savehyperref{#1}{Footnote~\ref*{#1}}}
\newrefformat{tab}{\savehyperref{#1}{Table~\ref*{#1}}}
\newrefformat{fig}{\savehyperref{#1}{Figure~\ref*{#1}}}
\newrefformat{proc}{\savehyperref{#1}{Procedure~\ref*{#1}}}

\usepackage{xspace}
\usepackage{amsmath}
\usepackage{graphicx}
\usepackage{framed}
\usepackage{bbm}
\usepackage{algorithm}
\usepackage{natbib}
\usepackage{mathrsfs}
\usepackage{amssymb}
\usepackage{bm}
\usepackage{makecell}
\usepackage{chngpage}

\usepackage{tabulary}
\usepackage{xcolor}
\usepackage{prettyref}
\usepackage{mathtools}
\usepackage{amsmath}

\usepackage{multirow}
\usepackage{nicefrac}
\usepackage{amssymb}
\usepackage{pifont}
\definecolor{Green}{rgb}{0.13, 0.65, 0.3}
\usepackage{colortbl}

\newcommand{\barq}{\bar{q}}

\newcommand{\fillcell}{\cellcolor{blue!25}}

\DeclareMathOperator*{\argmin}{argmin} 

\newcommand{\Reg}{\text{\rm Reg}}

\newcommand{\tildepi}{\widetilde{\pi}}

\newcommand{\calT}{\mathcal{T}}
\newcommand{\calG}{\mathcal{G}}

\newcommand{\calA}{\mathcal{A}}
\newcommand{\calB}{\mathcal{B}}
\newcommand{\calE}{\mathcal{E}}
\newcommand{\calZ}{\mathcal{Z}}

\newcommand{\calS}{\mathcal{S}}

\newcommand{\calN}{\mathcal{N}}
\newcommand{\calI}{\mathcal{I}}

\newcommand{\calP}{\mathcal{P}}

\newcommand{\ind}{\mathbb{I}}

\newcommand{\hatpi}{\widehat{\pi}}
\newcommand{\hatw}{\widehat{w}}

\newcommand{\tildep}{\widetilde{p}}

\newcommand{\hatx}{\widehat{x}}

\newcommand{\E}{\mathbb{E}}

\newcommand{\bbR}{\mathbb{R}}

\newcommand{\tildeA}{\widetilde{A}}

\newcommand{\hatell}{\widehat{\ell}}

\newcommand{\norm}[1]{\left\|#1\right\|}

\newcommand{\tr}{\operatorname{Tr}}

\newcommand{\calX}{\mathcal{X}}
\newcommand{\calY}{\mathcal{Y}}
\newcommand{\one}{\mathbf{1}}
\newcommand{\inner}[1]{\langle#1\rangle}

\newcommand{\term}{\textbf{term}}

\newcommand{\baralpha}{\widetilde{\alpha}}

\newcommand{\upd}{\scalebox{0.9}{\textsf{upd}}}

\newcommand{\stab}{\textit{stab}}

\usepackage{tikz}

\newcommand{\varH}[1]{H^{\mathbb{V}}_{#1}}
\newcommand{\varHinv}[1]{[H^{\mathbb{V}}_{#1}]^{-1}}
\newcommand{\varG}[1]{G^{\mathbb{V}}_{#1}}
\newcommand{\varGinv}[1]{[G^{\mathbb{V}}_{#1}]^{+}}
\newcommand{\mH}[1]{\mu_{#1}}
\newcommand{\mG}[1]{m_{#1}}

\newcommand{\nonl}{\renewcommand{\nl}{\let\nl}}

\usepackage{caption}

\newcommand{\bfell}{\bm{\ell}}
\newcommand{\bfhatell}{\bm{\hatell}}
\newcommand{\bfm}{\bm{m}}

\newcommand{\tildex}{\tilde{x}}

\newcommand{\ts}{\scalebox{0.5}{\textsf{\textup{Ts}}}}
\newcommand{\lo}{\scalebox{0.5}{\textsf{\textup{Lo}}}}

\usepackage{fancyhdr}

  \newtheorem{theorem}{Theorem}
  \newtheorem{lemma}[theorem]{Lemma}
  \newtheorem{proposition}[theorem]{Proposition}
  
  \newtheorem{corollary}[theorem]{Corollary}
  \newtheorem{definition}[theorem]{Definition}

\usepackage{times}
\date{}

\author{%
    Christoph Dann~\thanks{Google Research. Email: \texttt{cdann@cdann.net}.}\and Chen-Yu Wei~\thanks{MIT Institute for Data, Systems, and Society. Email: \texttt{chenyuw@mit.edu}.}\and Julian Zimmert~\thanks{Google Research. Email: \texttt{zimmert@google.com}.}
}

\begin{document}

\maketitle

\begin{abstract}%
Best-of-both-worlds algorithms for online learning which achieve near-optimal regret in both the adversarial and the stochastic regimes have received growing attention recently. Existing techniques often require careful adaptation to every new problem setup, including specialised potentials and careful tuning of algorithm parameters. Yet, in domains such as linear bandits, it is still unknown if there exists an algorithm that can simultaneously obtain $O(\log(T))$ regret in the stochastic regime and $\tilde{O}(\sqrt{T})$ regret in the adversarial regime. In this work, we resolve this question positively and present a general reduction from best of both worlds to a wide family of follow-the-regularized-leader (FTRL) and online-mirror-descent (OMD) algorithms. We showcase the capability of this reduction by transforming existing algorithms that are only known to achieve worst-case guarantees into new algorithms with best-of-both-worlds guarantees in contextual bandits, graph bandits and tabular Markov decision processes.
\end{abstract}

\section{Introduction}\label{sec: intro}
Multi-armed bandits and its various extensions have a long history  \citep{lai1985asymptotically,auer2002finite,auer2002nonstochastic}.
Traditionally, the stochastic regime in which all losses/rewards are i.i.d.\ and the adversarial regime in which an adversary chooses the losses in an arbitrary fashion have been studied in isolation.
However, it is often unclear in practice whether an environment is best modelled by the stochastic regime, a slightly contaminated regime, or a fully adversarial regime.
This is why the question of automatically adapting to the hardness of the environment, also called achieving \emph{best-of-both-worlds} guarantees, has received growing attention \sloppy\citep{bubeck2012best,seldin2014one,auer2016algorithm,seldin2017improved, wei2018more,zimmert2019optimal,ito2021parameter,ito2022adversarially,masoudian2021improved,gaillard2014second,mourtada2019optimality,ito2021optimal,lee2021achieving,rouyer2021algorithm,amirbetter,huang2022adaptive,tsuchiya2022best,erez2021towards,rouyernear,ito2022nearly,kong2022simultaneously,jin2020simultaneously,jin2021best,chen2022simultaneously,saha2022versatile,masoudianbest2022best,honda2023perturbed}. 
One of the most successful approaches has been to derive carefully tuned FTRL algorithms, which are canonically suited for the adversarial regime, and then show that these algorithms are also close to optimal in the stochastic regime as well.
This is achieved via proving a crucial self-bounding property, which then translates into stochastic guarantees. 
This type of algorithms have been proposed for multi-armed bandits \citep{wei2018more,zimmert2019optimal,ito2021parameter,ito2022adversarially}, combinatorial semi-bandits \citep{zimmert2019beating}, bandits with graph feedback \citep{ito2022nearly}, tabular MDPs \citep{jin2020simultaneously,jin2021best} and others.
Algorithms with self-bounding properties also automatically adapt to intermediate regimes of stochastic losses with adversarial corruptions \citep{lykouris2018stochastic,gupta2019better,zimmert2019optimal,ito2021optimal}, which highlights the strong robustness of this algorithm design.
However, every problem variation required careful design of potential functions and learning rate schedules.
For linear bandits, it has been still unknown whether self-bounding is possible and therefore the state-of-the-art best-of-both-worlds algorithm \citep{lee2021achieving} neither obtains optimal $\log T$ stochastic rate, nor canonically adapts to corruptions.

In this work, the make the following contributions: 1)  We propose a general reduction from best of both worlds to typical FTRL/OMD algorithms, sidestepping the need for customized potentials and learning rates. 2) We derive the first best-of-both-worlds algorithm for linear bandits that obtains $\log T$ regret in the stochastic regime, optimal adversarial worst-case regret and adapts canonically to corruptions. 3) We derive the first best-of-both-worlds algorithm for linear bandits and tabular MDPs with first-order guarantees in the adversarial regime. 4) We obtain the first best-of-both-worlds algorithms for bandits with graph feedback and bandits with expert advice with optimal $\log T$ stochastic regret.

\begin{table}[t]
\small
 
    \centering
    
\begin{adjustwidth}{-.25in}{-.25in}  
    \begin{tabular}{|c|c|c|c|c|c|c|}
        \hline
        Setting & Algorithm & Adversarial & Stochastic &  $o(C)$ & \makecell{Rd. 1} & \makecell{Rd. 2} \\
        \hline
        \multirow{4}{*}{\makecell{Linear  bandit}} & \cite{lee2021achieving} & $\sqrt{dT\log(|\calX|T)\log T}$ & $\frac{d\log(|\calX|T)\log(T)}{\Delta}$ &  & & \\
         \cline{2-7}
        & \fillcell \pref{thm: VR scrible} & $d\sqrt{\nu L_\star\log T}$ & $\frac{ d^2\nu\log T}{\Delta}$ & \checkmark &  & \\
        \cline{2-7}
        & \fillcell \pref{cor: linear bandit log det thm} & $d\sqrt{T\log T}$ & $\frac{d^2\log T}{\Delta}$ & \checkmark  & \checkmark & \\
        \cline{2-7}
        & \fillcell \pref{cor: linear bandit exponential thm} & $\sqrt{dT \log|\calX|}$ & $\frac{d\log|\calX|\log T}{\Delta}$ & \checkmark & \checkmark & \checkmark
        \\
        \hline
        \multirow{1}{*}{ \makecell{Contextual bandit}} & \fillcell \pref{cor: exp4}& $\sqrt{KT\log |\calX|}$ &$\frac{K\log|\calX|\log T}{\Delta}$  & \checkmark& \checkmark&\checkmark\\ 
        \hline
        \multirow{3}{*}{\makecell{Graph bandit\\Strongly observable}} &\cite{ito2022nearly} & $\sqrt{\alpha T\log T}\log(KT)$&$\frac{\alpha \log^2(KT)\log T}{\Delta}$ &\checkmark & & \\
        \cline{2-7} 
         & \cite{rouyernear}
        & $\sqrt{\widetilde\alpha T\log K}$ & $\frac{\widetilde\alpha\log(KT)\log T}{\Delta}$& &  &\\
        \cline{2-7} 
        & \fillcell\pref{cor: graph bandit tsallis}
        
        & $\sqrt{\min\{\baralpha,\alpha\log K\} T\log K}$ &$\frac{\min\{\baralpha,\alpha\log K\}\log K\log T}{\Delta}$ & \checkmark&\checkmark &\checkmark\\ 
        \hline
         \multirow{2}{*}{\makecell{Graph bandit\\Weakly observable}}
        &\cite{ito2022nearly} & $(\delta\log(KT)\log T)^\frac{1}{3}T^\frac{2}{3}$&$\frac{\delta \log(KT) \log T}{\Delta^2}$ &\checkmark & & 
        \\
        \cline{2-7} 
        &\fillcell\pref{cor: graph bandit weak exp3} & $(\delta\log K)^\frac{1}{3}T^\frac{2}{3}$ & $\frac{\delta\log K \log T}{\Delta^2}$& \checkmark&\checkmark & \checkmark 
        \\
        \hline
         \multirow{2}{*}{\makecell{Tabular MDP}} & \cite{jin2021best} & $\sqrt{H^2S^2AT\log^2 T}$ & $\frac{H^6S^2A\log^2 T}{\Delta'}$ & \checkmark & & \\
         \cline{2-7} 
         & \fillcell \pref{thm: MDP main theorem bound} & $\sqrt{HS^2AL_\star \ \log^4 T}$ & $\frac{H^2S^2A\ \log^3 T}{\Delta}$ & \checkmark & \checkmark &\checkmark  \\
        \hline
    \end{tabular}
\end{adjustwidth} 
\caption{Overview of regret bounds. The $o(C)$ column specifies whether the algorithm achieves the optimal dependence on the amount of corruption ($\sqrt{C}$ or $C^{\frac{2}{3}}$, depending on the setting). ``Rd. 1'' and ``Rd. 2'' indicate whether the result leverages the first and the second reductions described in \pref{sec: main tech}, respectively. $L_\star$ is the cumulative loss of the best action or policy. $\nu$ in \pref{thm: VR scrible} is the self-concordance parameter; it holds that $\nu\leq d$.  $\Delta'$ in \cite{jin2021best} is the gap in the $Q^\star$ function, which is different from the policy gap $\Delta$ we use; it holds that $\Delta\leq \Delta'$. $\alpha, \widetilde{\alpha}, \delta$ are complexity measures of feedback graphs defined in \pref{sec: graph bandits}.  }\label{tab:summary}
\end{table}

\section{Related Work}
Our reduction procedure is related to a class of model selection algorithms that uses a meta bandit algorithm to learn over a set of base bandit algorithms, and the goal is to achieve a comparable performance as if the best base algorithm is run alone. For the adversarial regime, \cite{agarwal2017corralling} introduced the Corral algorithm to learn over adversarial bandits algorithm that satisfies the stability condition. This framework is further improved by \cite{foster2020adapting,luo2022corralling}. For the stochastic regime, \cite{arora2021corralling,abbasi2020regret,cutkosky2021dynamic} introduced another set of techniques to achieve similar guarantees, but without explicitly relying on the stability condition. While most of these results focus on obtaining a $\sqrt{T}$ regret, \cite{arora2021corralling,cutkosky2021dynamic,wei2022model,pacchiano2022best} made some progress in obtaining $\text{polylog}(T)$ regret in the stochastic regime. Among them, \cite{wei2022model,pacchiano2022best} are most related to us since they also pursue a best-of-both-worlds guarantee across adversarial and stochastic regimes. However, their regret bounds, when applied to our problems, all suffer from highly sub-optimal dependence on $\log(T)$ and the amount of corruptions.

\section{Preliminaries}
We consider sequential decision making problems where the learning interacts with the environment in $T$ rounds. In each round $t$, the environment generates a loss vector $(\ell_{t, u})_{u \in \calX}$ and the learner generates a distribution over actions $p_t$ and chooses an action $A_t\sim p_t$. The learner then suffers loss $\ell_{t, A_t}$ and receives some information about $(\ell_{t, u})_{u \in \calX}$ depending on the concrete setting. In all settings but MDPs, we assume $\ell_{t,u}\in[-1,1]$. 

In the adversarial regime, $(\ell_{t,u})_{u\in\calX}$ is generated arbitrarily subject to the structure of the concrete setting (e.g., in linear bandits, $\E[\ell_{t,u}]=\inner{u,\ell_t}$ for arbitrary $\ell_t\in\mathbb{R}^d$). In the stochastic regime, we further assume that there exists an action $x^\star\in\calX$ and a gap $\Delta>0$ such that $\E[\ell_{t,x}-\ell_{t,x^\star}]\geq \Delta$ for all $x\neq x^\star$. We also consider the corrupted stochastic regime, where the assumption is relaxed to $\E[\ell_{t,x}-\ell_{t,x^\star}]\geq \Delta-C_t$ for some $C_t\geq 0$. We define $C=\sum_{t=1}^T C_t$.

\section{Main Techniques}\label{sec: main tech}

\subsection{The standard global self-bounding condition and a new linear bandit algorithm}
To obtain a best-of-both-world regret bound, previous works show  the following property for their algorithm: 
\begin{definition}[$\alpha$-global-self-bounding condition, or $\alpha$-GSB]  \label{def: previous condition}
\begin{align*}
    &\forall u\in\calX:\,\E\left[\sum_{t=1}^{T}\left(\ell_{t,A_t}-\ell_{t,u}\right)\right] \leq 
    \min\left\{c_0^{1-\alpha} T^\alpha, \ \ (c_1\log T)^{1-\alpha}\E\left[ \sum_{t=1}^T (1-p_{t,u})\right]^\alpha\right\} + c_2\log T
\end{align*}   
where $c_0, c_1, c_2$ are problem-dependent constants and $p_{t,u}$ is the probability of the learner choosing $u$ in round~$t$.   
\end{definition}
With this condition, one can use the standard \emph{self-bounding} technique to obtain a best-of-both-world regret bounds: 

\begin{proposition}\label{prop: standard algorithm}
    \sloppy
    If an algorithm satisfies $\alpha$-GSB, then its pseudo-regret is bounded by  $O\left(c_0^{1-\alpha}T^{\alpha} + c_2\log T\right)$ in the adversarial regime, and by $O\big( c_1\log(T)\Delta^{-\frac{\alpha}{1-\alpha}}+(c_1\log T)^{1-\alpha}\left(C\Delta^{-1}\right)^{\alpha}+c_2\log(T)\big)$ in the corrupted stochastic regime. 
\end{proposition}
Previous works have found algorithms with GSB in various settings, such as multi-armed bandits, semi-bandits, graph-bandits, and MDPs. 
Here, we provide one such algorithm (\pref{alg: VR scrible}) for linear bandits based on the framework of SCRiBLe \citep{abernethy2008competing}. The guarantee of \pref{alg: VR scrible} is stated in the next theorem under the more general ``learning with predictions'' setting introduced in \cite{rakhlin2013optimization}, where in every round $t$, the learner receives a loss predictor $m_t$ before making decisions. 
\begin{theorem}\label{thm: VR scrible}
    In the ``learning with predictions'' setting, \pref{alg: VR scrible} achieves a second-order regret bound of 
    $
        O\Big(d\sqrt{\nu \log(T)\sum_{t=1}^T (\ell_{t,A_t}-m_{t,A_t})^2} + d\nu\log T\Big)
    $
    in the adversarial regime, where $d$ is the feature dimension, $\nu\leq d$ is the self-concordance parameter of the regularizer, and $m_{t,x}=\inner{x,m_t}$ is the loss predictor. This also implies a first-order regret bound of $O\Big(d\sqrt{\nu\log (T)\sum_{t=1}^T \ell_{t,u}}\Big)$ if $\ell_{t,x}\geq 0$ and $m_t=\mathbf{0}$; it simultaneously achieves $\Reg=O\Big(\frac{d^2\nu\log T}{\Delta} + \sqrt{\frac{d^2\nu\log T}{\Delta}C}\Big)$ in the corrupted stochastic regime. 
\end{theorem}
See \pref{app: SCrible} for the algorithm and the proof. We call this algorithm Variance-Reduced SCRiBLe (VR-SCRiBLe) since it is based on the original SCRiBLe updates, but with some refinement in the construction of the loss estimator to reduce its variance. A good property of a SCRiBLe-based algorithm is that it simultaneously achieves data-dependent bounds (i.e., first- and second-order bounds), similar to the case of multi-armed bandit using FTRL/OMD with log-barrier regularizer \citep{wei2018more, ito2021parameter}. Like \cite{rakhlin2013optimization, bubeck2019improved, ito2021parameter, ito2022adversarially}, we can also use another procedure to learn $m_t$ and obtain path-length or variance bounds. The details are omitted.  

We notice, however, that the the bound in \pref{thm: VR scrible} is sub-optimal in $d$ since the best-known regret for linear bandits is either $d\sqrt{T\log T}$ or $\sqrt{dT\log|\calX|}$, depending on whether the number of actions is larger than $T^d$. These bounds also hint the possibility of getting better dependence on $d$ in the stochastic regime if we can also establish GSB for these algorithms. Therefore, we ask: can we achieve best-of-both-world bounds for linear bandits with the optimal $d$ dependence?   

An attempt is to try to show GSB for existing algorithms with optimal $d$ dependence, including the EXP2 algorithm \citep{bubeck2012towards} and the logdet-FTRL algorithm \citep{zimmert2022return}. Based on our attempt, it is unclear how to adapt the analysis of logdet-FTRL to show GSB. For EXP2, using the learning-rate tuning technique by \cite{ito2022nearly}, one can make it satisfy a similar guarantee as GSB, albeit with an additional $\log(T)$ factor, resulting in a bound $\frac{d\log(|\calX|T)\log(T)}{\Delta}$ in the stochastic regime. This gives a sub-optimal rate of $O(\log^2 T)$. Therefore, we further ask: can we achieve $O(\log T)$ regret in the stochastic regime with an optimal $d$ dependence? 

Motivated by these questions, we resort to approaches that do not rely on GSB, which appears for the first time in the best-of-both-world literature. Our approach is in fact a general \emph{reduction} --- it not only successfully achieves the desired bounds for linear bandits, but also provides a principled way to convert an algorithm with a worst-case guarantee to one with a best-of-both-world guarantee for general settings. In the next two subsections, we introduce our reduction techniques. 

\subsection{First reduction:~~Best of Both Worlds $\rightarrow$ Local Self-Bounding}

Our reduction approach relies on an algorithm to satisfy a \emph{weaker} condition than GSB, defined in the following:  

\begin{definition}[$\alpha$-local-self-bounding condition, or $\alpha$-LSB]\label{def: candidate aware}
We say an algorithm satisfies the $\alpha$-local-self-bounding condition if it takes a candidate action $\hatx\in\calX$ as input and satisfies the following pseudo-regret guarantee for any stopping time $t'\in[1,T]$: 
\begin{align*}
    &\forall u\in\calX:\,\E\left[\sum_{t=1}^{t'}\left(\ell_{t,A_t}-\ell_{t,u}\right)\right] \leq \min\left\{c_0^{1-\alpha}\E[t']^\alpha, \ \ (c_1\log T)^{1-\alpha}\E\left[\sum_{t=1}^{t'}(1-\bm{\mathbb{I}\{u=\hatx\}}p_{t,u})\right]^\alpha\right\} + c_2\log T 
\end{align*}
where $c_0, c_1, c_2$ are problem dependent constants. 
\end{definition}
The difference between LSB in \pref{def: candidate aware} and GSB in \pref{def: previous condition} is that LSB only requires the self-bounding term $\sum_{t}(1-p_{t,u})$ to appear when $u$ is a particular action $\hatx$ given as an input to the algorithm; for all other actions, the worst-case bound suffices. A minor additional requirement is that the pseudo-regret needs to hold for any stopping time (because our reduction may stop this algorithm during the learning procedure), but this is relatively easy to satisfy --- for all algorithms in this paper, without additional efforts, their regret bounds naturally hold for any stopping time. 


\begin{algorithm}[t]
    \caption{BOBW via LSB algorithms}\label{alg: adaptive alg2}
    \textbf{Input}: LSB algorithm $\calA$ \\
    $T_1\leftarrow 0$, \ $T_0\leftarrow -c_2\log(T)$. \\
    $\hatx_1 \sim \text{unif}(\calX)$. \\
    $t\leftarrow 1$. \\
    \For{$k=1, 2\ldots,$}{
       Initialize $\calA$ with candidate $\hatx_k$.\\
       Set counters $N_k(x) = 0$ for all $x\in\calX$. \\
       \For{$t=T_k+1, T_k+2, \ldots$}{
        Play action $A_t$ as suggested by $\calA$, and advance $\calA$ by one step. \\
        $N_k(A_t)\leftarrow N_k(A_t)+1$\\
        \If{
        $t - T_{k} \geq 2(T_k - T_{k-1})$ and $\exists x\in\calX\setminus\{\hatx_k\}$ such that $N_k(x) \geq \frac{t-T_k}{2}$}
        {
            $\hatx_{k+1}\leftarrow x$.\\
            $T_{k+1}\leftarrow t$. \\
            \textbf{break}. 
        }
        }
    }
\end{algorithm}

For linear bandits, we find that an adaptation of the logdet-FTRL algorithm \citep{zimmert2022return} satisfies $\frac{1}{2}$-LSB, as stated in the following lemma. The proof is provided in \pref{app: log det analysis}. 

\begin{lemma}\label{lem: CA logdet guarantee}
For $d$-dimensional linear bandits, by transforming the action set into $\big\{\binom{x}{0} \,\big|\,x\in\calX\setminus\{\hatx\}\big\}\cup \big\{\binom{\mathbf{0}}{1}\big\} \subset \bbR^{d+1}$, \pref{alg: logdet} satisfies $\frac{1}{2}$-LSB with $c_0=O((d+1)^2\log T)$ and $c_1=c_2=O((d+1)^2)$.
\end{lemma}

 With the LSB condition defined, we develop a reduction procedure (\pref{alg: adaptive alg2}) which turns any algorithm  with LSB into a best-of-both-world algorithm that has a similar guarantee as in \pref{prop: standard algorithm}. The guarantee is formally stated in the next theorem. 
\begin{theorem}
\label{thm: candidate aware to bob}
If $\calA$ satisfies $\alpha$-LSB, then the regret of \pref{alg: adaptive alg2} with $\calA$ as the base algorithm is upper bounded by
$O\left(c_0^{1-\alpha} T^\alpha + c_2 \log^2 T\right)$ 
in the adversarial regime and by
$\sloppy O\big( c_1\log(T)\Delta^{-\frac{\alpha}{1-\alpha}} + (c_1\log T)^{1-\alpha}\left(C\Delta^{-1}\right)^{\alpha}+$ $c_2\log(T)\log(C\Delta^{-1})\big)$
in the corrupted stochastic regime. 
\end{theorem}
See \pref{app: proof of first reduction} for a proof of \pref{thm: candidate aware to bob}. In particular, \pref{thm: candidate aware to bob} together with \pref{lem: CA logdet guarantee} directly lead to a better best-of-both-world bound for linear bandits.
\begin{corollary}\label{cor: linear bandit log det thm}
Combining \pref{alg: adaptive alg2} and \pref{alg: logdet} results in a linear bandit algorithm with $O\big(d\sqrt{T\log T}\big)$ regret in the adversarial regime and $O\Big(\frac{d^2\log T}{\Delta} + d\sqrt{\frac{C\log T}{\Delta}}\Big)$ regret in the corrupted stochastic regime simultaneously.
\end{corollary}
\paragraph{Ideas of the first reduction} \pref{alg: adaptive alg2} proceeds in epochs. In epoch $k$, there is an action $\hatx_k\in\calX$ being selected as the candidate ($\hatx_1$ is randomly drawn from $\calX$). The procedure simply executes the base algorithm $\calA$ with input  $\hatx=\hatx_k$, and monitors the number of draws to each action. If in epoch $k$ there exists some $x\neq \hatx_k$ being drawn more than half of the time, and the length of epoch $k$ already exceeds two times the length of epoch $k-1$, then a new epoch is started with $\hatx_{k+1}$ set to $x$.

\pref{thm: candidate aware to bob} is not sufficient to be considered a black-box reduction approach, since algorithms with LSB are not common. Therefore, our next step is to present a more general reduction that makes use of recent advances of Corral algorithms.

\subsection{Second reduction:~~Local Self-Bounding $\rightarrow$ Importance-Weighting Stable}\label{sec: second reduction}
In this subsection, we show that one can achieve LSB using the idea of \emph{model selection}. Specifically, we will run a variant of the Corral algorithm \citep{agarwal2017corralling} over two instances: one is $\hatx$, and the other is a \emph{importance-weighting-stable} algorithm (see \pref{def: iw stable}) over the action set $\calX\backslash \{\hatx\}$. Here, we focus on the case of $\alpha=\frac{1}{2}$, which is the case for most standard settings where the worst-case regret is $\sqrt{T}$; examples for $\alpha=\frac{2}{3}$ in the graph bandit setting is discussed in \pref{sec: graph bandits}. 

First, we introduce the notion of importance-weighting stablility. 
\begin{definition}[$\frac{1}{2}$-iw-stable]\label{def: iw stable}
An algorithm is \emph{$\frac{1}{2}$-iw-stable} (importance-weighting stable), if given an adaptive sequence of weights $q_1,q_2,\dots\in(0,1]$ and the assertion that the feedback in round $t$ is observed with probability $q_t$, it obtains the following pseudo regret guarantee for any stopping time~$t'$: 
\begin{align*}
    \E\left[\sum_{t=1}^{t'} \left(\ell_{t,A_t} - \ell_{t,u}\right)\right] \leq 
        \E\left[\sqrt{c_1\sum_{t=1}^{t'}\frac{1}{q_t}} + \frac{c_2}{\min_{t\leq t'} q_t}\right]. 
\end{align*}
\end{definition} 
\pref{def: iw stable} requires that even if the algorithm only receives the desired feedback with probability $q_t$ in round $t$, it still achieves a meaningful regret bound that smoothly degrades with $\sum_{t}\frac{1}{q_t}$. In previous works on Corral and its variants \citep{agarwal2017corralling, foster2020adapting}, a similar notion of $\frac{1}{2}$-stability is defined as having a regret of $\sqrt{c_1\left(\max_{\tau\leq t'}\frac{1}{q_{\tau}}\right) t'}$, which is a weaker assumption than our $\frac{1}{2}$-iw-stability. Our stronger notion of stability is crucial to get a best-of-both-world bound, but it is still very natural and holds for a wide range of algorithms.

\begin{algorithm}[t] 
\caption{LSB via Corral (for $\alpha=\frac{1}{2}$)}\label{alg: corral}
    \textbf{Input}:  candidate action $\hatx$, $\frac{1}{2}$-iw-stable algorithm $\calB$ over $\calX\backslash\{\hatx\}$ with constants $c_1,c_2$. \\
    \textbf{Define}: $\psi_t(q)= -\frac{2}{\eta_t}\sum_{i=1}^2 \sqrt{q_i} + \frac{1}{\beta}\sum_{i=1}^2 \ln\frac{1}{q_i}$.\ \ \\
    $B_0=0$. \\
    \For{$t=1,2,\ldots$}{
        Let
        \begin{align*}
            &\barq_{t}= \argmin_{q\in\Delta_2}\left\{\left\langle q, \sum_{\tau=1}^{t-1} z_{\tau} - \begin{bmatrix} 0 \\ B_{t-1}
            \end{bmatrix}\right\rangle + \psi_t(q)\right\}, \quad
            q_{t} = \left(1-\frac{1}{2t^2}\right)\barq_{t} + \frac{1}{4t^2}\one, \\
            &\text{with}\ \ \ \eta_{t} = \frac{1}{\sqrt{t} + 8\sqrt{c_1}}, \qquad \beta = \frac{1}{8c_2}\,. 
        \end{align*}
        Sample $i_t\sim q_t$. \\
        \lIf{$i_t=1$}{
            draw $A_t=\hatx$ and observe $\ell_{t,A_t}$ 
        }
        \lElse{
            draw $A_t$ according to base algorithm $\calB$ and observe $\ell_{t,A_t}$  
        }
        Define $z_{t,i} = \frac{(\ell_{t,A_t}+1)\ind\{i_t=i\}}{q_{t,i}}-1$ and
        \begin{align*}
            B_t= 
                \sqrt{c_1\sum_{\tau=1}^t \frac{1}{q_{\tau,2}}} + \frac{c_2}{\min_{\tau\leq t} q_{\tau,2}}.  
        \end{align*}
        
    }
\end{algorithm}

Below, we provide two examples of $\frac{1}{2}$-iw-stable algorithms. Their analysis is mostly standard FTRL analysis (considering extra importance weighting) and can be found in \pref{app: iw-stable}. 
\begin{lemma}\label{lem: EXP2}
For linear bandits, EXP2 with adaptive learning rate (\pref{alg: EXP2}) is $\frac{1}{2}$-iw-stable with constants $c_1=c_2=O(d\log |\calX|)$. 
\end{lemma}
\begin{lemma}\label{lem: EXP4}
For contextual bandits, EXP4 with adaptive learning rate (\pref{alg: EXP4}) is $\frac{1}{2}$-iw-stable with constants $c_1=O(K\log |\calX|), c_2 =0$.
\end{lemma}
Our reduction procedure is \pref{alg: corral}, which turns an $\frac{1}{2}$-iw-stable algorithm into one with $\frac{1}{2}$-LSB. The guarantee is formalized in the next theorem, whose proof is in \pref{app: main corral theorem proof}. 

\begin{theorem}\label{thm: basic corral thm}
    If $\calB$ is a $\frac{1}{2}$-iw-stable algorithm with constants $(c_1, c_2)$,  then \pref{alg: corral} with $\calB$ as the base algorithm satisfies $\frac{1}{2}$-LSB with constants $(c_0',c_1',c_2')$ where $c_0'=c_1'=O(c_1)$ and $c_2'=O(c_2)$. 
\end{theorem}

Cascading \pref{thm: basic corral thm} and \pref{thm: candidate aware to bob}, and combining it with \pref{lem: EXP2} and  \pref{lem: EXP4}, respectively, we get the following corollaries: 

\begin{corollary}\label{cor: linear bandit exponential thm}
Combining \pref{alg: adaptive alg2}, \pref{alg: corral} and \pref{alg: EXP2} results in a linear bandit algorithm with $O\big(\sqrt{dT\log |\calX|}\big)$ regret in the adversarial regime and $O\Big(\frac{d\log|\calX|\log T}{\Delta} + \sqrt{\frac{d\log|\calX|\log T}{\Delta}C}\Big)$ regret in the corrupted stochastic regime simultaneously.
\end{corollary}

\begin{corollary}
\label{cor: exp4}
Combining \pref{alg: adaptive alg2}, \pref{alg: corral} and \pref{alg: EXP4} results in a contextual bandit algorithm with $O\big(\sqrt{KT\log |\calX|}\big)$ regret in the adversarial regime and $O\Big(\frac{K\log |\calX| \log T}{\Delta} + \sqrt{\frac{K\log |\calX|\log T}{\Delta}C}\Big)$ regret in the corrupted stochastic regime simultaneously, where $K$ is the number of arms. 
\end{corollary}

\paragraph{Ideas of the second reduction} \pref{alg: corral} is related to the Corral algorithm \citep{agarwal2017corralling, foster2020adapting, luo2022corralling}. We use a special version which is essentially an FTRL algorithm (with a hybrid $\frac{1}{2}$-Tsallis entropy and log-barrier regularizer) over two base algorithms: the candidate $\hatx$, and an algorithm $\calB$ operating on the reduced action set $\calX\backslash\{\hatx\}$. For base algorithm~$\calB$, the Corral algorithm adds a bonus term that upper bounds the regret of $\calB$ under importance weighting (i.e., the quantity $B_t$ defined in \pref{alg: corral dd}). In the regret analysis, the bonus creates a negative term as well as a bonus overhead in the learner's regret. The negative term can be used to cancel the positive regret of $\calB$, and the analysis reduces to bounding the bonus overhead. Showing that the bonus overhead is upper bounded by the order of $\sqrt{c_1\log(T)\sum_{t=1}^T(1-p_{t,\hatx})}$ is the key to establish the LSB property. 

Combining \pref{alg: corral} and \pref{alg: adaptive alg2}, we have the following general reduction: \emph{as long as we have an $\frac{1}{2}$-iw-stable algorithm over $\calX\backslash\{\hatx\}$ for any $\hatx\in\calX$, we have an algorithm with the best-of-both-world guarantee when the optimal action is unique.} Notice that $\frac{1}{2}$-iw-stable algorithms are quite common -- usually it can be just a FTRL/OMD algorithm with adaptive learning rate. 

The overall reduction is reminiscent of the \scalebox{0.9}{\textsf{G-COBE}} procedure by \cite{wei2022model}, which also runs model selection over $\hatx$ and a base algorithm for $\calX\backslash\{\hatx\}$ (similar to \pref{alg: corral}), and dynamically restarts the procedure with increasing epoch lengths (similar to \pref{alg: adaptive alg2}). However, besides being more complicated, \scalebox{0.9}{\textsf{G-COBE}} only guarantees a bound of $\frac{\text{polylog}(T)}{\Delta}+C$ in the corrupted stochastic regime (omitting dependencies on $c_1, c_2$), which is sub-optimal in both $C$~and~$T$.\footnote{However, \cite{wei2022model} is able to handle a more general type of corruption. }

\section{Case Study:~~Graph Bandits}\label{sec: graph bandits}
In this section, we show the power of our reduction by improving the state of the art of best-of-both-worlds algorithms for bandits with graph feedback.
In bandits with graph feedback, the learner is given a directed feedback graph $G = (V, E)$, where the nodes $V=[K]$ correspond to the $K$-arms of the bandit.
Instead of observing the loss of the played action $A_t$, the player observes the loss $\ell_{t,j}$ iff $(A_t,j) \in E$.
Learnable graphs are divided into \emph{strongly observable graphs} and \emph{weakly observable graphs}.
In the first case, every arm $i\in[K]$ must either receive its own feedback, i.e. $(i,i)\in E$, or \emph{all} other arms do, i.e. $\forall j\in[K]\setminus\{i\}:\,(j,i) \in E$.
In the weakly observable case, every arm $i\in[K]$ must be observable by at least one arm, i.e. $\exists j\in[K]:\,(j,i)\in E$.
The following graph properties characterize the difficulty of the two regimes.
An independence set is any subset $V'\subset V$ such that no edge exists between any two distinct nodes in $V'$, i.e. $\forall i,j\in V', i\neq j$ we have $(i,j)\not\in E$ and $(j,i)\not\in E$. The independence number $\alpha$ is the size of the largest independence set.
A related number is the weakly independence number $\baralpha$, which is the independence number of the subgraph $(V,E')$ obtained by removing all one-sided edges, i.e. $(i,j)\in E'$ iff $(i,j)\in E$ and $(j,i)\in E$.
For undirected graphs, the two notions coincide, but in general $\baralpha$ can be larger than $\alpha$ by up to a factor of $K$.
Finally, a weakly dominating set is a subset $D\subset V$ such that every node in $V$ is observable by some node in $D$, i.e. $\forall j\in V\,\exists i\in D:\,(i,j)\in E$.
The weakly dominating number $\delta$ is the size of the smallest weakly dominating set.

\citet{alon2015online} provides a almost tight characterization of the adversarial regime, providing upper bounds of $O(\sqrt{\alpha T\log T\log K })$ and $O(\sqrt{\baralpha T\log K })$ for the strongly observable case and $O((\delta\log K)^\frac{1}{3}T^\frac{2}{3})$ for the weakly observable case, as well as matching lower bounds up to all log factors.
\citet{zimmert2019connections} have shown that the $\log(T)$ dependency can be avoided and that hence $\alpha$ is the right notion of difficulty for strongly observable graphs even in the limit $T\rightarrow\infty$ (though they pay for this with a larger $\log K$ dependency).

State-of-the-art best-of-both-worlds guarantees  for bandits with graph feedback are derivations of EXP3 and obtain either $O(\frac{\baralpha\log^2(T)}{\Delta})$ or $O(\frac{\alpha\log(T)\log^2(TK)}{\Delta})$ regret for strongly observable graphs and $O(\frac{\delta \log^2(T)}{\Delta^2})$ for weakly observable graphs \citep{rouyernear,ito2022nearly}\footnote{
\citet{rouyernear} obtains better bounds when the gaps are not uniform, while \cite{ito2022nearly} can handle graphs that are unions of strongly observable and weakly observable sub-graphs. We do not extend our analysis to the latter case for the sake of simplicity.
}.  
Our black-box reduction leads directly to algorithms with optimal $\log(T)$ regret in the stochastic regime.
\subsection{Strongly observable graphs}
We begin by providing an examples of $\frac{1}{2}$-iw-stable algorithm
 for bandits with strongly observable graph feedback. The algorithm and the proof are in \pref{app: log K tsallis}. 
\begin{lemma}\label{lem: strong graph 1}
For bandits with strongly observable graph feedback, $(1-1/\log K)$-Tsallis-INF with adaptive learning rate (\pref{alg: Tsallis-Inf}) is $\frac{1}{2}$-iw-stable with constants $c_1=O(\min\{\baralpha,\alpha \log K\}\log K), c_2=0$. 
\end{lemma}

Before we apply the reduction, note that \pref{alg: corral}
requires that the player observes the loss $\ell_{t,A_t}$ when playing arm $A_t$, which is not necessarily the case for all strongly observable graphs.
However, this can be overcome via defining an observable surrogate loss that is used in \pref{alg: corral} instead.
We explain how this works in detail in \pref{sec: no self loop} and assume from now that this technical issue does not arise.

Cascading the two reduction steps, we immediately obtain the following.
\begin{corollary}\label{cor: graph bandit tsallis}
Combining \pref{alg: adaptive alg2}, \pref{alg: corral} and \pref{alg: Tsallis-Inf} results in a graph bandit algorithm with $O\big(\sqrt{\min\{\baralpha,\alpha\log K\} T\log K}\big)$ regret in the adversarial regime and \sloppy$O\Big(\frac{\min\{\baralpha,\alpha\log K\}\log T\log K}{\Delta} + \sqrt{\frac{\min\{\baralpha,\alpha\log K\}\log T\log K}{\Delta}C}\Big)$ regret in the corrupted stochastic regime simultaneously.
\end{corollary}
\subsection{Weakly observable graphs}
Weakly observable graphs motivate our general definition of LSB stable beyond $\alpha=\frac{1}{2}$.
We first need to define the equivalent of $\frac{1}{2}$-iw-stable for this regime.
\begin{definition}[$\frac{2}{3}$-iw-stable]\label{def: iw stable 2/3}
An algorithm is \emph{$\frac{2}{3}$-iw-stable} (importance-weighting stable), if given an adaptive sequence of weights $q_1,q_2,\dots,q_\tau$ and the feedback in any round being observed with probability $q_t$, it obtains the following pseudo regret guarantee for any stopping time $t'$: 
\begin{align*}
    \E\left[\sum_{t=1}^{t'} \left(\ell_{t,A_t} - \ell_{t,u}\right)\right] \leq 
        \E\left[c_1^\frac{1}{3}\left(\sum_{t=1}^{t'}\frac{1}{\sqrt{q_t}}\right)^\frac{2}{3}+\max_{t\leq t'}\frac{c_2}{q_t}\right].
\end{align*}
\end{definition} 
An example of such an algorithm is given by the following lemma (see \pref{app: weakly obs exp} for the algorithm and the proof). 
\begin{lemma}\label{lem: weak graph}
For bandits with weakly observable graph feedback, EXP3 with adaptive learning rate (\pref{alg: weak exp3}) is $\frac{2}{3}$-iw-stable with constants $c_1=c_2=O(\delta\log K)$. 
\end{lemma}
Similar to the $\frac{1}{2}$-case, this allows for a general reduction to LSB algorithms.
\begin{theorem}\label{thm: 2/3 corral thm}
    If $\calB$ is a $\frac{2}{3}$-iw-stable algorithm with constants $(c_1, c_2)$,  then \pref{alg: 2/3 corral} with $\calB$ as the base algorithm satisfies $\frac{2}{3}$-LSB with constants $(c_0',c_1',c_2')$ where $c_0'=c_1'=O(c_1)$ and $c_2'=O(c_1^\frac{1}{3} + c_2)$. 
\end{theorem}
See \pref{app: 2/3 corral} for the proof. \pref{alg: 2/3 corral} works almost the same as \pref{alg: corral}, but we need to adapt the bonus terms to the definition of $\frac{2}{3}$-iw-stable.
The major additional difference is that  we do not necessarily observe the loss of the action played and hence need to play exploratory actions with probability $\gamma_t$ in every round to estimate the performance difference between $\hatx$ and $\calB$.
A second point to notice is that the base algorithm $\calB$ uses the action set $\calX\setminus\{\hatx\}$, but is still allowed to use $\hatx$ in its internal exploration policy.
This is necessary because the sub-graph with one arm removed might have a larger dominating number, or might even not be learnable at all. By allowing $\hatx$ in the internal exploration, we guarantee the the regret of the base algorithm is not larger than over the full action set.
Finally, cascading the lemma and the reduction, we obtain
\begin{corollary}\label{cor: graph bandit weak exp3}
Combining \pref{alg: adaptive alg2}, \pref{alg: 2/3 corral} and \pref{alg: weak exp3} results in a graph bandit algorithm with $O\big((\delta\log K)^\frac{1}{3} T^\frac{2}{3}\big)$ regret in the adversarial regime and $O\Big(\frac{\delta\log K\log T}{\Delta^2} + \left(\frac{C^2\delta\log K\log T}{\Delta^2}\right)^\frac{1}{3}\Big)$ regret in the corrupted stochastic regime simultaneously.
\end{corollary}

\section{More Adaptations}
To demonstrate the versatility of our reduction framework, we provide two more adaptations. The first demonstrates that our reduction can be easily adapted to obtain a \emph{data-dependent} bound (i.e., first- or second-order bound) in the adversarial regime, provided that the base algorithm achieves a corresponding data-dependent bound. The second tries to eliminate the undesired requirement by the corral algorithm (\pref{alg: corral}) that the base algorithm needs to operate in $\calX\backslash\{\hatx\}$ instead of the more natural $\calX$. We show that under a stronger stability assumption, we can indeed just use a base algorithm that operates in $\calX$. This would be helpful for settings where excluding one single action/policy in the algorithm is not straightforward (e.g., MDP). Finally, we combine the two techniques to obtain a first-order best-of-both-world bound for tabular MDPs.

\subsection{Reduction with first- and second-order bounds}\label{sec: data-dependent}

A first-order regret bound refers to a regret bound of order  $O\big(\sqrt{c_1\text{polylog}(T)L_\star}+c_2\text{polylog}(T)\big)$, where $L_\star=\min_{x\in\calX}\E\big[\sum_{t=1}^T \ell_{t,x}\big]$ is the best action's cumulative loss. This is meaningful under the assumption that $\ell_{t,x}\geq 0$ for all $t,x$. A second-order regret bound refers to a bound of order $O\Big(\sqrt{c_1\text{polylog}(T)\E\big[\sum_{t=1}^T (\ell_{t,A_t}-m_{t,A_t})^2}\big]+c_2\text{polylog}(T)\Big)$, where $m_{t,x}$ is a \emph{loss predictor} for action $x$ that is available to the learner at the beginning of round $t$. We refer the reader to \cite{wei2018more, ito2021parameter} for more discussions on data-dependent bounds. 

We first define counterparts of the LSB condition and iw-stability condition with data-dependent guarantees: 
\begin{definition}[$\frac{1}{2}$-dd-LSB]\label{def: candidate aware data-dep}
We say an algorithm satisfies $\frac{1}{2}$-dd-LSB (data-dependent LSB) if it takes a candidate action $\hatx\in\calX$ as input and satisfies the following pseudo-regret guarantee for any stopping time $t'\in[1,T]$: 
\begin{align*}
    &\forall u\in\calX: \\
    &\E\left[\sum_{t=1}^{t'}\left(\ell_{t,A_t}-\ell_{t,u}\right)\right] \leq \sqrt{(c_1\log T)\E\left[\sum_{t=1}^{t'} \left(\sum_{x} p_{t,x}\xi_{t,x} -\bm{\mathbb{I}\{u=\hatx\}}p_{t,u}^2\xi_{t,u}\right)\right]} + c_2\log T 
\end{align*}
where $c_1, c_2$ are some problem dependent constants, $\xi_{t,x} = (\ell_{t,x}-m_{t,x})^2$ in the second-order bound case, and $\xi_{t,x}=\ell_{t,x}$ in the first-order bound case.   
\end{definition}

\begin{definition}[$\frac{1}{2}$-dd-iw-stable]\label{def: iw stable data-dependent}
An algorithm is \emph{$\frac{1}{2}$-dd-iw-stable} (data-dependent-iw-stable) if given an adaptive sequence of weights $q_1,q_2,\dots\in(0,1]$ and the assertion that the feedback in round $t$ is observed with probability $q_t$, it obtains the following pseudo regret guarantee for any stopping time~$t'$: 
\begin{align*}
    \E\left[\sum_{t=1}^{t'} \left(\ell_{t,A_t} - \ell_{t,u}\right)\right] \leq 
        \sqrt{c_1\E\left[\sum_{t=1}^{t'}\frac{\upd_t\cdot\xi_{t,A_t}}{q_t^2}\right]} + \E\left[\frac{c_2}{\min_{t\leq t'} q_t}\right],   
\end{align*}
where $\upd_t=1$ if feedback is observed in round $t$ and $\upd_t=0$ otherwise, and $\xi_{t,x}$ is defined in the same way as in \pref{def: candidate aware data-dep}. 
\end{definition} 
We can turn a dd-iw-stable algorithm into one with dd-LSB (see \pref{app: dd LSB to dd iw sta} for the proof): 
\begin{theorem}\label{thm: basic corral thm data-dep}
    If $\calB$ is $\frac{1}{2}$-dd-iw-stable with constants $(c_1, c_2)$,  then \pref{alg: corral dd} with $\calB$ as the base algorithm satisfies $\frac{1}{2}$-dd-LSB with constants $(c_1',c_2')$ where $c_1'=O(c_1)$ and $c_2'=O(\sqrt{c_1} + c_2)$. 
\end{theorem}
Then we turn an algorithm with dd-LSB into one with data-dependent best-of-both-world guarantee (see \pref{app: data dep first red} for the proof): 
\begin{theorem}\label{thm: data-dependent reduction }
If an algorithm $\calA$ satisfies $\frac{1}{2}$-dd-LSE, then the regret of \pref{alg: adaptive alg2} with $\calA$ as the base algorithm is upper bounded by $O\Big(\sqrt{c_1\E\left[\sum_{t=1}^T \xi_{t,A_t}\right] \log^2T} + c_2\log^2 T\Big)$ in the adversarial regime and $O\Big(\frac{c_1\log T}{\Delta} + \sqrt{\frac{c_1\log T}{\Delta}C} + c_2\log(T)\log(C\Delta^{-1})\Big)$ in the corrupted stochastic regime.  
\end{theorem}

\subsection{Achieving LSB without excluding $\hatx$}\label{sec: not excluding}
Our reduction in \pref{sec: second reduction} requires that the base algorithm $\calB$ to operate in the action set of $\calX\backslash\{\hatx\}$. This is sometimes not easy to implement for structural problems where actions share common components (e.g., MDPs or combinatorial bandits). To eliminate this requirement so that we can simply use a base algorithm $\calB$ that operates in the original action space $\calX$, we propose the following stronger notion of iw-stability that we called \emph{strongly-iw-stable}: 
\begin{definition}[$\frac{1}{2}$-strongly-iw-stable]\label{def: strongly iw stable}
An algorithm is \emph{$\frac{1}{2}$-strongly-iw-stable} if the following holds: given an adaptive sequence of weights $q_1,q_2,\dots\in(0,1]^{\calX}$ and the assertion that the feedback in round $t$ is observed with probability $q_t(x)$ if the algorithm chooses $A_t=x$, it obtains the following pseudo regret guarantee for any stopping time~$t'$: 
\begin{align*}
    \E\left[\sum_{t=1}^{t'} \left(\ell_{t,A_t} - \ell_{t,u}\right)\right] \leq 
        \E\left[\sqrt{c_1\sum_{t=1}^{t'}\frac{1}{q_t(A_t)}} + \frac{c_2}{\min_{t\leq t'} \min_x q_t(x)}\right]. 
\end{align*}
\end{definition} 
Compared with iw-stability defined in \pref{def: iw stable}, strong-iw-stability requires the bound to have an additional flexibility: if choosing action $x$ results in a probability $q_t(x)$ of observing the feedback, the regret bound needs to adapts to $\sum_{t} \frac{1}{q_t(A_t)}$. For the class of FTRL/OMD algorithms, strong-iw-stability holds if the \emph{stability term} is bounded by a constant no matter what action is chosen. Examples include log-barrier-OMD/FTRL for multi-armed bandits, SCRiBLe \citep{abernethy2008competing} or a truncated version of continuous exponential weights \citep{ito2020tight} for linear bandits. In fact, Upper-Confidence-Bound (UCB) algorithms also satisfy strong-iw-stability, though it is mainly designed for the pure stochastic setting; however, we will need this property when we design algorithms for adversarial MDPs, where the transition estimation part is done through UCB approaches. This will be demonstrated in \pref{sec: adversarial MDP}. With strong-iw-stability, the second reduction only requires a base algorithm over $\calX$: 

\begin{theorem}\label{thm: basic corral thm strong}
    If $\calB$ is a $\frac{1}{2}$-strongly-iw-stable algorithm with constants $(c_1, c_2)$,  then \pref{alg: corral strong} with $\calB$ as the base algorithm satisfies $\frac{1}{2}$-LSB with constants $(c_0',c_1',c_2')$ where $c_0'=c_1'=O(c_1)$ and $c_2'=O(\sqrt{c_1} + c_2)$. 
\end{theorem}
The proof is provided in \pref{app: strongly stable}. 

\subsection{Application to MDPs}\label{sec: adversarial MDP}
Finally, we combine the techniques developed in \pref{sec: data-dependent} and \pref{sec: not excluding} to obtain a best-of-both-world guarantee for tabular MDPs with a first-order regret bound in the adversarial regime. We use \pref{alg: log-barrier MDP} as the base algorithm, which is adapted from the UCB-log-barrier Policy Search algorithm by \cite{lee2020bias} to satisfy both the data-dependent property (\pref{def: iw stable data-dependent}) and the strongly iw-stable property (\pref{def: strongly iw stable}). The corral algorithm we use is \pref{alg: corral MDP}, which takes a base algorithm with the dd-strongly-iw-stable property and turns it into a data-dependent best-of-both-world algorithm. The details and notations are all provided in \pref{app: tabular MDP}. The guarantee of the algorithm is formally stated in the next theorem. 
\begin{theorem}\label{thm: MDP main theorem bound}
Combining \pref{alg: adaptive alg2}, \pref{alg: corral MDP}, and \pref{alg: log-barrier MDP} results in an MDP algorithm with 
$
       O\Big(\sqrt{S^2A  HL_\star\log^2(T)\iota^2} + S^5A^2\log^2(T)\iota^2\Big)
$
   regret in the adversarial regime, and \sloppy
   $
       O\Big(\frac{H^2S^2A\iota^2\log T}{\Delta} + \sqrt{\frac{H^2S^2A\iota^2\log T}{\Delta}C} + S^5A^2\iota^2\log(T)\log(C\Delta^{-1})\Big)
   $
   in the corrupted stochastic regime, where $S$ is the number of states, $A$ is the number of actions, $H$ is the horizon, $L_\star$ is the cumulative loss of the best policy, and $\iota=\log(SAT)$. 
\end{theorem}

\section{Conclusion}
We provided a general reduction from the best-of-both-worlds problem to a wide range of FTRL/OMD algorithms, which improves the state of the art in several problem settings.
We showed the versatility of our approach by extending it to preserving data-dependent bounds.

Another potential application of our framework is partial monitoring, where one might improve the $\log^2(T)$ rates of \citet{} to $\log(T)$ for both the $T^\frac{1}{2}$ and the $T^\frac{2}{3}$ regime using our respective reductions.

A weakness of our approach is the uniqueness requirement of the best action.
While this assumption is typical in the best-of-both-worlds literature, it is not merely an artifact of the analysis for us due to the doubling procedure in the first reduction. 
Additionally, our reduction can only obtain worst-case $\Delta$ dependent bounds in the stochastic regime, which can be significantly weaker than more refined notions of complexity.

\bibliography{reference}

\newpage

\appendix
\appendixpage

{
\startcontents[section]
\printcontents[section]{l}{1}{\setcounter{tocdepth}{2}}
}



\newpage

\section{FTRL Analysis} 
\begin{lemma}\label{lem: FTRL}  
    The optimistic FTRL algorithm over a convex set $\Omega$: 
    \begin{align*}
    p_t = \argmin_{p\in\Omega}\left\{ \left\langle p, \sum_{\tau=1}^{t-1}\ell_\tau \right\rangle + m_t + \psi_t(p) \right\} 
\end{align*}
guarantees the following for any $t'$: 
\begin{align*}
    &\sum_{t=1}^{t'}\langle p_t-u, \ell_t \rangle  
    \leq \\
    &\psi_0(u)-\min_{p\in\Omega}\psi_0(p) + \sum_{t=1}^{t'} (\psi_t(u) - \psi_{t-1}(u) - \psi_t(p_t) + \psi_{t-1}(p_t)) + \sum_{t=1}^{t'}\underbrace{\max_{p\in\Omega}\left(\langle p_t-p, \ell_t - m_t\rangle - D_{\psi_t}(p, p_t)\right)}_{\text{stability}}.
\end{align*}
\end{lemma}

\begin{proof}
Let $L_t\triangleq \sum_{\tau=1}^t \ell_\tau$. 
Define $F_t(p) = \left\langle p, L_{t-1}+m_t\right\rangle + \psi_t(p)$ and $G_t(p)=\left\langle p, L_t\right\rangle + \psi_t(p)$. Therefore, $p_t$ is the minimizer of $F_t$ over $\Omega$. Let $p_{t+1}'$ be minimizer of $G_t$ over $\Omega$. Then by the first-order optimality condition, we have
\begin{align}
    F_t(p_t) - G_t(p_{t+1}') &\leq F_t(p_{t+1}') - G_t(p_{t+1}') - D_{\psi_t}(p_{t+1}', p_t)
    = -\langle p_{t+1}', \ell_t - m_t \rangle - D_{\psi_t}(p_{t+1}', p_t).     \label{eq: eq1}
\end{align}
By definition, we also have
\begin{align}
    G_t(p_{t+1}') - F_{t+1}(p_{t+1}) &\leq G_t(p_{t+1}) - F_{t+1}(p_{t+1}) = \psi_t(p_{t+1}) - \psi_{t+1}(p_{t+1}) - \inner{p_{t+1},m_{t+1}}. \label{eq: eq2}
\end{align}
Thus, 
\allowdisplaybreaks
\begin{align*}
    &\sum_{t=1}^{t'} \langle p_t, \ell_t\rangle \\ &\leq \sum_{t=1}^{t'} \left(\langle p_t, \ell_t\rangle - \inner{p_{t+1}', \ell_t-m_t} - D_{\psi_t}(p_{t+1}', p_t) + G_t(p_{t+1}') - F_t(p_t) \right)  \tag{by \eqref{eq: eq1}}  \\
    &= \sum_{t=1}^{t'} \left(\langle p_t, \ell_t\rangle - \inner{p_{t+1}', \ell_t-m_t} - D_{\psi_t}(p_{t+1}', p_t) + G_{t-1}(p_{t}') - F_t(p_t) \right) + G_{t'}(p_{t'+1}') - G_0(p_1') \\
    &= \sum_{t=1}^{t'} \left(\langle p_t, \ell_t\rangle - \inner{p_{t+1}', \ell_t-m_t} - D_{\psi_t}(p_{t+1}', p_t) -\psi_t(p_t) + \psi_{t-1}(p_t) - \inner{p_t,m_t} \right) + G_{t'}(p_{t'+1}') - G_0(p_1')\\
    &\leq \sum_{t=1}^{t'} \left(\langle p_t - p_{t+1}', \ell_t-m_t\rangle - D_{\psi_t}(p_{t+1}', p_t) - \psi_t(p_t) + \psi_{t-1}(p_t) \right) + G_{t'}(u) - \min_{p\in\Omega} \psi_0(p)   \tag{by \eqref{eq: eq2}, using that $p'_{t'+1}$ is the minimizer of $G_{t'}$} \\
    &= \sum_{t=1}^{t'} \left(\max_{p\in\Omega}\left\{\langle p_t - p, \ell_t - m_t\rangle - D_{\psi_t}(p, p_t)\right\} - \psi_t(p_t) + \psi_{t-1}(p_t) \right) \\
    &\qquad + \sum_{t=1}^{t'}\langle u, \ell_t\rangle + \psi_{t'}(u)  - \min_{p\in\Omega} \psi_0(p) \\
    &= \sum_{t=1}^{t'} \left(\max_{p\in\Omega}\left\{\langle p_t - p, \ell_t - m_t\rangle - D_{\psi_t}(p, p_t)\right\} - \psi_t(p_t) + \psi_{t-1}(p_t) \right) \\
    &\qquad + \sum_{t=1}^{t'}\langle u, \ell_t\rangle + \psi_{t'}(u)  - \min_{p\in\Omega} \psi_0(p) \\
    &= \sum_{t=1}^{t'} \left(\max_{p\in\Omega}\left\{\langle p_t - p, \ell_t - m_t\rangle - D_{\psi_t}(p, p_t)\right\} + (\psi_t(u) - \psi_{t-1}(u) -\psi_t(p_t) + \psi_{t-1}(p_t)) \right) \\
    &\qquad + \sum_{t=1}^{t'}\langle u, \ell_t\rangle + \psi_0(u)  - \min_{p\in\Omega} \psi_0(p) 
\end{align*}
Re-arranging finishes the proof. 
\end{proof}

\begin{lemma}\label{lem: hybrid FTRL}
    Consider the optimistic FTRL algorithm with bonus $b_t\geq \mathbf{0}$: 
    \begin{align*}
       p_t = \argmin_{p\in\Omega}\left\{ \left\langle p, \sum_{\tau=1}^{t-1}(\ell_\tau - b_\tau) \right\rangle + m_t + \psi_t(p) \right\} 
    \end{align*}
with $\psi_t(p) = \frac{1}{\eta_t}\psi^{\ts}(p) + \frac{1}{\beta} \psi^{\lo}(p)$ where $\psi^{\ts}(p)=\frac{-1}{1-\alpha}\sum_{i} p_i^\alpha$ for some $\alpha\in(0,1)$, $\psi^{\lo}(p) = \sum_i \ln\frac{1}{p_i}$, and $\eta_t$ is non-increasing. We have 
    \begin{align*}
        \sum_{t=1}^{t'}\langle p_t-u, \ell_t \rangle  &\leq \frac{1}{1-\alpha}\frac{K^{1-\alpha}}{\eta_0} + \frac{1}{1-\alpha}\sum_{t=1}^{t'} \left(\frac{1}{\eta_t} - \frac{1}{\eta_{t-1}}\right)\left(\sum_{i=1}^K p_{t,i}^{\alpha} - 1\right) + \frac{K\ln \frac{1}{\delta}}{\beta}  + \frac{1}{\alpha}\sum_{t=1}^{t'}\eta_t \sum_{i=1}^K p_{t,i}^{2-\alpha}(\ell_{t,i}^{\ts}-x_t)^2  \\
        &\qquad + 2\beta \sum_{t=1}^{t'}\sum_{i=1}^K p_{t,i}^2 (\ell^{\lo}_{t,i}-w_t)^2 + \left(1+\frac{1}{\alpha}\right) \sum_{t=1}^{t'} \inner{p_t,b_t} -  \sum_{t=1}^{t'} \inner{u, b_t} + \delta \sum_{t=1}^{t'} \left\langle -u + \frac{1}{K}\one, \ell_t-b_t\right\rangle.  
    \end{align*}
    for any $\delta\in(0,1)$ and any $\ell^{\ts}_t, b^{\ts}_t\in\mathbb{R}^K$, $\ell^{\lo}_t, b^{\lo}_t\in\mathbb{R}^K$, $x_t, w_t\in\mathbb{R}$ such that 
    \begin{align}
        \ell_t^{\ts} + \ell_t^{\lo} &= \ell_t - m_t, \\
        b_t^{\ts} + b_t^{\lo} &= b_t \\
        \eta_t p_{t,i}^{1-\alpha} (\ell_{t,i}^{\ts} - x_t) &\geq -\frac{1}{4},  \label{eq: tsallis condition} \\
        \beta p_{t,i} (\ell_{t,i}^{\lo} - w_t) &\geq -\frac{1}{4}, \label{eq: logb condition} \\
        0\leq \eta_t p_{t,i}^{1-\alpha} b_{t,i}^{\ts} &\leq \frac{1}{4},  \label{eq: tsallis condition bonus} \\
        0\leq \beta p_{t,i} b_{t,i}^{\lo} &\leq \frac{1}{4} \label{eq: logb condition bonus}
    \end{align}
    for all $t,i$. 
\end{lemma}
\begin{proof}
    Let $u'=(1-\delta)u + \frac{\delta}{K}\one$. By  \pref{lem: FTRL}, we have 
    \begin{align*}
        \sum_{t=1}^{t'} \inner{p_t - u', \ell_t - b_t} &\leq \frac{1}{1-\alpha}\frac{1}{\eta_0}\sum_{i=1}^K p^\alpha_{1,i}+ \frac{1}{1-\alpha}\sum_{t=1}^{t'}\left(\frac{1}{\eta_t} - \frac{1}{\eta_{t-1}}\right)\sum_{i=1}^K  \left(p_{t,i}^\alpha-u_i'^\alpha\right) + \frac{1}{\beta}\sum_{i=1}^K \ln\frac{p_{1,i}}{u_{i}'}\\
        & \qquad + \sum_{t=1}^{t'} \max_{p}\left( \inner{p_t - p, \ell_t - b_t -m_t} - \frac{1}{\eta_t} D_{\psi^{\ts}}(p,p_t) - \frac{1}{\beta} D_{\psi^{\lo}}(p,p_t)\right) \\
        &\leq \frac{1}{1-\alpha}\frac{K^{1-\alpha}}{\eta_0}+ \frac{1}{1-\alpha}\sum_{t=1}^{t'}\left(\frac{1}{\eta_t} - \frac{1}{\eta_{t-1}}\right)\left(\sum_{i=1}^K  p_{t,i}^\alpha-1\right) + \frac{K\ln \frac{1}{\delta}}{\beta} \\
        &\qquad + \sum_{t=1}^{t'}  \underbrace{\max_p \left( \inner{p_t-p, \ell^{\ts} - x_t\one} - \frac{1}{2\eta_t}D_{\psi^{\ts}}(p,p_t) \right)}_{\textbf{stability-1}} \\ 
        &\qquad + \sum_{t=1}^{t'}  \underbrace{\max_p \left( \inner{p_t-p, - b_t^{\ts}\one} - \frac{1}{2\eta_t}D_{\psi^{\ts}}(p,p_t) \right)}_{\textbf{stability-2}} \\
        &\qquad + \sum_{t=1}^{t'}  \underbrace{\max_p \left( \inner{p_t-p, \ell^{\lo} - w_t\one} - \frac{1}{2\beta}D_{\psi^{\lo}}(p,p_t) \right)}_{\textbf{stability-3}} \\
        &\qquad + \sum_{t=1}^{t'}  \underbrace{\max_p \left( \inner{p_t-p, -b_{t}^{\lo}\one} - \frac{1}{2\beta}D_{\psi^{\lo}}(p,p_t) \right)}_{\textbf{stability-4}}
    \end{align*}
    where in the last inequality we use $\inner{p_t-p, \one}=0$. 
    
    By Problem~1 in \cite{luo2022homework3}, we can bound $\textbf{stability-1}$ by $\frac{\eta_t}{\alpha} \sum_{i=1}^K p_{t,i}^{2-\alpha}(\ell_{t,i}^{\ts}-x_t)^2$ under the condition \eqref{eq: tsallis condition}. Similarly, $\textbf{stability-2}$ can be upper bounded by $\frac{\eta_t}{\alpha} \sum_{i=1}^K p_{t,i}^{2-\alpha}b^{\ts 2}_{t,i}\leq \frac{1}{\alpha}\sum_{i=1}^K p_{t,i}b^{\ts}_{t,i}$ under the condition \eqref{eq: tsallis condition bonus}.  Using \pref{lem: stabiity hw 2}, we can bound $\textbf{stability-3}$ by $2\beta \sum_{i=1}^K p_{t,i}^2 (\ell^{\lo}_{t,i}-w_t)^2$ under the condition \eqref{eq: logb condition}. Similarly, we can bound $\textbf{stability-4}$ by $2\beta \sum_{i=1}^K p_{t,i}^2 b^{\lo 2}_{t,i}\leq \sum_{i=1}^K p_{t,i}b^{\lo}_{t,i}$ under \eqref{eq: logb condition bonus}. Collecting all terms and using the definition of $u'$ finishes the proof. 
    
\end{proof}

\begin{lemma}\label{lem: logbarrier analysis}
    Consider the optimistic FTRL algorithm with bonus $b_t\geq \mathbf{0}$: 
    \begin{align*}
       p_t = \argmin_{p\in\Omega}\left\{ \left\langle p, \sum_{\tau=1}^{t-1}(\ell_\tau - b_\tau) \right\rangle + m_t + \psi_t(p) \right\} 
    \end{align*}
    with $\psi_t(p) = \frac{1}{\eta_t}\sum_i \ln\frac{1}{p_i}$, and $\eta_t$ is non-increasing. We have 
    \begin{align*}
        \sum_{t=1}^{t'}\langle p_t-u, \ell_t \rangle 
        &\leq \frac{K\ln \frac{1}{\delta}}{\eta_{t'}} + 2\sum_{t=1}^{t'} \eta_t \sum_{i=1}^K p_{t,i}^2 (\ell_{t,i}-m_{t,i}-x_t)^2 \\
        &\qquad + 2\sum_{t=1}^{t'} \inner{p_{t}, b_t} - \sum_{t=1}^{t'}\inner{u,b_t} + \delta \sum_{t=1}^{t'} \left\langle -u + \frac{1}{K}\one, \ell_t-b_t\right\rangle.  
    \end{align*}
    for any $\delta\in(0,1)$ and any $x_t\in\mathbb{R}$ if the following hold: 
    $
        \eta_t p_{t,i} (\ell_{t,i}-m_{t,i} - x_t) \geq -\frac{1}{4}
    $ and $\eta_t p_{t,i}b_{t,i}\leq \frac{1}{4}$
    for all $t,i$. 
\end{lemma}
\begin{proof}
    Let $u'=(1-\delta)u + \frac{\delta}{K}\one$. By \pref{lem: FTRL} and letting $\eta_0=\infty$, we have 
    \begin{align*}
        \sum_{t=1}^{t'} \inner{p_t - u', \ell_t-b_t} &\leq \sum_{t=1}^{t'}\sum_{i=1}^K  \left(\ln\frac{p_{t,i}}{u_i'}\right)\left(\frac{1}{\eta_t} - \frac{1}{\eta_{t-1}}\right) + \sum_{t=1}^{t'} \max_{p}\left( \inner{p_t - p, \ell_t-b_t-m_t} -  D_{\psi_t}(p,p_t)\right) \\
        &\leq  \frac{K\ln \frac{K}{\delta}}{\eta_{t'}} + \sum_{t=1}^{t'}  \underbrace{\max_p \left( \inner{p_t-p, \ell_t - m_t - x_t\one} - \frac{1}{2}D_{\psi_t}(p,p_t) \right)}_{\textbf{stability-1}} \\
        &\qquad + \sum_{t=1}^{t'}  \underbrace{\max_p \left( \inner{p_t-p, -b_t} - \frac{1}{2}D_{\psi_t}(p,p_t) \right)}_{\textbf{stability-2}}
    \end{align*}
    where in the last inequality we use $\inner{p_t-p, \one}=0$. 
    
    Using \pref{lem: stabiity hw 2}, we can bound $\textbf{stability-1}$ by $2\eta_t \sum_{i=1}^K p_{t,i}^2 (\ell_{t,i}-m_{t,i}-x_t)^2$, and bound $\textbf{stability-2}$ by $2\eta_t \sum_{i=1}^K p_{t,i}^2 b_{t,i}^2\leq \sum_{i=1}^K p_{t,i}b_{t,i}$ under the specified conditions. Collecting all terms and using the definition of $u'$ finishes the proof. 
\end{proof}

\begin{lemma}[Stability under log barrier]\label{lem: stabiity hw 2}
    Let $\psi(p)=\frac{1}{\beta}\sum_i \ln\frac{1}{p_i}$, and let $ \ell_t\in\mathbb{R}^K$ be such that $\beta p_i\ell_{t,i}\geq -\frac{1}{2}$. Then 
    \begin{align*}
        \max_{p\in\Delta([K])}\left\{\langle p_t-p,\ell_t\rangle -  D_{\psi}(p,p_t)\right\}\leq \sum_i \beta p_{t,i}^{2}\ell_{t,i}^2. 
    \end{align*}
\end{lemma}

\begin{proof}
    \begin{align*}
        &\max_{p\in\Delta([K])}\left\{ \langle p_t - p,\ell_t\rangle - D_{\psi}(p,p_t)\right\} 
        \leq \max_{q\in\mathbb{R}^K_+}\left\{ \langle p_t - q,\ell_t\rangle - D_{\psi}(q,p_t)\right\}
    \end{align*}
    Define $f(q)=\langle p_t - q,\ell_t\rangle - D_{\psi}(q,p_t)$. Let $q^\star$ be the solution in the last expression. Next, we verify that under the specified conditions, we have $\nabla f(q^\star)=0$. It suffices to show that there exists $q\in\mathbb{R}^K_+$ such that $\nabla f(q)=0$ since if such $q$ exists, then it must the maximizer of $f$ and thus $q^\star=q$. 
    
    \begin{align*}
        [\nabla f(q)]_i = -\ell_{t,i} - [\nabla \psi(q)]_i + [\nabla \psi(p_t)]_i = -\ell_{t,i} + \frac{1}{\beta q_i} - \frac{1}{\beta p_{t,i}}
    \end{align*}
    By the condition, we have $-\frac{1}{\beta p_{t,i}}-\ell_{t,i} < 0$ for all $i$. and so $\nabla f(q)=\mathbf{0}$ has solution in $\mathbb{R}_+$, which is $q_i=\left(\frac{1}{p_{t,i}}+\eta_{t,i}\ell_{t,i}\right)^{-1}$. 
    
    
    Therefore, $\nabla f(q^\star)= -\ell_t - \nabla\psi_t(q^\star) + \nabla \psi_t(p_t)=0$, and we have 
    \begin{align*}
        \max_{q\in\mathbb{R}^K_+}\left\{ \langle p_t - q,\ell_t\rangle - D_{\psi}(q,p_t)\right\} = \langle p_t - q^\star, \nabla \psi(p_t) - \nabla\psi(q^\star)\rangle - D_{\psi}(q^\star,p_t) = D_{\psi}(p_t, q^\star). 
    \end{align*}
    It remains to bound $D_{\psi}(p_t, q^\star)$, which by definition can be written as 
    \begin{align*}
        D_{\psi}(p_t, q^\star) = \sum_{i} \frac{1}{\beta}h\left(\frac{p_{t,i}}{q^\star_i}\right)
    \end{align*}
    where $h(x)=x-1-\ln(x)$. By the relation between $q^\star_i$ and $p_{t,i}$ we just derived, it holds that $\frac{p_{t,i}}{q^\star_i}=1+\beta p_{t,i}\ell_{t,i}$. By the fact that $\ln(1+x)\geq x-x^2$ for all $x\geq -\frac{1}{2}$, we have 
    \begin{align*}
        h\left(\frac{p_{t,i}}{q^\star_i}\right)=\beta p_{t,i}\ell_{t,i} - \ln(1+\beta p_{t,i}\ell_{t,i}) \leq \beta^2p_{t,i}^2\ell_{t,i}^2
    \end{align*}
    which gives the desired bound. 
    
\end{proof}

\begin{lemma}[Stability under negentropy]
\label{lem: exp3 stab}
If $\psi$ is the negentropy $\psi(p)=\sum_{i}p_i\log p_i$ and $\ell_{t,i}>0$, then for any $\eta > 0$ the stability is bounded by
\begin{align*}
    \max_{p\in\mathbb{R}_{\geq 0}^K}\langle p_t-p, \ell_t\rangle - \frac{1}{\eta}D_{\psi}(p, p_t) \leq \frac{\eta}{2}\sum_i p_{t,i}\ell_{t,i}^2\,.
\end{align*}
If $\ell_t>-\frac{1}{\eta}$, then the stability is bounded by 
\begin{align*}
    \max_{p\in\mathbb{R}_{\geq 0}^K}\langle p_t-p, \ell_t\rangle - \frac{1}{\eta}D_{\psi}(p, p_t) \leq \eta\sum_i p_{t,i}\ell_{t,i}^2\,.
\end{align*}
instead.  
\end{lemma}
\begin{proof}
    Let $f_i(p_i) = (p_{t,i}-p_i)\ell_{ti} - \frac{1}{\eta}\left(p_i(\log p_i-1) - p_i\log p_{t,i}+ p_{t,i}\right) $.
    Then we maximize $\sum_{i}f_i(p_{i})$, which is upper bounded by maximizing the expression in each coordinate for $p_i\geq 0$.
    We have
    \begin{align*}
        f_i'(p) = -\ell_{t,i}-\frac{1}{\eta}(\log p-\log p_{t,i}),
    \end{align*}
    and hence the maximum is obtained for $p^\star = p_{t,i}\exp(-\eta\ell_{t,i})$.
    Plugging this in leads to
    \begin{align*}
        f_i(p^\star) &= p_{t,i}\ell_{t,i}(1-\exp(-\eta\ell_{t,i}))-\frac{1}{\eta}\left(-p_{t,i}\exp(-\eta\ell_{t,i})\eta\ell_{t,i}+p_{t,i}(1-\exp(-\eta\ell_{t,i}))\right)\\
        &=p_{t,i}\ell_{t,i}-\frac{p_{t,i}}{\eta}(1-\exp(-\eta\ell_{t,i}))\leq p_{t,i}\ell_{t,i}-\frac{p_{t,i}}{\eta}\left(\eta\ell_{t,i}-\frac{1}{2}\eta^2\ell_{t,i}^2\right)=\frac{\eta}{2}p_{t,i}\ell_{t,i}^2\,, 
    \end{align*}
    for non-negative $\ell_t$ and 
    \begin{align*}
        f_i(p^*)\leq \eta p_{t,i}\ell_{t,i}^2
    \end{align*}
    for $\ell_{t,i}\geq -1/\eta$ respectively, where we used the bound 
$\exp(-x)\leq 1-x+\frac{x^2}{2}$ which holds for any $x\geq0$ and $\exp(-x)\leq 1-x+x^2$, which holds for $x\geq -1$ respectively.
    Summing over all coordinates finishes the proof.
\end{proof}
\begin{lemma}[Stability of Tsallis-INF]
\label{lem: tsallis stab}
For the potential $\psi(p)=-\sum_i \frac{p_i^\alpha}{\alpha(1-\alpha)}$, any $p_t\in\Delta([K])$, any non-negative loss $\ell_t\geq 0$ and any positive learning rate $\eta>0$, we have
\begin{align*}
    \max_p\inner{p-p_t,\ell_t}-\frac{1}{\eta}D_{\psi}(p,p_t) \leq \frac{\eta}{2}\sum_{i}p_i^{2-\alpha}\ell_{t,i}^2\,.
\end{align*}
\end{lemma}
\begin{proof}
We upper bound this term by maximizing over $p\geq 0$ instead of $p\in\Delta([K])$.
Since $\ell_t$ is positive, the optimal $p^\star$ satisfies $p^\star_i\leq p_i$ in all components. We have
\begin{align*}
    D_\psi(p^\star,p_t) = \sum_{i=1}^K\int_{p_{t,i}}^{p^\star_i}\frac{p_{t,i}^{\alpha-1}-p^{\alpha-1}}{1-\alpha}\,dp= \sum_{i=1}^K\int_{p_{t,i}}^{p^\star_i}\int_{p_{t,i}}^{\tildep}p^{\alpha-2}\,dp\,d\tildep&\geq \sum_{i=1}^K\int_{p_{t,i}}^{p^\star_i}\int_{p_{t,i}}^{\tildep}p_{t,i}^{\alpha-2}\,dp\,d\tildep\\
    &=\sum_{i=1}^K\frac{1}{2}(p^{\star}_i-p_{t,i})^2p_{t,i}^{\alpha-2}\,.
\end{align*}
By AM-GM inequality, we further have
\begin{align*}
    \max_p\inner{p-p_t,\ell_t}-\frac{1}{\eta}D_{\psi}(p,p_t) \leq\max_p\sum_{i}\left((p_i-p_{t,i})\ell_{t,i}-\frac{(p_i-p_{t,i})^2p_{t,i}^{\alpha-2}}{2\eta}\right)\leq\frac{\eta}{2}\sum_{i}p_{t,i}^{2-\alpha}\ell_{t,i}^2\,.
\end{align*}
\end{proof}

\section{Analysis for Variance-Reduced SCRiBLe (\pref{alg: VR scrible} / \pref{thm: VR scrible})}
\label{app: SCrible}
\begin{algorithm}[H]
    \caption{VR-SCRiBLe} \label{alg: VR scrible}
    \textbf{Define}: Let $\psi(\cdot)$ be a $\nu$-self-concordant barrier of $\text{conv}(\calX)\subset\mathbb{R}^d$. 
    \\
    \For{$t=1, 2, \ldots$}{
        Receive $m_t\in\mathbb{R}^d$. Compute
        \begin{align}
            w_t = \argmin_{w\in\text{conv}(\calX)}\left\{ \left\langle w, \sum_{\tau=1}^{t-1} \hatell_\tau + m_t \right\rangle + \frac{1}{\eta_t}\psi(w)\right\}  \label{eq: scrible procedure}
        \end{align}
        where
        \begin{align*}
            \eta_t = \min\left\{\frac{1}{16d}, \sqrt{\frac{\nu \log T}{\sum_{\tau=1}^{t-1}\|\hatell_\tau-m_\tau\|^2_{H_\tau^{-1}}}}\right\}.
        \end{align*}
        where $H_t=\nabla^2\psi(w_t)$. \\ 
        Sample $s_t$ uniformly from $\mathbb{S}_d$ (the unit sphere in $d$-dimension). \\
        Define 
        \begin{align*}
            w_t^{+} = w_t + H_t^{-\frac{1}{2}}s_{t}, \qquad 
            w_t^{-} = w_t - H_t^{-\frac{1}{2}}s_{t}. 
        \end{align*}
        Find distributions $q_t^+$ and $q_t^-$ over actions such that 
        \begin{align*}
            w_t^+ = \sum_{x\in\calX} q_{t,x}^+ x, \qquad w_t^- = \sum_{x\in\calX} q_{t,x}^- x
        \end{align*}
        Sample $A_t\sim q_t\triangleq \frac{q_t^+ + q_t^-}{2}$, receive $\ell_{t,A_t}\in[-1,1]$ with $\E[\ell_{t,x}]=\inner{x,\ell_t}$, and define 
        \begin{align*}
            \hatell_t = d(\ell_{t,A_t}-m_{t,A_t})\left(\frac{q_{t,A_t}^+-q_{t,A_t}^-}{q_{t,A_t}^++q_{t,A_t}^-}\right)H_t^{\frac{1}{2}}s_{t} + m_t
        \end{align*}
        where $m_{t,x}:=\inner{x,m_t}$. 
    }
\end{algorithm}

\begin{lemma}\label{lem: key lem scrible}
    In \pref{alg: VR scrible}, we have 
    $\E\left[\hatell_t\right] = \ell_t$ and 
    \begin{align*}
        &\E\left[\left\|\hatell_t-m_t\right\|_{\nabla^{-2}\psi(w_t)}^2\right]\leq d^2 \E\left[\sum_{x\in\calX} \min\{p_{t,x}, 1-p_{t,x}\}(\ell_{t,x}-m_{t,x})^2\right].  
    \end{align*}
    where $p_{t,x}$ is the probability of choosing action $x$ in round $t$. 
\end{lemma}
\begin{proof}
\begin{align*}
    \E\left[\hatell_t\right] 
    &= \E\left[ d (\ell_{t,A_t}-m_{t,A_t})\left(\frac{q_{t,A_t}^+-q_{t,A_t}^-}{q_{t,A_t}^++q_{t,A_t}^-}\right)H_t^{\frac{1}{2}}s_t + m_t\right] \\
    &= \E\left[ d \E\left[(\ell_{t,A_t}-m_{t,A_t})\left(\frac{q_{t,A_t}^+-q_{t,A_t}^-}{q_{t,A_t}^++q_{t,A_t}^-}\right)~\bigg|s_t\right]H_t^{\frac{1}{2}}s_t + m_t\right] \tag{note that $q_t^+, q_t^-$ depend on $s_t$} \\
    &= \E\left[ d \E\left[ \sum_{x}q_{t,x}(\ell_{t,x}-m_{t,x})\left(\frac{q_{t,x}^+-q_{t,x}^-}{q_{t,x}^++q_{t,x}^-}\right)~\bigg|s_t\right]H_t^{\frac{1}{2}}s_t + m_t\right] \\
    &= \E\left[ d \E\left[ \sum_{x}\inner{x,\ell_t-m_t}\left(\frac{q_{t,x}^+-q_{t,x}^-}{2}\right)~\bigg|s_t\right]H_t^{\frac{1}{2}}s_t + m_t\right]\\
    &= \E\left[ d \E\left[ \frac{\inner{w_t^+-w_t^-,\ell_t-m_t}}{2}~\bigg|s_t\right]H_t^{\frac{1}{2}}s_t + m_t\right] \\
    &= \E\left[ d \inner{H^{-\frac{1}{2}}s_t,\ell_t-m_t}  H_t^{\frac{1}{2}}s_t + m_t\right]\\ 
    &= \E\left[ d   H_t^{\frac{1}{2}}s_ts_t^\top H_t^{-\frac{1}{2}}(\ell_t-m_t) + m_t\right] \\
    &=\ell_t.   \tag{because $\E[s_ts_t^\top]=\frac{1}{d}I$}
\end{align*} 
\begin{align*}
    \E\left[\left\|\hatell_t-m_t\right\|_{\nabla^{-2}\psi(w_t)}^2\right] 
    &= \E\left[d^2(\ell_{t,A_t}-m_{t,A_t})^2\left(\frac{q_{t,A_t}^+-q_{t,A_t}^-}{q_{t,A_t}^++q_{t,A_t}^-}\right)^2\left\|H_t^{\frac{1}{2}}s_t\right\|_{H_t^{-1}}^2\right]\\
    &= \E\left[d^2(\ell_{t,A_t}-m_{t,A_t})^2\left(\frac{q_{t,A_t}^+-q_{t,A_t}^-}{q_{t,A_t}^++q_{t,A_t}^-}\right)^2\right]\\
    &= \E\left[~\E\left[d^2(\ell_{t,A_t}-m_{t,A_t})^2\left(\frac{q_{t,A_t}^+-q_{t,A_t}^-}{q_{t,A_t}^++q_{t,A_t}^-}\right)^2~\Bigg|~s_t\right]~\right] \\
    &\leq \E\left[~\E\left[\sum_{x}q_{t,x}d^2(\ell_{t,x}-m_{t,x})^2\left|\frac{q_{t,x}^+-q_{t,x}^-}{q_{t,x}^++q_{t,x}^-}\right|~\Bigg|~s_t\right]~\right] 
\end{align*}
For any $x$, we have $q_{t,x}\left|\frac{q_{t,x}^+-q_{t,x}^-}{q_{t,x}^++q_{t,x}^-}\right|\leq q_{t,x}$ and 
\begin{align*}
    q_{t,x}\left|\frac{q_{t,x}^+-q_{t,x}^-}{q_{t,x}^++q_{t,x}^-}\right| = \frac{|q_{t,x}^+-q_{t,x}^-|}{2} \leq 1 - \frac{q_{t,x}^++q_{t,x}^-}{2} = 1-q_{t,x}. 
\end{align*}
Therefore, we continue to bound $\E\left[\left\|\hatell_t-m_t\right\|_{\nabla^{-2}\psi(w_t)}^2\right] $ by
\begin{align*}
    &\E\left[~\E\left[ \sum_{x\in\calX} \min\{q_{t,x}, 1-q_{t,x}\}d^2(\ell_{t,x}-m_{t,x})^2 ~\Bigg|~s_t\right]~\right]  \\
    &\leq \E\left[~\sum_{x\in\calX} \min\big\{p_{t,x}, 1-p_{t,x}\big\}d^2(\ell_{t,x}-m_{t,x})^2 ~\right].  \tag{$\E[\min(\cdot,\cdot)]\leq \min(\E[\cdot], \E[\cdot])$ and $p_{t,x}= \E[\E[q_{t,x}~|~s_t]]$}  
\end{align*}
\end{proof}

\begin{lemma}\label{lem: stability scrible}
    If $\eta_t\leq \frac{1}{16d}$, 
    then $\max_w\left(\inner{w_t - w, \hatell_t-m_t} - \frac{1}{\eta_t}D_\psi(w,w_t)\right)\leq 8\eta_t \|\hatell_t-m_t\|_{\nabla^{-2}\psi(w_t)}^2$. 
\end{lemma}
\begin{proof}
    We first show that $\eta_t\|\hatell_t-m_t\|_{\nabla^{-2}\psi(w_t)}\leq \frac{1}{16}$. By the definition of $\hatell_t$, we have 
    \begin{align*}
        \eta_t \left\|\hatell_t-m_t\right\|_{\nabla^{-2}\psi(w_t)} 
        &\leq \frac{1}{16d}\cdot d\|H_t^{\frac{1}{2}}s_t\|_{H_t^{-1}} \leq \frac{1}{16}. 
    \end{align*}
    Define 
    \begin{align*}
        F(w) &= \inner{w_t - w, \hatell_t-m_t} - \frac{1}{\eta_t}D_\psi(w,w_t).  
    \end{align*}
     Define $\lambda=\|\hatell_t-m_t\|_{\nabla^{-2}\psi(w_t)}$.  Let $w'$ be the maximizer of $F$. it suffices to show $\|w'-w_t\|_{\nabla^{2}\psi(w_t)}\leq 8\eta_t\lambda$ because this leads to 
     \begin{align*}
         F(w')\leq \|w'-w_t\|_{\nabla^{2}\psi(w_t)}\|\hatell_t-m_t\|_{\nabla^{-2}\psi(w_t)} \leq 8\eta_t\lambda^2. 
     \end{align*}
     To show $\|w'-w_t\|_{\nabla^{2}\psi(w_t)}\leq 8\eta_t\lambda$, it suffices to show that for all $u$ such that $\|u-w_t\|_{\nabla^{2}\psi(w_t)}=8\eta_t\lambda$, $F(u)\leq 0$. To see why this is the case, notice that $F(w_t)=0$, and $F(w')\geq 0$ because $w'$ is the maximizer of $F$. Therefore, if $\|w'-w_t\|_{\nabla^{2}\psi(w_t)}> 8\eta_t\lambda$, then there exists $u$ in the line segment between $w_t$ and $w'$ with $\|u-w_t\|_{\nabla^{2}\psi(w_t)}=8\eta_t\lambda$ such that $F(u)\leq 0\leq \min\{F(w_t), F(w')\}$, contradicting that $F$ is strictly concave. 
    
   Below, consider any $u$ with $\|u-w_t\|_{\nabla^{2}\psi(w_t)}=8\eta_t\lambda$. By Taylor expansion, there exists $u'$ in the line segment between $u$ and $w_t$ such that
    \begin{align*}
        F(u) &\leq \|u-w_t\|_{\nabla^{2}\psi(w_t)}\|\hatell_t-m_t\|_{\nabla^{-2}\psi(w_t)} - \frac{1}{2\eta_t}  \|u-w_t\|_{\nabla^2\psi(u')}^2. 
    \end{align*}
    Because $\psi$ is a self-concordant barrier and that $\|u'-w_t\|_{\nabla^2\psi(w_t)}\leq \|u-w_t\|_{\nabla^2\psi(w_t)}= 8\eta_t\lambda\leq \frac{1}{2}$, we have $\nabla^2\psi(u')\succeq \frac{1}{4}\nabla^2\psi(w_t)$. Continuing from the previous inequality, 
    \begin{align*}
        F(u) &\leq \|u-w_t\|_{\nabla^{2}\psi(w_t)}\|\hatell_t-m_t\|_{\nabla^{-2}\psi(w_t)} - \frac{1}{8\eta_t}  \|u-w_t\|_{\nabla^2\psi(w_t)}^2 = 8\eta_t\lambda \cdot \lambda - \frac{(8\eta_t\lambda)^2}{8\eta_t} = 0. 
    \end{align*}
    This concludes the proof. 
\end{proof}

\begin{proof}[Proof of \pref{thm: VR scrible}] 
    By the standard analysis of optimistic-FTRL and \pref{lem: stability scrible}, we have for any $u$, 
    \begin{align}
        &\E\left[\sum_{t=1}^T (\ell_{t,A_t}-\ell_{t,u})\right]   \nonumber \\
        &\leq O\left(\E\left[\frac{\nu\log T}{\eta_T} + \sum_{t=1}^T \eta_t \left\|\hatell_t-m_t\right\|_{\nabla^{-2} \psi(w_t)}^2 \right]\right)   \nonumber \\
        &= O\left(\sqrt{\nu\log T\E\left[ \sum_{t=1}^T \left\|\hatell_t-m_t\right\|_{\nabla^{-2} \psi(w_t)}^2 \right]} + d\nu\log (T)\right). \tag{by the tuning of $\eta_t$} \\
        &\label{eq: scrible regret tmp}
    \end{align}
    In the adversarial regime, using \pref{lem: key lem scrible}, we continue to bound \eqref{eq: scrible regret tmp} by 
    \begin{align*}
        O\left(d\sqrt{\nu\log T\E\left[ \sum_{t=1}^T (\ell_{t,A_t}-m_{t,A_t})^2 \right]} + d\nu\log (T)\right)
    \end{align*}
    When losses are non-negative and $m_t=\mathbf{0}$, we can further upper bound it by 
    \begin{align*}
        O\left(d\sqrt{\nu\log T\E\left[ \sum_{t=1}^T \ell_{t,A_t} \right]} + d\nu\log (T)\right). 
    \end{align*}
    Then solving the inequality for $\E\left[\sum_{t=1}^T \ell_{t,A_t}\right]$, we can further get the first-order regret bound of 
    \begin{align*}
        \E\left[\sum_{t=1}^T (\ell_{t,A_t}-\ell_{t,u})\right]\leq O\left(d\sqrt{\nu\log T\E\left[ \sum_{t=1}^T \ell_{t,u} \right]} + d\nu\log (T)\right). 
    \end{align*}
    In the corrupted stochastic setting, using \pref{lem: key lem scrible}, we continue to bound \eqref{eq: scrible regret tmp} by 
    \begin{align*}
        &O\left(d\sqrt{\nu\log(T) \E\left[\sum_{t=1}^T \left((1-p_{t,u})(\ell_{t,u}-m_{t,u})^2 + \sum_{x\neq u} p_{t,x}(\ell_{t,x}-m_{t,x})^2\right]\right)}\right) \\
        &\leq O\left(d\sqrt{\nu\log(T)\E\left[\sum_{t=1}^T (1-p_{t,u})\right]}\right). 
    \end{align*}
    Then by the self-bounding technique stated in \pref{prop: standard algorithm}, we can further bound the regret in the corrupted stochastic setting by 
    \begin{align*}
        O\left(\frac{d^2\nu\log(T)}{\Delta}+ \sqrt{\frac{d^2\nu\log(T)}{\Delta}C}\right). 
    \end{align*}
\end{proof}

\section{Analysis for LSB Log-Determinant FTRL (\pref{alg: logdet} / \pref{lem: CA logdet guarantee})} \label{app: log det analysis}

\begin{algorithm}[H]
     \caption{LSB-logdet}\label{alg: logdet}
     \textbf{Input}: $\calX(=\{\binom{x}{0}\}),\hatx=\binom{\bm{0}}{1}$.  \\
     \textbf{Define}: Let $\pi_\calX$ be John's exploration over $\calX$. $H(\alpha):= \E_{x\sim \alpha}[xx^\top]$\,,\  $\mu_\alpha:=\E_{x\sim \alpha}[x]$. \\
     \For{$t=1, 2, \ldots$}{
          Let 
          \begin{align*}
              &\eta_{t} = \min\left\{\frac{1}{4d},\sqrt{\frac{\log(T)}{\sum_{\tau=1}^{t-1}(1-p_{\tau,\hatx})}}\right\}\,,\\
              &p_t := \argmin_{\alpha\in \Delta(\calX\cup\{\hatx\})}\left\langle\mH{\alpha},\sum_{\tau=1}^{t-1}\hatell_\tau\right\rangle - \frac{1}{\eta_t}\log\det\left(H(\alpha)-\mH{\alpha}\mH{\alpha}^\top\right)\,,\\
              &\tildep_t := (1-\eta_t d)p_t + \eta_td((1-p_{t,\hatx})\pi_{\calX}+p_{t,\hatx}\pi_{\hatx} )
          \end{align*}
          where $\pi_{\hatx}$ denotes the sampling distribution that picks $\hatx$ with probability 1. \\
          Sample an action $A_t\sim \tildep_t$.  \\
          Construct loss estimator: 
          \begin{align*}
              \hatell_t(a) =  a^\top\Big(H(\tildep_t)-\mH{\tildep_t}\mH{\tildep_t}^\top\Big)^{-1}\big(a_t-\mH{\tildep_t}\big) \ell_t(a_t)\,. 
          \end{align*}
     }
\end{algorithm}

We begin by using the following simplifying notation.
For a distribution $\alpha \in \Delta(\calX\cup\{\hatx\})$, we define $\alpha_{\calX}$ as the restricted distribution over $\calX$ such that $\alpha_{\calX} \propto \alpha$ over $\calX$, i.e. $\alpha_{\calX,x} = \frac{\alpha_x}{1-\alpha_{\hatx}}$ for any $x\in\calX$. We further define
\begin{align*}
    &H(\alpha) = \E_{x\sim \alpha}[xx^\top]\,,\quad
    \mH{\alpha} = \E_{x\sim\alpha}[x],\\
    &\varH{\alpha} = H(\alpha)-\mH{\alpha}\mH{\alpha}^\top,\\
    &G(\alpha) = H(\alpha_{\calX}),\quad \mG{\alpha} = \mH{\alpha_\calX}, \\
    &\varG{\alpha} = G(\alpha)-\mG{\alpha}\mG{\alpha}^\top\,.
\end{align*}
\begin{lemma}\label{lem: logdet lemma 1}
Assume $\alpha\in\Delta(\calX\cup\{\hatx\}))$ is such that $\varG{\alpha}$ is of rank $d-1$ (i.e. full rank over $\calX$). Then 
we have the following properties: 
\begin{align*}
    &\varH{\alpha}=(1-\alpha_{\hatx})\left(\varG{\alpha} +\alpha_{\hatx}(\mG{\alpha}-\hatx)(\mG{\alpha}-\hatx)^\top\right),\\
    &\varHinv{\alpha}=\frac{1}{1-\alpha_{\hatx}}\left(\varGinv{\alpha}+\varGinv{\alpha}\mG{\alpha}\hatx^\top+\hatx \mG{\alpha}^\top\varGinv{\alpha}+\left(\norm{\mG{\alpha}}^2_{\varGinv{\alpha}}+\frac{1}{\alpha_{\hatx}}\right)\hatx\hatx^\top\right)\,,
\end{align*}
where $\varGinv{\alpha}$ denotes the pseudo-inverse.
\end{lemma}
\begin{proof}
The first identity is a simple algebraic identity
\begin{align*}
    \varH{\alpha} &= H(\alpha)-\mH{\alpha}\mH{\alpha}^\top=(1-\alpha_{\hatx})G(\alpha)+\alpha_{\hatx}\hatx\hatx^{\top}-((1-\alpha_{\hatx})\mG{\alpha}+\alpha_{\hatx}\hatx)((1-\alpha_{\hatx})\mG{\alpha}+\alpha_{\hatx}\hatx)^\top\\
    &=(1-\alpha_{\hatx})\left(\varG{\alpha} +\alpha_{\hatx}(\mG{\alpha}-\hatx)(\mG{\alpha}-\hatx)^\top\right)\,,
\end{align*}
which holds for any $\alpha$. For the second identity note that
by the definition of $\hatx=\binom{\bm{0}}{1}$ and $\forall x\in\calX:\,\inner{x,\hatx}=0$, we have $\varG{\alpha}\hatx=\varGinv{\alpha}\hatx=0$. Furthermore $\varG{\alpha}\varGinv{\alpha}=I-\hatx\hatx^{-1}$ due to the rank $d-1$ assumption.
Multiplying the two matrices yields
\begin{align*}
    &\left(\varG{\alpha} +\alpha_{\hatx}(\mG{\alpha}-\hatx)(\mG{\alpha}-\hatx)^\top\right)\cdot\left(\varGinv{\alpha}+\varGinv{\alpha}\mG{\alpha}\hatx^\top+\hatx \mG{\alpha}^\top\varGinv{\alpha}+\left(\norm{\mG{\alpha}}^2_{\varGinv{\alpha}}+\frac{1}{\alpha_{\hatx}}\right)\hatx\hatx^\top\right)\\
    &=I-\hatx\hatx^{\top}+\mG{\alpha}\hatx^{\top}+\alpha_{\hatx}(\mG{\alpha}-\hatx)\left(\varGinv{\alpha}\mG{\alpha}+\norm{\mG{\alpha}}^2_{\varGinv{\alpha}}\hatx-\varGinv{\alpha}\mG{\alpha}-\left(\norm{\mG{\alpha}}^2_{\varGinv{\alpha}}+\frac{1}{\alpha_{\hatx}}\right)\hatx\right)^\top\\
    &=I-\hatx\hatx^{\top}+\mG{\alpha}\hatx^{\top}-(\mG{\alpha}-\hatx)\hatx^\top = I
\end{align*}
\end{proof}
which implies for any $x,y\in\operatorname{span}(\calX)$
\begin{align}
&\inner{x,\varHinv{\alpha}y } = \frac{\inner{x,\varGinv{\alpha}y}}{1-\alpha_{\hatx}},\label{eq: H by G}\\
&\inner{x,\varHinv{\alpha}(\hatx-\mG{\alpha}) } = 0, \label{eq: orthogonal}
\\
    &\norm{\hatx-\mG{\alpha}}^2_{\varHinv{\alpha}} = \frac{1}{(1-\alpha_{\hatx})\alpha_{\hatx}}.\label{eq: hatx norm}
\end{align}
\begin{lemma}
\label{lem: succ of mix}
Let $\pi_1,\pi_2\in\Delta(\calX\cup\{\hatx\})$, then for any $\lambda\in[0,1]$:
\begin{align*}
    \varH{\lambda\pi_1+(1-\lambda)\pi_2} \succeq \lambda \varH{\pi_1}\,.
\end{align*}
\end{lemma}
\begin{proof}
Simple algebra shows
\begin{align*}
    \varH{\lambda\pi_1+(1-\lambda)\pi_2}=\lambda \varH{\pi_1}+(1-\lambda) \varH{\pi_2}+\lambda(1-\lambda)(\mH{\pi_1}-\mH{\pi_2} )(\mH{\pi_1}-\mH{\pi_2} )^\top\,.
\end{align*}
\end{proof}
\begin{lemma}
\label{lem: mixed norm}
Let $\pi\in\Delta(\calX)$ be arbitrary, and $\tilde\pi = (1-\eta_td)\pi+\eta_td \pi_{\calX}$ for $(\eta_td)\in(0,1)$, then it holds for any $x\in\calX$:
\begin{align*}
    &\norm{\mG{\pi}-\mG{\tilde\pi}}_{\varGinv{\tilde\pi}} \leq \sqrt{\frac{\eta_td}{1-\eta_td}}\\
    &\norm{x-\mG{\tilde\pi}}_{\varGinv{\tilde\pi}} \leq \frac{2}{ \sqrt{\eta_t}}
    \end{align*}
\end{lemma}
\begin{proof}
We have
\begin{align*}
    \varG{\tildepi} = \eta_td\varG{\pi_{\calX}}+(1-\eta_td)\varG{\pi}+\eta_td(1-\eta_td)(\mG{\pi}-\mG{\pi_{\calX}})(\mG{\pi}-\mG{\pi_{\calX}})^\top\,,
\end{align*}
hence
\begin{align*}
    \norm{\mG{\pi}-\mG{\tilde\pi}}_{\varGinv{\tilde\pi}}^2 &= (\eta_td)^2\norm{\mG{\pi}-\mG{\pi_{\calX}}}_{\left[\eta_td\varG{\pi_{\calX}}+(1-\eta_td)\varG{\pi}+\eta_td(1-\eta_td)(\mG{\pi}-\mG{\pi_{\calX}})(\mG{\pi}-\mG{\pi_{\calX}})^\top\right]^{-1}}^2\\
    &\leq (\eta_td)^2\norm{\mG{\pi}-\mG{\pi_{\calX}}}^2_{\left[\eta_td(1-\eta_td)(\mG{\pi}-\mG{\pi_{\calX}})(\mG{\pi}-\mG{\pi_{\calX}})^\top\right]^{+}} = \frac{\eta_td}{1-\eta_td}\,.
\end{align*}
For the second inequality, we have
\begin{align*}
    \norm{x-\mG{\tilde\pi}}_{\varGinv{\tilde\pi}}&\leq \norm{x-\mG{\pi_{\calX}}}_{\varGinv{\tilde\pi}}+\norm{\mG{\tilde\pi}-\mG{\pi_{\calX}}}_{\varGinv{\tilde\pi}}\\
    &\leq \frac{1}{\sqrt{\eta_td}}\left(\norm{x-\mG{\pi_{\calX}}}_{\varGinv{\pi_{\calX}}}+\norm{\mG{\tilde\pi}-\mG{\pi_{\calX}}}_{\varGinv{\pi_{\calX}}}\right)\leq \frac{2}{\sqrt{\eta_t}}\,,
\end{align*}
where the last inequality uses that John's exploration satisfies $\norm{x-\mG{\pi_{\calX}}}^2_{\varGinv{\pi_{\calX}}}\leq d$ for all $x\in\calX$.
\end{proof}
\begin{lemma}
The Bregman divergence between two distributions $\alpha,\beta$ over $\calX\cup\{\hatx\}$ with respect to the function $F(\alpha)=-\log\det\left(\varH{\alpha}\right)$ is bounded by
\begin{align*}
    D(\alpha,\beta) \geq D_{\log}(\alpha_{\hatx},\beta_{\hatx})+D_{\log}(1-\alpha_{\hatx},1-\beta_{\hatx}) +\frac{1-\alpha_{\hatx}}{1-\beta_{\hatx}}\norm{\mG{\alpha}-\mG{\beta}}^2_{\varGinv{\beta}}
\end{align*}
where $D_{\log}$ is the Bregman divergence of $-\log(x)$. 
\end{lemma}
\begin{proof}
We begin by simplifying $F(\alpha)$. Note that 
\begin{align*}
    \varH{\alpha} = (1-\alpha_{\hatx})\left(\varG{\alpha}+\alpha_{ \hatx}\hatx^{\top}\right)^{\frac{1}{2}}M\left(\varG{\alpha}+\alpha_{\hatx}\hatx^{\top}\right)^{\frac{1}{2}}\,,\\
    M=I+(\sqrt{\alpha_{\hatx}}\varGinv{\alpha}\mG{\alpha}-\hatx)(\sqrt{\alpha_{\hatx}}\varGinv{\alpha}\mG{\alpha}-\hatx))^\top-\hatx\hatx^{\top}. 
\end{align*}
$M$ is a matrix with $d-2$ eigenvalues of size $1$, since it is the identity with two rank-1 updates. The product of the remaining eigenvalues is given by considering the determinant of the $2\times 2$ sub-matrix with coordinates $\frac{\varGinv{\alpha}\mG{\alpha}}{\norm{\varGinv{\alpha}\mG{\alpha}}}$ and $\hatx$. We have that \begin{align*}
    \det(M) &= \hatx^\top M\hatx \cdot \frac{\mG{\alpha}^\top\varGinv{\alpha}}{\norm{\varGinv{\alpha}\mG{\alpha}}}M\frac{\varGinv{\alpha}\mG{\alpha}}{\norm{\varGinv{\alpha}\mG{\alpha}}} - \left(\hatx^\top M \frac{\varGinv{\alpha}\mG{\alpha}}{\norm{\varGinv{\alpha}\mG{\alpha}}}\right)^2\\
    &= 1\cdot\left(1+\alpha_{\hatx}\norm{\varGinv{\alpha}\mG{\alpha}}^2\right)-\alpha_{\hatx}\norm{\varGinv{\alpha}\mG{\alpha}}^2=1. 
\end{align*}
Hence we have
\begin{align*}
    F(\alpha)=-\log(1-\alpha_{\hatx})-\log(\alpha_{\hatx})-\log{\textstyle\det_{d-1}}((1-\alpha_{\hatx})\varG{\alpha})
\end{align*}
Where $\det_{d-1}$ is the determinant over the first $d-1$ eigenvalues of the submatrix of the first $(d-1)$ coordinates.
The derivative term is given by
\begin{align*}
    &\inner{\alpha-\beta,\nabla F(\beta)}=\sum_{x\in\calX\cup\{\hatx\}}(\alpha_x-\beta_x)\norm{x-\mH{\beta}}^2_{\varHinv{\beta}}\\
    &= \sum_{x\in\calX\cup\{\hatx\}} \alpha_x \norm{x - \mH{\beta}}^2_{\varHinv{\beta}} - d \\
&=  \tr\left(\sum_{x\in\calX\cup\{\hatx\}} \alpha_x (x - \mH{\beta})(x - \mH{\beta})^\top \varHinv{\beta}\right) - d\\
&=  \tr\left(\left(\varH{\alpha} + (\mH{\alpha} - \mH{\beta}) (\mH{\alpha} - \mH{\beta})^\top\right) \varHinv{\beta}\right) - d \\
&=  \tr\left(\varH{\alpha} \varHinv{\beta}\right) + \norm{\mH{\alpha} - \mH{\beta}}^2_{\varHinv{\beta}} - d.
\end{align*}
The first term is
\begin{align*}
        &\tr\left(\varH{\alpha}\varHinv{\beta}\right) \\
        &= \tr\left((1-\alpha_{\hatx})\varG{\alpha} \varHinv{\beta}\right)+\alpha_{\hatx}(1-\alpha_{\hatx})\norm{\hatx-\mG{\alpha}}^2_{\varHinv{\beta}}\tag{by \pref{lem: logdet lemma 1}}\\
        &=\tr_{d-1}\left(\frac{1-\alpha_{\hatx}}{1-\beta_{\hatx}}\varG{\alpha} \varGinv{\beta}\right)+\alpha_{\hatx}(1-\alpha_{\hatx})\left(\norm{\hatx-\mG{\beta}}^2_{\varHinv{\beta}}+\norm{\mG{\alpha}-\mG{\beta}}^2_{\varHinv{\beta}}\right)  \tag{by \pref{lem: logdet lemma 1}}\\
        &=\tr_{d-1}\left(\frac{1-\alpha_{\hatx}}{1-\beta_{\hatx}}\varG{\alpha} \varGinv{\beta}\right)+\frac{\alpha_{\hatx}(1-\alpha_{\hatx})}{\beta_{\hatx}(1-\beta_{\hatx})}+\frac{\alpha_{\hatx}(1-\alpha_{\hatx})\norm{\mG{\alpha}-\mG{\beta}}^2_{\varGinv{\beta}}}{1-\beta_{\hatx}}\,.  \tag{by \eqref{eq: hatx norm}}
    \end{align*}
    The norm term is
    \begin{align*}
        &\norm{\mH{\alpha}-\mH{\beta}}^2_{\varHinv{\beta}}\\
        &=\norm{\mH{\alpha}-(\alpha_{\hatx} \hatx+(1-\alpha_{\hatx})\mG{\beta})}^2_{\varHinv{\beta}}+\norm{(\alpha_{\hatx} \hatx+(1-\alpha_{\hatx})\mG{\beta})-\mH{\beta}}^2_{\varHinv{\beta}}\\
        &= (1-\alpha_{\hatx})^2\norm{\mG{\alpha}-\mG{\beta}}^2_{\varHinv{\beta}}+(\alpha_{\hatx}-\beta_{\hatx})^2\norm{\hatx-\mG{\beta}}^2_{\varHinv{\beta}}\\
        &= \frac{(1-\alpha_{\hatx})^2\norm{\mG{\alpha}-\mG{\beta}}^2_{\varGinv{\beta}}}{1-\beta_{\hatx}}+\frac{(\alpha_{\hatx}-\beta_{\hatx})^2}{\beta_{\hatx}(1-\beta_{\hatx})}\,.
    \end{align*}
Combining both terms
\begin{align*}
    &\inner{\alpha-\beta,\nabla F(\beta)}\\
    &=\tr_{d-1}\left(\frac{1-\alpha_{\hatx}}{1-\beta_{\hatx}}\varG{\alpha} \varGinv{\beta}\right)+\frac{\alpha_{\hatx}(1-\alpha_{\hatx})+(\alpha_{\hatx}-\beta_{\hatx})^2}{\beta_{\hatx}(1-\beta_{\hatx})}+\frac{1-\alpha_{\hatx}}{1-\beta_{\hatx}}\norm{\mG{\alpha}-\mG{\beta}}^2_{\varGinv{\beta}}-d\\
    &=\tr_{d-1}\left(\frac{1-\alpha_{\hatx}}{1-\beta_{\hatx}}\varG{\alpha} \varGinv{\beta}\right)+\frac{\alpha_{\hatx}}{\beta_{\hatx}}+\frac{1-\alpha_{\hatx}}{1-\beta_{\hatx}}-1+\frac{1-\alpha_{\hatx}}{1-\beta_{\hatx}}\norm{\mG{\alpha}-\mG{\beta}}^2_{\varGinv{\beta}}-d\,.
\end{align*}
Combining the everything
\begin{align*}
    D(\alpha,\beta) &= D_{\log}(\alpha_{\hatx},\beta_{\hatx})+D_{\log}(1-\alpha_{\hatx},1-\beta_{\hatx})+\frac{1-\alpha_{\hatx}}{1-\beta_{\hatx}}\norm{\mG{\alpha}-\mG{\beta}}^2_{\varGinv{\beta}} + D_{d-1}\left((1-\alpha_{\hatx})\varG{\alpha},(1-\beta_{\hatx})\varG{\beta}\right)\\
    &\geq D_{\log}(\alpha_{\hatx},\beta_{\hatx})+D_{\log}(1-\alpha_{\hatx},1-\beta_{\hatx})+\frac{1-\alpha_{\hatx}}{1-\beta_{\hatx}}\norm{\mG{\alpha}-\mG{\beta}}^2_{\varGinv{\beta}}\,,
\end{align*}
where the last inequality follows from the positiveness of Bregman divergences.
\end{proof}
\begin{lemma}
\label{lem: logdet}
For any $b\in(0,1)$, any $x\in\bbR$ and $\eta \leq \frac{1}{2|x|}$, it holds that
\begin{align*}
    \sup_{\alpha\in[0,1]} |a-b|x - \frac{1}{\eta}D_{\log}(a,b) \leq \eta b^2 x^2\,.
\end{align*}
\end{lemma}
\begin{proof}
    The statement is equivalent to
    \begin{align*}
      \sup_{a\in[0,1]}f(a)=\sup_{a\in[0,1]} (b-a)x - \frac{1}{\eta}D_{\log}(a,b) \leq \eta b^2 x^2\,,
    \end{align*}
    since $x$ can take positive or negative sign. The function is concave, so setting the derivative to $0$ is the optimal solution if that value lies in $(0,\infty)$.
    \begin{align*}
        f'(a) = -x+\frac{1}{\eta}\left(\frac{1}{a}-\frac{1}{b}\right)\,\qquad a^\star = \frac{b}{1+\eta b x}\,.
    \end{align*}
    Plugging this in, leads to
    \begin{align*}
        f(a^\star) &= \frac{\eta b^2x^2}{1+\eta b x} - \frac{1}{\eta}\left(\log(1+\eta b x)+\frac{1}{1+\eta b x}-1\right)\\
        &\leq\frac{\eta b^2x^2}{1+\eta b x} - \frac{1}{\eta}\left(\eta b x-\eta^2 b^2 x^2-\frac{\eta b x}{1+\eta b x}\right)=\eta b^2x^2\,,
    \end{align*}
    where the last line uses $\log(1+x)\geq x-x^2$ for any $x\geq -\frac{1}{2}$.
\end{proof}
\begin{lemma}
\label{lem: logdet stab}
The stability term
\begin{align*}
    \stab_t := \sup_{\alpha\in\Delta(\calX\cup\{\hatx\})}\inner{\mH{p_t}-\mH{\alpha},\hatell_t}-\frac{1}{\eta_t}D(\alpha,p_t)\,,
\end{align*}
satisfies
\begin{align*}
    \E_t[\stab_t] = O((1-p_{t,\hatx})\eta_t d)\,.
\end{align*}
\end{lemma}
\begin{proof}
\allowdisplaybreaks
We have
\begin{align*}
    \mH{p_t}-\mH{\alpha} &= (p_{t,\hatx}-\alpha_{\hatx})\hatx + (1-p_{t,\hatx})\mG{p_t}-(1-\alpha_{\hatx})\mG{\alpha}\\ &=(p_{t,\hatx}-\alpha_{\hatx})(\hatx-\mG{\tildep_t}+\mG{\tildep_t}-\mG{p_t})+(1-\alpha_{\hatx})(\mG{\tildep_t}-\mG{\alpha})\,.
\end{align*}
Hence for $A_t=y\in\calX$:
    
\begin{align*}
\inner{\mH{p_t}-\mH{\alpha},\hatell_t}
 &=\inner{\mH{p_t}-\mH{\alpha},\varHinv{\tildep_t}(y-\mH{\tildep_t})}\ell_{t,A_t}\\
&=(p_{t,\hatx}-\alpha_{\hatx})\inner{\hatx - \mG{\tildep_t},\varHinv{\tildep_t}(y-\mH{\tildep_t})}\ell_{t,A_t}\\
&\qquad +(p_{t,\hatx}-\alpha_{\hatx})\inner{\mG{\tildep_t}-\mG{p_t},\varHinv{\tildep_t}(y-\mH{\tildep_t})}\ell_{t,A_t}\\
&\qquad+(1-\alpha_{\hatx})\inner{\mG{\tildep_t}-\mG{\alpha},\varHinv{\tildep_t}(y-\mH{\tildep_t})}\ell_{t,A_t}\\
&=(p_{t,\hatx}-\alpha_{\hatx})\inner{\hatx - \mG{\tildep_t},\varHinv{\tildep_t}(\mG{\tildep_t}-\mH{\tildep_t})}\ell_{t,A_t}\tag{By \pref{eq: orthogonal}}\\
&\qquad +(p_{t,\hatx}-\alpha_{\hatx})\inner{\mG{\tildep_t}-\mG{p_t},\varHinv{\tildep_t}(y-\mG{\tildep_t})}\ell_{t,A_t}\\
&\qquad+(1-\alpha_{\hatx})\inner{\mG{\tildep_t}-\mG{\alpha},\varHinv{\tildep_t}(y-\mG{\tildep_t})}\ell_{t,A_t}\\
&\leq\frac{|\alpha_{\hatx}-p_{t,\hatx}|}{1-p_{t,\hatx}} +\frac{|p_{t,\hatx}-\alpha_{\hatx}|}{1-p_{t,\hatx}}|\inner{\mG{\tildep_t}-\mG{p_t},\varHinv{\tildep_t}(y-\mG{\tildep_t})}|\\
&+\quad\frac{1-\alpha_{\hatx}}{1-p_{t,\hatx}}|\inner{\mG{\tildep_t}-\mG{\alpha},\varGinv{\tildep_t}(y-\mG{\tildep_t})}|\tag{By \pref{eq: hatx norm} and \pref{eq: H by G}}\\
&\leq \frac{|p_{t,\hatx}-\alpha_{\hatx}|}{1-p_{t,\hatx}}(1+2\sqrt{d})\\
&\qquad+\frac{4}{3}\times\frac{1-\alpha_{\hatx}}{1-p_{t,\hatx}}\norm{\mG{p_t}-\mG{\alpha}}_{\varGinv{p_t}}\norm{y-\mG{\tildep_t})}_{\varGinv{\tildep_t}} \tag{Cauchy-Schwarz and \pref{lem: mixed norm}, \pref{lem: succ of mix}}
\end{align*}

Equivalently for $A_t=\hatx$,
\begin{align*}
    \inner{\mH{p_t}-\mH{\alpha},\hatell_t}&=\inner{\mH{\tildep_t}-\mH{\alpha},\varHinv{\tildep_t}(\hatx-\mH{\tildep_t})}\ell_{t,A_t}\\
    &=(p_{t,\hatx}-\alpha_{\hatx})\inner{\hatx - \mG{\tildep_t},\varHinv{\tildep_t}(\hatx-\mH{\tildep_t})}\hatell_{t,A_t}\\
    &\leq\frac{|p_{t,\hatx}-\alpha_{\hatx}|}{p_{t,\hatx}}\,.
\end{align*}
Hence the stability term for $A_t\in\calX$, is bounded by
\begin{align*}
&\sup_{\alpha\in\Delta(\calX\cup\{\hatx\})}\inner{\mH{p_t}-\mH{\alpha},\hatell_t}-\frac{1}{\eta_t}D(\alpha,p_t)\\
&\leq \sup_{\alpha\in\Delta(\calX\cup\{\hatx\})} \frac{|\alpha_{\hatx}-p_{t,\hatx}|}{1-p_{t,\hatx} }(1+2\sqrt{d})-\frac{1}{\eta_t}D(1-\alpha_{\hatx},1-p_{t,\hatx}) \\
    &\qquad + \frac{1-\alpha_{\hatx}}{1-p_{t,\hatx}}\left(\frac{4}{3}\norm{\mG{\alpha}-\mG{\tildep_t}}_{\varGinv{\tildep_t}}\norm{y-\mG{\tildep_t}}_{\varGinv{\tildep_t}}-\frac{1}{\eta_t}\norm{\mG{\alpha}-\mG{\tildep_t}}_{\varGinv{\tildep_t}}^2\right)\\
&\leq\eta_t\norm{y-\mG{\tildep_t}}^2_{\varGinv{\tildep_t}}+\sup_{\alpha_{\hatx}\in[0,1]} \frac{|\alpha_{\hatx}-p_{t,\hatx}|}{1-p_{t,\hatx} }(1+2\sqrt{d}) \\
    &\qquad + \eta_t\frac{|\alpha_{\hatx}-p_{t,\hatx}|}{1-p_{t,\hatx}}\norm{y-\mG{\tildep_t}}^2_{\varGinv{\tildep_t}}-\frac{1}{\eta_t}D(1-\alpha_{\hatx},1-p_{t,\hatx})\tag{AM-GM inequality}\\
&\leq\eta_t\norm{y-\mG{\tildep_t}}^2_{\varGinv{\tildep_t}}\\
&\qquad+\sup_{\alpha_{\hatx}\in[0,1]} \frac{|\alpha_{\hatx}-p_{t,\hatx}|}{1-p_{t,\hatx} }(5+2\sqrt{d}) -\frac{1}{\eta_t}D(1-\alpha_{\hatx},1-p_{t,\hatx})\tag{\pref{lem: mixed norm}}\\
&\leq \eta_t\norm{y-\mG{\tildep_t}}^2_{\varGinv{\tildep_t}}+O(\eta_t d)\,.\tag{\pref{lem: logdet}}
\end{align*}
Taking the expectation over $y\sim \tildep_{t,\calX}$ leads to
\begin{align*}
    \E_{A_t\sim\tildep_{t,\calX}}\left[\stab_t\right] = O(\eta_t d)
\end{align*}
For $A_t=\hatx$ we have two cases.
If $p_{t,\hatx} \geq \frac{1}{2}$, then by \pref{lem: logdet}, 
\begin{align*}
    \stab_t &\leq \sup_{\alpha_{\hatx}\in[0,1]} \frac{|p_{t,\hatx}-\alpha_{\hatx}|}{p_{t,\hatx}} - \frac{1}{\eta_t}D_{\log}(1-\alpha_{\hatx},1-p_{t,\hatx})
    \leq O\left(\eta_t(1-p_{t,\hatx})^2\times \frac{1}{p_{t,\hatx}^2}\right)\leq O\left(\eta_t(1-p_{t,\hatx})\right)\,.
\end{align*}
Otherwise If $1- p_{t,\hatx} \geq \frac{1}{2}$, then by \pref{lem: logdet}, 
\begin{align*}
    \stab_t &\leq \sup_{\alpha_{\hatx}\in[0,1]} \frac{|p_{t,\hatx}-\alpha_{\hatx}|}{p_{t,\hatx}} - \frac{1}{\eta_t}D_{\log}(\alpha_{\hatx},p_{t,\hatx})
    \leq O\left(\eta_t p_{t,\hatx}^2\times \frac{1}{p_{t,\hatx}^2} \right)\leq O(\eta_t(1-p_{t,\hatx}))\,.
\end{align*}
Finally we have 
\begin{align*}
    \E_t\left[\stab_t\right]=(1-p_{t,\hatx})\E_{A_t\sim p_{t,\calX}}[\stab_t]+p_{t,\hatx}\E_{A_t\sim \pi_{\hatx}}[\stab_t] = O((1-p_{t,\hatx})\eta_t d)
\end{align*}
\end{proof}

\begin{proof}[Proof of \pref{lem: CA logdet guarantee}] 
By standard analysis of FTRL (\pref{lem: FTRL}) and \pref{lem: logdet stab}, for any $\tau$ and $x$
\begin{align*}
    \sum_{t=1}^\tau\E_{t}\left[\inner{p_t,\hatell_t}-\hatell_t(x)\right]\leq \frac{d\log T}{\eta_\tau} + \sum_{t=1}^\tau O(\eta_t(1-p_{t,\hatx})d) \leq O\left(\sqrt{d\log T\sum_{t=1}^\tau (1-\tildep_{t,\hatx})}\right)\,,
\end{align*}
where the last inequality follows from the definition of learning rate and $\tildep_{t,\hatx}=p_{t,\hatx}$.
Additionally, we have
\begin{align*}
    \E_t\left[\inner{\tildep_{t}-p_t,\hatell_t} \right]=\inner{\tildep_{t}-p_t,\ell_t} \leq \eta_t d(1-p_{t,\hatx})\,,
\end{align*}
so that 
\begin{align*}
    \sum_{t=1}^\tau\E_{t}\left[\inner{\tildep_t,\hatell_t}-\hatell_t(x)\right]= O\left(\sqrt{d\log T\sum_{t=1}^\tau (1-\tildep_{t,\hatx})}\right)\,.
\end{align*}
Taking expectations on both sides finishes the proof.
\end{proof}

\section{Analysis for the First Reduction}\label{app: main theorem proof}

\subsection{BOBW to LSB (\pref{alg: adaptive alg2} / \pref{thm: candidate aware to bob})}\label{app: proof of first reduction}
\begin{proof}[Proof of \pref{thm: candidate aware to bob}]
We use $\tau_k=T_{k+1}-T_k$ to denote the length of epoch $k$, and let $n$ be the last epoch (define $T_{n+1}=T$). Also, we use $\E_t[\cdot]$ to denote expectation conditioned on all history up to time $t$. 
We first consider the adversarial regime. By the definition of local-self-bounding algorithms, we have for any $u$, 
\begin{align*}
    \E_{T_k}\left[ \sum_{t=T_{k}+1}^{T_{k+1}} \left(\ell_{t,A_t} - \ell_{t,u}\right)  \right] \leq c_0^{1-\alpha} \E_{T_k}[\tau_k]^\alpha +c_2\log(T),  
\end{align*}
which implies 
\begin{align*}
    \E\left[ \sum_{t=T_{k}+1}^{T_{k+1}} \left(\ell_{t,A_t} - \ell_{t,u}\right)  \right] \leq c_0^{1-\alpha}\E[\tau_k]^\alpha + c_2\log(T)  
\end{align*}
using the property $\E[x^\alpha] \leq \E[x]^\alpha$ for $x\in\mathbb{R}_+$ and $0<\alpha<1$. 
Summing the bounds over $k=1, 2,\ldots, n$, and using the fact that $\tau_{k}\geq 2\tau_{k-1}$ for all $k<n$, we get 
\begin{align}
    &\E\left[\sum_{t=1}^T \left(\ell_{t,A_t} - \ell_{t,u}\right)\right] \nonumber   \\
    &\leq c_0^{1-\alpha}\left(\E\left[\tau_{n}\right]^\alpha + \E\left[\tau_{n-1}\right]^\alpha\left(1+\frac{1}{2^\alpha} + \frac{1}{2^{2\alpha}} + \cdots\right)\right)+c_2\log(T)\log_2\left(\frac{T}{c_2\log(T)}\right)   \nonumber \\
    &\leq O\left(c_0^{1-\alpha}T^\alpha+c_2\log^2(T)\right).   \label{eq: reduction 1 tmp}
\end{align} 
The same analysis also gives 
\begin{align}
    \E\left[\sum_{t=1}^T \left(\ell_{t,A_t} - \ell_{t,u}\right)\right] \leq O\left((c_1\log T)^{1-\alpha}T^{\alpha} + c_2\log^2(T)\right).   \label{eq: sto bound 1}
\end{align}
Next, consider the corrupted stochastic regime. We first argue that it suffices to consider the regret comparator $x^\star$. This is because if $x^\star\neq \argmin_{u\in\calX}\E\left[\sum_{t=1}^T \ell_{t,u}\right]$, then it holds that $C\geq T\Delta$. Then the right-hand side of \eqref{eq: sto bound 1} is further upper bounded by 
\begin{align*}
    O\left((c_1\log T)^{1-\alpha}(C\Delta^{-1})^\alpha + c_2\log(T)\log(C\Delta^{-1})\right), 
\end{align*}
which fulfills the requirement for the stochastic regime. Below, we focus on bounding the pseudo-regret with respect to $x^\star$.

Let $m = \max\{k\in\mathbb{N}\,|\,\hatx_k \neq x^\star\}$. Notice that $m$ is either the last epoch (i.e., $m=n$) or the second last (i.e., $m=n-1$), because for any two consecutive epochs, at least one of them must have $\hatx_k\neq x^\star$. 
Below we show that $|\{t\in [T_{m+1}]\,|\,A_t\neq x^\star\}|\geq \frac{T_{m+1}}{8}-2\tau_0$. 

If $\tau_{m} > 2 \tau_{m-1}$, by the fact that the termination condition was not triggered one round earlier, the number of plays $N_m(x^\star)$ in epoch $m$ is at most $\frac{\tau_{m}-1}{2}+1 = \frac{\tau_{m}+1}{2}$; in other words, the number of times $\sum_{t=T_{m}+1}^{T_{m+1}}\ind\{A_t\neq x^\star\}$ is at least $\frac{\tau_{m}-1}{2}$. Further notice that because $\tau_m > 2\tau_{m-1}\geq 4\tau_{m-2} > \cdots$, we have $\frac{\tau_m-1}{2}\geq \frac{1}{4}\sum_{k=1}^m \tau_k - \frac{1}{2} = \frac{T_{m+1}}{4} - \frac{1}{2}$.

Now consider the case $m>1$ and $\tau_{m} \leq 2\tau_{m-1}$. Recall that $\hatx_m$ is the action with $N_{m-1}(\hatx_m)\geq \frac{\tau_{m-1}}{2}$. 
This implies that $\sum_{t=T_{m-1}+1}^{T_{m}}\ind\{A_t\neq x^\star\}\geq \frac{\tau_{m-1}}{2} \geq \frac{1}{2}\max\left\{\frac{\tau_m}{2}, \frac{1}{2}\sum_{k=1}^{m-1}\tau_k\right\} \geq \frac{1}{8}\sum_{k=1}^m \tau_k = \frac{T_{m+1}}{8}$.
  
Finally, consider the case $m=1$ and $\tau_1\leq 2\tau_0$, then we have $T_2-2\tau_0\leq 0$ and the statement holds trivially.

The regret up to and including epoch $m$ can be lower and upper bounded using the self-bounding technique: 
\begin{align*}
    &\E\left[\sum_{t=1}^{T_{m+1}} (\ell_{t,A_t} - \ell_{t,x^\star})\right]
    = (1+\lambda)\E\left[\sum_{t=1}^{T_{m+1}} (\ell_{t,A_t} - \ell_{t,x^\star})\right] - \lambda\E\left[\sum_{t=1}^{T_{m+1}} (\ell_{t,A_t} - \ell_{t,x^\star})\right]  \tag{for $0\leq \lambda \leq 1$}\\
    &\leq O\left((c_1\log T)^{1-\alpha}\E\left[T_{m+1} \right]^{\alpha} +c_2\log(T)\log\left(\frac{\E[T_{m+1}]}{c_2\log(T)}\right)\right) - \lambda\left(\left(\frac{1}{8}\E\left[T_{m+1}\right]-c_2\log(T)\right)\Delta - C\right)  \tag{the first term is by a similar calculation as \eqref{eq: reduction 1 tmp}, but replacing $T$ by $T_m$ and $c_0$ by $c_1\log T$ }  \\
    &\leq O\left( c_1\log(T)\Delta^{-\frac{\alpha}{1-\alpha}} + (c_1\log T)^{1-\alpha}\left(C\Delta^{-1}\right)^{\alpha}+c_2\log(T)\log(C\Delta^{-1})\right)
\end{align*}
where in the last inequality we use \pref{lem: simple lem}. 

If $m$ is not the last epoch, then it holds that $\hatx_{n}=x^\star$ for the final epoch $n$. In this case, the regret in the final epoch is 
\begin{align*}
    &\E\left[\sum_{t=T_n+1}^{T} (\ell_{t,A_t} - \ell_{t,x^\star})\right] 
    = (1+\lambda)\E\left[\sum_{t=T_n+1}^{T} (\ell_{t,A_t} - \ell_{t,x^\star})\right] - \lambda\E\left[\sum_{t=T_n+1}^{T} (\ell_{t,A_t} - \ell_{t,x^\star})\right] \\
    &\leq O\left((c_1\log T)^{1-\alpha} \E\left[\sum_{t=T_n+1}^T (1-p_{t,x^\star})\right]^\alpha\right)- \lambda\left(\E\left[\sum_{t=T_n+1}^T (1-p_{t,x^\star}) \right]\Delta - C\right)+c_2\log(T) \\
    &\leq O\left( c_1\log(T)\Delta^{-\frac{\alpha}{1-\alpha}} + (c_1\log T)^{1-\alpha}\left(C\Delta^{-1}\right)^{\alpha}\right)+c_2\log(T). \tag{\pref{lem: simple lem}}
\end{align*}
\end{proof}

\begin{lemma}\label{lem: simple lem}
    For $\Delta\in (0,1]$ and $c,c',X\geq 1$ and $C\geq 0$, we have 
    \begin{align*}
        \min_{\lambda\in [0,1]}\left\{\frac{1}{2}c^{1-\alpha}X^\alpha + \frac{1}{2}c'\log X - \lambda(X\Delta-C)\right\}\leq c\Delta^{-\frac{\alpha}{1-\alpha}} + 2c^{1-\alpha}\left(\frac{C}{\Delta}\right)^\alpha + 2c'\log\left(1+\frac{c'+C}{\Delta}\right)
    \end{align*}
\end{lemma}
\begin{proof}
    If $c^{1-\alpha}X^\alpha \geq c'\log T$, we have 
    \begin{align*}
        \frac{1}{2}c^{1-\alpha} X^\alpha + \frac{1}{2}c'\log X 
        \leq c^{1-\alpha} X^\alpha  
        \leq  \lambda X\Delta + c\lambda^{-\frac{\alpha}{1-\alpha}}\Delta^{-\frac{\alpha}{1-\alpha}}. 
    \end{align*}
    where the last inequality is by the weighted AM-GM inequality. 
    Therefore, 
    \begin{align*}
        \min_{\lambda\in [0,1]}\left\{\frac{1}{2}c^{1-\alpha}X^\alpha + \frac{1}{2}c'\log X - \lambda(X\Delta-C)\right\}\leq \min_{\lambda\in[0,1]}  c\lambda^{-\frac{\alpha}{1-\alpha}}\Delta^{-\frac{\alpha}{1-\alpha}} + C\lambda. 
    \end{align*}
    Choosing $\lambda=\min\{1, c^{1-\alpha}C^{-(1-\alpha)}\Delta^{-\alpha}\}$, we bound the last expression by 
    \begin{align*}
        &c\max\left\{1, \left(c^{1-\alpha}C^{-(1-\alpha)}\Delta^{-\alpha}\right)^{-\frac{\alpha}{1-\alpha}}\right\}\Delta^{-\frac{\alpha}{1-\alpha}} + c^{1-\alpha}C^\alpha\Delta^{-\alpha} \\
        &\leq c\max\left\{1, c^{-\alpha}C^\alpha\Delta^{\frac{\alpha^2}{1-\alpha}}\right\}\Delta^{-\frac{\alpha}{1-\alpha}}+ c^{1-\alpha}C^\alpha\Delta^{-\alpha} \\
        &\leq c\Delta^{-\frac{\alpha}{1-\alpha}} + 2c^{1-\alpha}C^\alpha\Delta^{-\alpha}. 
    \end{align*}
    If $c^{1-\alpha}X^\alpha \leq c'\log T$, we have 
    \begin{align*}
        \frac{1}{2}c^{1-\alpha} X^\alpha + \frac{1}{2}c'\log X - \lambda(X\Delta -C)
        \leq c'\log X - \lambda X\Delta + \lambda C 
        \leq c'\log\left(1+\frac{c'^2}{\lambda^2\Delta^2}\right) + \lambda C
    \end{align*}
    where the last inequality is because if $X\geq \frac{c'^2}{\lambda^2\Delta^2}$ then $c'\log X-\lambda X\Delta \leq c'\log X - c'\sqrt{X}<0$. 
    Choosing $\lambda=\min\{1,\frac{c'}{C}\}$, we bound the last expression by $c'\log(1+(c'^2+C^2)/\Delta^2)\leq 2c'\log(1+(c'+C)/\Delta)$. Combining cases finishes the proof. 
\end{proof}

\subsection{dd-BOBW to dd-LSB (\pref{alg: adaptive alg2} / \pref{thm: data-dependent reduction })}\label{app: data dep first red}
\begin{proof}[Proof of \pref{thm: data-dependent reduction }]
In the adversarial regime, we have that the regret in each phase $k$ is bounded by 
\begin{align*}
    \E_{T_k}\left[\sum_{t=T_{k}+1}^{T_{k+1}} (\ell_{t,A_t}-\ell_{t,u})\right] \leq \sqrt{c_1\log(T)\E_{T_k}\left[\sum_{t=T_{k}+1}^{T_{k+1}}\sum_{x}p_{t,x}\xi_{t,x}\right]}+c_2\log T\,.
\end{align*}
We have maximally $\log T$ episodes, since the length doubles every time. Via Cauchy-Schwarz, we get
\begin{align*}
    \sum_{k=1}^{k_{\max}}\E_{T_k}\left[\sum_{t=T_{k}+1}^{T_{k+1}} (\ell_{t,A_t}-\ell_{t,u})\right] \leq \sqrt{c_1\log T\sum_{k=1}^{k_{\max}}\E_{T_k}\left[\sum_{t=T_{k}+1}^{T_{k+1}}\sum_{x}p_{t,x}\xi_{t,x}\right]}\sqrt{\log T}+c_2\log^2 T\,.
\end{align*}
Taking the expectation on both sides and the tower rule of expectations finishes the bound for the adversarial regime.
For the stochastic regime, note that $\xi_{t,x} \leq 1$ and hence
\begin{align*}
    &\sum_{x} p_{t,x}\xi_{t,x} -\bm{\mathbb{I}\{u=\hatx\}}p_{t,u}^2\xi_{t,u}\\ 
    &\leq 1-\bm{\mathbb{I}\{u=\hatx\}}p_{t,u}^2\\
    &=(1-\bm{\mathbb{I}\{u=\hatx\}}p_{t,u})(1+\bm{\mathbb{I}\{u=\hatx\}}p_{t,u})\leq 2(1-\bm{\mathbb{I}\{u=\hatx\}}p_{t,u})\,.
\end{align*}
dd-LSB implies regular LSB (up to a factor of $2$) and hence the stochastic bound of regular LSB applies. 
\end{proof}

\section{Analysis for the Second Reduction}
\subsection{$\frac{1}{2}$-LSB to $\frac{1}{2}$-iw-stable (\pref{alg: corral} / \pref{thm: basic corral thm})} \label{app: main corral theorem proof}
\begin{proof}[Proof of \pref{thm: basic corral thm}]
    The per-step bonus $b_t=B_t-B_{t-1}$ is the sum of two terms: 
    \begin{align*}
        b_t^{\ts} &= \sqrt{c_1\sum_{\tau=1}^t \frac{1}{q_{\tau,2}}} - \sqrt{c_1\sum_{\tau=1}^{t-1} \frac{1}{q_{\tau,2}}}\leq  \frac{\frac{c_1}{q_{t,2}}}{\sqrt{c_1\sum_{\tau=1}^t \frac{1}{q_{\tau,2}}}} \leq \sqrt{\frac{c_1}{q_{t,2}}}, \\
        b_t^{\lo} &= \frac{c_2}{\min_{\tau\leq t} q_{\tau,2}} - \frac{c_2}{\min_{\tau\leq t-1} q_{\tau,2}} = \frac{c_2}{q_{t,2}}\left(1-\frac{\min_{\tau\leq t} q_{\tau,2}}{\min_{\tau\leq t-1} q_{\tau,2}}\right). 
    \end{align*}
    Since $\frac{\barq_{t,2}}{q_{t,2}}\leq 2$, using the inequalities above, we have 
    \begin{align}
        \eta_t \sqrt{\barq_{t,2}} b_t^{\ts} &\leq \eta_t \sqrt{\frac{\barq_{t,2}}{q_{t,2}}c_1}\leq \eta_t \sqrt{2c_1}\leq \frac{1}{4}. \label{eq: satisfy 1} \\ 
        \beta \barq_{t,2} b_t^{\lo} &\leq \beta \frac{\barq_{t,2}}{q_{t,2}} c_2\leq 2\beta c_2 \leq  \frac{1}{4}.  \label{eq: satisfy 2}
    \end{align}
    By \pref{lem: hybrid FTRL} and that $\frac{\barq_{t,2}}{q_{t,2}}\leq 2$, 
    we have for any $u$,  
    \begin{align}
        &\sum_{t=1}^{t'} \inner{q_t-u, z_t} \leq \underbrace{\sum_{t=1}^{t'} \inner{q_t-\barq_{t}, z_t}}_{\term_1} \nonumber \\
        &\ \  + O\bigg(\sqrt{c_1}+\underbrace{ \sum_{t=1}^{t'} \frac{\sqrt{q_{t,2}}}{\sqrt{t}}  }_{\term_2} + \underbrace{\frac{\log t'}{\beta}}_{\term_3} +  \underbrace{\sum_{t=1}^{t'}\eta_t \min_{|\theta_t|\leq 1} q_{t,i}^{\frac{3}{2}}(z_{t,i}-\theta_t)^2}_{\term_4} +  \underbrace{\sum_{t=1}^{t'} q_{t,2}b_t^{\ts}}_{\term_5} + \underbrace{\sum_{t=1}^{t'} q_{t,2}b_t^{\lo}}_{\term_6}\bigg) - u_2\sum_{t=1}^{t'}b_t.  \label{eq: corral decompose}
    \end{align}

    We bound individual terms below: 
    \begin{align*}
        \E[\term_1]=\E\left[\sum_{t=1}^{t'} \inner{q_t-\barq_{t}, z_t}\right] \leq O\left(\sum_{t=1}^{t'}\frac{1}{t^2}\right) = O(1). 
    \end{align*}
    \begin{align*}
        \term_2&\leq O\left( \min\left\{\sqrt{t'},\ \ \sqrt{\sum_{t=1}^{t'} q_{t,2}\log T}\right\}\right). \\
        \term_3&=\frac{\log t'}{\beta} \leq O\left(c_2 \log T\right).
    \end{align*}
    
    \begin{align*}
        \E\left[\term_4\right]&=\E\left[\sum_{t=1}^{t'}\eta_t \min_{\theta_t\in[-1,1]} q_{t,i}^{\frac{3}{2}}(z_{t,i}-\theta_t)^2\right] \\
        &\leq \E\left[\sum_{t=1}^{t'} \eta_t \sum_{i=1}^2 q_{t,i}^{\frac{3}{2}}\left(z_{t,i}-\ell_{t,A_t}\right)^2 \right] \\
        &= \E\left[\sum_{t=1}^{t'} \eta_t\sum_{i=1}^2\frac{1}{\sqrt{q_{t,i}}}\left( \ind[i_t=i]\ell_{t,A_t}-q_{t,i}\ell_{t,A_t}\right)^2 \right] \\
        &\leq \E\left[\sum_{t=1}^{t'}\eta_t \sum_{i=1}^2 \left(\sqrt{q_{t,i}}(1-q_{t,i})^2 + (1-q_{t,i})q_{t,i}^\frac{3}{2}\right) \right] \\
        &\leq O\left(\E\left[\sum_{t=1}^{t'} \eta_t \sqrt{q_{t,2}}  \right]\right) \leq O\left(\E[\term_2]\right). 
    \end{align*}
    \allowdisplaybreaks
    \begin{align}
            \term_5=\sum_{t=1}^{t'}q_{t,2}b_t^{\ts} 
            &\leq \sum_{t=1}^{t'} q_{t,2}\left(\frac{\frac{c_1}{q_{t,2}}}{\sqrt{c_1\sum_{\tau=1}^t \frac{1}{q_{\tau,2}}}}\right) \nonumber \\
            &=\sqrt{c_1}\sum_{t=1}^{t'} \frac{\frac{1}{\sqrt{q_{t,2}}}}{\sqrt{\sum_{\tau=1}^{t} \frac{1}{q_{\tau,2}}}}\times\sqrt{q_{t,2}} \label{eq: continue from} \\
            &\leq  \sqrt{c_1}\sqrt{\sum_{t=1}^{t'} \frac{\frac{1}{q_{t,2}}}{\sum_{\tau=1}^t \frac{1}{q_{\tau,2}}}}\sqrt{ \sum_{t=1}^{t'}q_{t,2}}    \tag{Cauchy-Schwarz}\\ 
            &\leq  \sqrt{c_1}\sqrt{1+\log\left(\sum_{t=1}^{t'}\frac{1}{q_{t,2}}\right)}\sqrt{ \sum_{t=1}^{t'}q_{t,2}} \nonumber \\
            &\leq O\left(\sqrt{c_1\sum_{t=1}^{t'} q_{t,2} \log T } \right).  \nonumber 
    \end{align}
    Continuing from \eqref{eq: continue from}, we also have $\term_5\leq \sum_{t=1}^{t'}q_{t,2}b_t^{\ts}\leq \sqrt{c_1}\sum_{t=1}^{t'}\frac{1}{\sqrt{t}}\leq 2\sqrt{c_1t'}$ because $q_{\tau,2}\leq 1$. 
    \begin{align*}
            \term_6=\sum_{t=1}^{t'}q_{t,2}b_t^{\lo}
            &=c_2\sum_{t=1}^{t'} \left(1-\frac{\min_{\tau\leq t}q_{\tau,2}}{\min_{\tau\leq t-1}q_{\tau,2}}\right) \leq  c_2 \sum_{t=1}^{t'}\log\left(\frac{\min_{\tau\leq t-1}q_{\tau,2}}{\min_{\tau\leq t}q_{\tau,2}}\right) \leq O\left(c_2\log T\right). 
    \end{align*}
    Using all bounds above in \eqref{eq: corral decompose}, we can bound $\E\left[\sum_{t=1}^{t'} \langle q_t - u, z_t\rangle\right]$ by 
    \begin{align*}
        \underbrace{O\left(\min\left\{\sqrt{c_1 \E[t']},\ \ \sqrt{c_1\left[\sum_{t=1}^{t'}q_{t,2}\right]\log T}\right\}+c_2\log T\right)}_{\textbf{pos-term}}  - u_2\underbrace{\E\left[ \sqrt{c_1\sum_{t=1}^{t'} \frac{1}{q_{t,2}}} + \frac{c_2}{\min_{t\leq t'} q_{t,2}}\right]}_{\textbf{neg-term}}.
    \end{align*}
    For comparator $\hatx$, we choose $u=\mathbf{e}_1$ and bound $\E\left[\sum_{t=1}^{t'} (\ell_{t,A_t} - \ell_{t,\hatx})\right]$ by the \textbf{pos-term} above. 
    For comparator $x\neq \hatx$, we first choose $u=\mathbf{e}_2$ and upper bound $\E\left[\sum_{t=1}^{t'} (\ell_{t,A_t} - \ell_{t,\tildeA_t})\right]$ by $\textbf{pos-term} - \textbf{neg-term}$.  On the other hand, by the $\frac{1}{2}$-iw-stable assumption, $\E\left[\sum_{t=1}^{t'}(\ell_{t,\tildeA_t} - \ell_{t,x})\right]\leq \textbf{neg-term}$. Combining them, we get that for all $x\neq \hatx$, we also have $\E\left[\sum_{t=1}^{t'} (\ell_{t,A_t} - \ell_{t,x})\right]\leq \textbf{pos-term}$. 
    Comparing the coefficients in \textbf{pos-term} with those in \pref{def: candidate aware}, we see that \pref{alg: corral} satisfies $\frac{1}{2}$-LSB with constants $(c_0',c_1',c_2')$ where $c_0'=c_1'=O(c_1)$ and $c_2'=O(c_2)$. 
\end{proof}

\subsection{$\frac{2}{3}$-LSB to $\frac{2}{3}$-iw-stable (\pref{alg: 2/3 corral} / \pref{thm: 2/3 corral thm})} \label{app: 2/3 corral}

\begin{algorithm}[H] 
\caption{LSB via Corral (for $\alpha=\frac{2}{3}$)}\label{alg: 2/3 corral}
    \textbf{Input}:  candidate action $\hatx$, $\frac{2}{3}$-iw-stable algorithm $\calB$ over $\calX\backslash\{\hatx\}$ with constants $c_1,c_2$. \\
    \textbf{Define}: $\psi_t(q)= -\frac{3}{\eta_t}\sum_{i=1}^2 q_i^\frac{2}{3} + \frac{1}{\beta}\sum_{i=1}^2 \ln\frac{1}{q_i}$.\ \ \\
    \For{$t=1,2,\ldots$}{
        Let $\calB$ generate an action $\tildeA_t$. \\
        Let 
        \begin{align*}
            &\barq_{t}= \argmin_{q\in\Delta_2}\left\{\left\langle q, \sum_{\tau=1}^{t-1} z_{\tau} - \begin{bmatrix} 0 \\ B_{t-1}
            \end{bmatrix}\right\rangle + \psi_t(q)\right\}, \quad
            q_t = (1-\gamma_t)\barq_t, \\
            &\text{where}\ \ \ \eta_{t} = \frac{1}{t^{\frac{2}{3}} + 8c_1^{\frac{1}{3}}},\,\beta = \frac{1}{8c_2}\text{, and }\gamma_t = \max\left\{\sqrt{\eta_t} q_{t,2}^\frac{2}{3},\eta_tq_{t,2}^\frac{1}{3}\right\} . 
        \end{align*}
        Sample $i_t\sim \barq_t$. \\
        \lIf{$i_t=1$}{
            set $\bar A_t=\hatx$
        }
        \lElse{
            set $\bar A_t = \tildeA_t$
        }
        Sample $j_t\sim \gamma_t$. \\
        \lIf{$j_t=1$}{
            draw a revealing action of $\bar A_t$ and observe $\ell_{t,\bar A_t}$
        }
        \lElse{
            draw $A_t=\bar A_t$  
        }
        Define $z_{t,i} = \frac{\ell_{t,\bar A_t}\ind\{i_t=i\}\ind\{j_t=1\}}{\gamma_t}$ and
        \begin{align*}
            B_t= 
                c_1^\frac{1}{3}\left(\sum_{\tau=1}^t \frac{1}{\sqrt{q_{\tau,2}}}\right)^\frac{2}{3} + \frac{c_2}{\min_{\tau\leq t} q_{\tau,2}}.  
        \end{align*}
    }
\end{algorithm}

\begin{proof}[Proof of \pref{thm: 2/3 corral thm}]
    The per-step bonus $b_t=B_t-B_{t-1}$ is the sum of two terms: 
    \begin{align*}
        b_t^{\ts} &=c_1^\frac{1}{3}\left(\left(\sum_{\tau=1}^t \frac{1}{\sqrt{q_{\tau,2}}}\right)^\frac{2}{3}-\left(\sum_{\tau=1}^{t-1} \frac{1}{\sqrt{q_{\tau,2}}}\right)^\frac{2}{3}\right) \leq  c_1^\frac{1}{3}\frac{\frac{1}{\sqrt{q_{t,2}}}}{\left(\sum_{\tau=1}^t \frac{1}{\sqrt{q_{\tau,2}}}\right)^\frac{1}{3}} \leq \left(\frac{c_1}{q_{t,2}}\right)^\frac{1}{3}, \\
        b_t^{\lo} &= \frac{c_2}{\min_{\tau\leq t} q_{\tau,2}} - \frac{c_2}{\min_{\tau\leq t-1} q_{\tau,2}} = \frac{c_2}{q_{t,2}}\left(1-\frac{\min_{\tau\leq t} q_{\tau,2}}{\min_{\tau\leq t-1} q_{\tau,2}}\right). 
    \end{align*}
    Since $\frac{\barq_{t,2}}{q_{t,2}}\leq 2$, using the inequalities above, we have 
    \begin{align}
       \eta_t \barq_{t,2}^\frac{1}{3} b_t^{\ts}&\leq \eta_t \left(\frac{\barq_{t,2}c_1}{q_{t,2}}\right)^\frac{1}{3}\leq \eta_t (2c_1)^\frac{1}{3}\leq \frac{1}{4}. \label{eq: satisfy 1} \\ 
        \beta \barq_{t,2} b_t^{\lo} &\leq \beta \frac{\barq_{t,2}}{q_{t,2}} c_2\leq 2\beta c_2 \leq  \frac{1}{4}.  \label{eq: satisfy 2}
    \end{align}
    By \pref{lem: hybrid FTRL} and that $\frac{\barq_{t,2}}{q_{t,2}}\leq 2$, 
    we have for any $u$,  
    \begin{align}
        &\sum_{t=1}^{t'} \inner{q_t-u, z_t} \leq \underbrace{\sum_{t=1}^{t'} \inner{q_t-\barq_{t}, z_t}}_{\term_1} \nonumber \\
        &\ \  + O\bigg(c_1^{\frac{1}{3}}+\underbrace{\sum_{t=1}^{t'}\frac{q_{t,2}^{\frac{2}{3}}}{t^{\frac{1}{3}}}}_{\term_2} + \underbrace{\frac{\log t'}{\beta}}_{\term_3} +  \underbrace{\sum_{t=1}^{t'}\eta_t \min_{|\theta_t|\leq 1 } q_{t,i}^{\frac{4}{3}}(z_{t,i}-\theta_t)^2}_{\term_4} +  \underbrace{\sum_{t=1}^{t'} q_{t,2}b_t^{\ts}}_{\term_5} + \underbrace{\sum_{t=1}^{t'} q_{t,2}b_t^{\lo}}_{\term_6}\bigg)  - u_2\sum_{t=1}^{t'}b_t  \label{eq: corral decompose 2/3}
    \end{align}
    We bound individual terms below: 
    \begin{align*}
        \E[\term_1]=\E\left[\sum_{t=1}^{t'} \inner{q_t-\barq_{t}, z_t}\right] \leq O\left(\sum_{t=1}^{t'}\gamma_t\right) &= O\left(\sum_{t=1}^{t'}\frac{q_{t,2}^\frac{2}{3}}{t^{\frac{1}{3}}}+\frac{q_{t,2}^\frac{1}{3}}{t^{\frac{2}{3}}}\right)\\
        &=O\left(\min\left\{{t'}^\frac{2}{3},\left(\sum_{t=1}^{t'}q_t\right)^\frac{2}{3}(\log T)^\frac{1}{3}+\log T\right\}\right). 
    \end{align*}
    \begin{align*}
        \term_2&\leq 
        O\left(\term_1\right)\\
        \term_3&=\frac{\log t'}{\beta} \leq O\left(c_2 \log T\right).
    \end{align*}
    
    \begin{align*}
        \E\left[\term_4\right]&=\E\left[\sum_{t=1}^{t'}\eta_t  q_{t,2}^{\frac{4}{3}}(z_{t,2}-z_{t,1})^2\right] \\
        &\leq \E\left[\sum_{t=1}^{t'} \eta_t  \frac{q_{t,i}^{\frac{4}{3}}}{\gamma_t} \right] 
        \leq \E\left[\sum_{t=1}^{t'} \frac{\gamma_t^2}{\gamma_t} \right] = \E\left[O(\term_1)\right].
    \end{align*}
    \allowdisplaybreaks
    \begin{align}
           \term_5&=\sum_{t=1}^{t'}q_{t,2}b_t^{\ts} 
            \leq \sum_{t=1}^{t'} q_{t,2}\left(c_1^\frac{1}{3}\frac{\frac{1}{\sqrt{q_{t,2}}}}{\left(\sum_{\tau=1}^t \frac{1}{\sqrt{q_{\tau,2}}}\right)^\frac{1}{3}}\right) \nonumber\\
            &=c_1^\frac{1}{3}\sum_{t=1}^{t'} \frac{q_{t,2}^{-\frac{1}{6}}}{\left(\sum_{\tau=1}^t \frac{1}{\sqrt{q_{\tau,2}}}\right)^\frac{1}{3}}\times q_{t,2}^\frac{2}{3}   \label{eq: continue from 2/3}\\
            &\leq  c_1^\frac{1}{3}\left(\sum_{t=1}^{t'} \frac{\frac{1}{\sqrt{q_{t,2}}}}{\sum_{\tau=1}^t \frac{1}{\sqrt{q_{\tau,2}}}}\right)^\frac{1}{3}\left( \sum_{t=1}^{t'}q_{t,2}\right)^\frac{2}{3}  \tag{Cauchy-Schwarz}\\ 
            &\leq  c_1^\frac{1}{3}\left(1+\log\left(\sum_{t=1}^{t'}\frac{1}{q_{t,2}}\right)\right)^\frac{1}{3}\left( \sum_{t=1}^{t'}q_{t,2}\right)^\frac{2}{3}  \nonumber\\
            &\leq O\left(c_1^\frac{1}{3}\left(\sum_{t=1}^{t'} q_{t,2}\right)^\frac{2}{3}( \log T)^\frac{1}{3} \right)\nonumber.  
    \end{align}
    Continuing from \eqref{eq: continue from 2/3}, we also have $\term_5\leq \sum_{t=1}^{t'}q_{t,2}b_t^{\ts}\leq c_1^\frac{1}{3}\sum_{t=1}^{t'}\frac{1}{t^\frac{1}{3}}\leq O(c_1^\frac{1}{3}{t'}^\frac{2}{3})$ because $q_{\tau,2}\leq 1$. 
    \begin{align*}
            \term_6=\sum_{t=1}^{t'}q_{t,2}b_t^{\lo}
            &=c_2\sum_{t=1}^{t'} \left(1-\frac{\min_{\tau\leq t}q_{\tau,2}}{\min_{\tau\leq t-1}q_{\tau,2}}\right) \leq  c_2 \sum_{t=1}^{t'}\log\left(\frac{\min_{\tau\leq t-1}q_{\tau,2}}{\min_{\tau\leq t}q_{\tau,2}}\right) \leq O\left(c_2\log T\right). 
    \end{align*}
    Using all bounds above in \eqref{eq: corral decompose 2/3}, we can bound $\E\left[\sum_{t=1}^{t'} \langle q_t - u, z_t\rangle\right]$ by 
    \begin{align*}
        \underbrace{O\left(\min\left\{c_1^\frac{1}{3} \E[t']^\frac{2}{3},\ \ (c_1\log T)^\frac{1}{3}\left[\sum_{t=1}^{t'}q_{t,2}\right]^{\frac{2}{3}}\right\}+c_2\log T\right)}_{\textbf{pos-term}}  - u_2\underbrace{\E\left[ c_1^\frac{1}{3}\left(\sum_{t=1}^{t'} \frac{1}{q_{t,2}}\right)^\frac{2}{3} + \frac{c_2}{\min_{t\leq t'} q_{t,2}}\right]}_{\textbf{neg-term}}.
    \end{align*}
    For comparator $\hatx$, we choose $u=\mathbf{e}_1$ and bound $\E\left[\sum_{t=1}^{t'} (\ell_{t,A_t} - \ell_{t,\hatx})\right]$ by the \textbf{pos-term} above. 
    For comparator $x\neq \hatx$, we first choose $u=\mathbf{e}_2$ and upper bound $\E\left[\sum_{t=1}^{t'} (\ell_{t,A_t} - \ell_{t,\tildeA_t})\right]$ by $\textbf{pos-term} - \textbf{neg-term}$.  On the other hand, by the $\frac{1}{2}$-iw-stable assumption, $\E\left[\sum_{t=1}^{t'}(\ell_{t,\tildeA_t} - \ell_{t,x})\right]\leq \textbf{neg-term}$. Combining them, we get that for all $x\neq \hatx$, we also have $\E\left[\sum_{t=1}^{t'} (\ell_{t,A_t} - \ell_{t,x})\right]\leq \textbf{pos-term}$.
    
 Finally, notice that $\E_t[\ind\{A_t\neq \hatx\}]\geq \barq_{t,2}(1-\gamma_t)= q_{t,2},$. This implies that \pref{alg: 2/3 corral} is $\frac{2}{3}$-LSB with coefficient $(c_0',c_1',c_2')$ with $c_0'=c_1'=O(c_1)$, and $c_2'=O(c_1^{\frac{1}{3}}+c_2)$. 
    
\end{proof}

\begin{algorithm}[t]
\caption{dd-LSB via Corral (for $\alpha=\frac{1}{2}$)}\label{alg: corral dd}
    \textbf{Input}:  candidate action $\hatx$, $\frac{1}{2}$-iw-stable algorithm $\calB$ over $\calX\backslash\{\hatx\}$ with constant $c$. \\
    \textbf{Define}: $\psi(q)= \sum_{i=1}^2 \ln\frac{1}{q_{i}}$. $B_0=0$. \\
    \textbf{Define}: For first-order bound, $\xi_{t,x}=\ell_{t,x}$ and $m_{t,x}=0$; for second-order bound, $\xi_{t,x}=(\ell_{t,x}-m_{t,x})^2$ where $m_{t,x}$ is the loss predictor. \\ 
    \For{$t=1,2,\ldots$}{
        Let $\calB$ generate an action $\tildeA_t$ (which is the action to be chosen if $\calB$ is selected in this round). \\ 
        Receive prediction $m_{t,x}$ for all $x\in\calX$, and set $y_{t,1}=m_{t,\hatx}$ and $y_{t,2}=m_{t,\tildeA_t}$.
        \\
        Let
        \begin{align*}
            &\barq_{t}= \argmin_{q\in\Delta_2}\left\{\left\langle q, \sum_{\tau=1}^{t-1} z_{\tau} + y_t - \begin{bmatrix} 0 \\ B_{t-1}
            \end{bmatrix}  \right\rangle + \frac{1}{\eta_{t}}\psi(q)\right\}, \quad
            q_{t} = \left(1-\frac{1}{2t^2}\right)\barq_{t} + \frac{1}{4t^2}\one, \\
            &\text{where}\ \ \ \eta_{t} = \frac{1}{4}(\log T)^{\frac{1}{2}}\left(\sum_{\tau=1}^{t-1}(\ind[i_\tau=i]-q_{\tau,i})^2\xi_{\tau,A_\tau} + (c_1 + c_2^2)\log T\right)^{-\frac{1}{2}}. 
        \end{align*}
        Sample $i_t\sim q_t$. \\
        \lIf{$i_t=1$}{
            draw $A_t=\hatx$ and observe $\ell_{t,A_t}$ 
        }
        \lElse{
            draw $A_t=\tildeA_t$ and observe $\ell_{t,A_t}$  
        }
        Define $z_{t,i} = \frac{(\ell_{t,A_t}-y_{t,i})\ind\{i_t=i\}}{q_{t,i}} + y_{t,i}$ and
        \begin{align*}
            B_t= 
                \sqrt{c_1\sum_{\tau=1}^{t} \frac{\xi_{t,A_t}\ind[i_\tau=2]}{q_{\tau,2}^2}} + \frac{c_2}{\min_{\tau\leq t} q_{\tau,2}}. 
        \end{align*}
    }
\end{algorithm}

\subsection{$\frac{1}{2}$-dd-LSB to $\frac{1}{2}$-dd-iw-stable (\pref{alg: corral dd} / \pref{thm: basic corral thm data-dep})}\label{app: dd LSB to dd iw sta}
\begin{proof}[Proof of \pref{thm: basic corral thm data-dep}]
    Define $b_t=B_t-B_{t-1}$. Notice that we have 
    \begin{align}
        \eta_t \barq_{t,2}b_t &\leq 2\eta_t q_{t,2}b_t \nonumber \\ 
        &\leq 2\eta_t q_{t,2} \left(\frac{\frac{c_1\xi_{t,A_t}\ind[i_{t}=2]}{q_{t,2}^2} }{\sqrt{c_1\sum_{\tau=1}^{t-1} \frac{\xi_{\tau,A_\tau}\ind[i_\tau=2]}{q_{\tau,2}^2}}} \right)+ 2c_2\left(\frac{1}{\min_{\tau\leq t}q_{\tau,2}} - \frac{1}{\min_{\tau\leq t-1}q_{\tau,2}}\right)  \nonumber \\
        &\leq 2\eta_t \sqrt{c_1} + 2\eta_t c_2 \left(1-\frac{\min_{\tau\leq t }q_{\tau,2}}{\min_{\tau\leq t-1 }q_{\tau,2}}\right) \nonumber \\
        &\leq \frac{1}{4}.  \label{eq: eta b bound} 
    \end{align}
    By \pref{lem: logbarrier analysis} and that $\frac{\barq_{t,2}}{q_{t,2}}\leq 2$, we have for any $u$, 
    \begin{align*}
        \sum_{t=1}^{t'} \langle q_t - u, z_t\rangle 
        &\leq O\Bigg(\underbrace{\frac{\log T}{\eta_{t'}}}_{\term_1} + \underbrace{\sum_{t=1}^{t'} \eta_t \min_{|\theta|\leq 1} \sum_{i=1}^2 q_{t,i}^{2}\left(z_{t,i}-y_{t,i}-\theta_t\right)^2}_{\term_2} \\
        &\qquad \qquad + \underbrace{\sum_{t=1}^{t'}\inner{q_t- \barq_t, z_t}}_{\term_3} + \underbrace{\sum_{t=1}^{t'}q_{t,2}b_t}_{\term_4}\Bigg) - \sum_{t=1}^{t'}u_2b_t
    \end{align*}
    \begin{align*}
        \term_1 
        &\leq O\left(\sqrt{\frac{\sum_{t=1}^{t'-1}\sum_{i=1}^2 (\ind[i_t=i]-q_{t,i})^2\xi_{t,A_t} + (c_1+c_2^2)\log T}{\log T}}\log T\right) \\
        &\leq O\left(\sqrt{\sum_{t=1}^{t'}\sum_{i=1}^2 (\ind[i_t=i]-q_{t,i})^2\xi_{t,A_t}\log T} + (\sqrt{c_1}+c_2)\log T\right). 
    \end{align*}
    \begin{align*}
        \E\left[\term_2\right] 
        &\leq \E\left[ \sum_{t=1}^{t'}\eta_t\min_{\theta} \sum_{i=1}^2 q_{t,i}^2\left(\frac{(\ell_{t,A_t}-m_{t,A_t})\ind[i_t=i]}{q_{t,i}}-\theta\right)^2  \right] \\
        &\leq  \E\left[\sum_{t=1}^{t'}\eta_t\sum_{i=1}^2 (\ind[i_t=i]-q_{t,i})^2(\ell_{t,A_t}-m_{t,A_t})^2\right]  \tag{choosing $\theta=\ell_{t,A_t}-m_{t,A_t}$} \\
        &\leq \E\left[\sqrt{\sum_{t=1}^{t'}\sum_{i=1}^2 (\ind[i_t=i]-q_{t,i})^2\xi_{t,A_t}\log T}\right].   
    \end{align*}
    \begin{align*}
        \E\left[\term_3\right] = \E\left[ \sum_{t=1}^{t'}\inner{q_t-\barq_t, \E_t[z_t]} \right] = \E\left[ \sum_{t=1}^{t'}\left\langle-\frac{1}{2t^2}\barq_t + \frac{1}{4t^2}\one, \E_t[z_t]\right\rangle \right]\leq O(1). 
    \end{align*}
    \allowdisplaybreaks
    \begin{align*}
            \term_4 &\leq 
            \sum_{t=1}^{t'}q_{t,2}b_t \\
            &=\sum_{t=1}^{t'} q_{t,2}\left(\frac{\frac{c_1\xi_{t,A_t}\ind[i_t=2]}{q_{t,2}^2}}{\sqrt{c_1\sum_{\tau=1}^t \frac{\xi_{\tau,A_\tau}\ind[i_\tau=2]}{q_{\tau,2}^2}}}+c_2\left(\frac{1}{\min_{\tau\leq t}q_{\tau,2}}  - \frac{1}{\min_{\tau\leq t-1 }q_{\tau,2}}\right)\right) \\
            &\leq\sqrt{c_1}\sum_{t=1}^{t'} \sqrt{\frac{\frac{\xi_{t,A_t}\ind[i_t=2]}{q_{t,2}^2}}{\sum_{\tau=1}^{t} \frac{\xi_{\tau,A_\tau}\one[i_\tau=2]}{q_{\tau,2}^2}}}\times\sqrt{\xi_{t,A_t}\ind[i_t=2]} + c_2\sum_{t=1}^{t'} \left(1-\frac{\min_{\tau\leq t}q_{\tau,2}}{\min_{\tau\leq t-1}q_{\tau,2}}\right)  \\
            &\leq  \sqrt{c_1} \sqrt{\sum_{t=1}^{t'}\frac{\frac{\xi_{t,A_t}\ind[i_t=2]}{q_{t,2}^2}}{\sum_{\tau=1}^{t} \frac{\xi_{\tau,A_\tau}\one[i_\tau=2]}{q_{\tau,2}^2}}}\times\sqrt{\sum_{t=1}^{t'}\xi_{t,A_t}\ind[i_t=2]} + c_2 \sum_{t=1}^{t'}\log\left(\frac{\min_{\tau\leq t-1}q_{\tau,2}}{\min_{\tau\leq t}q_{\tau,2}}\right)    \tag{Cauchy-Schwarz}\\ 
            &\leq  \sqrt{c_1}\sqrt{1+\log\left(\sum_{t=1}^{t'}\frac{\xi_{t,A_t}\ind[i_t=2]}{q_{t,2}^2}\right)}\sqrt{ \sum_{t=1}^{t'}\xi_{t,A_t}\ind[i_t=2]} + c_2 \log \frac{1}{\min_{\tau\leq t'}q_{\tau,2}} \\
            &\leq O\left(\sqrt{c_1\sum_{t=1}^{t'} \xi_{t,A_t}\ind[i_t=2] \log T } + c_2\log T\right). 
    \end{align*}
    \begin{align*}
        \term_5 = -u_2 B_{t'}  = -u_2\left( \sqrt{c_1\sum_{t=1}^{t'} \frac{\xi_{t,A_t}\ind[i_t=2]}{q_{t,2}^2}} + \frac{c_2}{\min_{t\leq t'} q_{t,2}}\right). 
    \end{align*}
    Combining all inequalities above, we can bound $\E\left[\sum_{t=1}^{t'} \langle q_t - u, z_t\rangle\right]$ by 
    \begin{align}
        &\underbrace{O\left(\E\left[\sqrt{c_1\sum_{t=1}^{t'}\sum_{i=1}^2 (\ind[i_t=i]-q_{t,i})^2\xi_{t,A_t}\log T} + \sqrt{c_1\sum_{t=1}^{t'} \xi_{t,A_t}\ind[i_t=2] \log T } \right]+ (\sqrt{c_1}+c_2)\log T\right)}_{\textbf{pos-term}} \label{eq: data-dependent pos} \\
        &\qquad - u_2\underbrace{\left( \sqrt{c_1\sum_{t=1}^{t'} \frac{\xi_{t,A_t}\ind[i_t=2]}{q_{t,2}^2}} + \frac{c_2}{\min_{t\leq t'} q_{t,2}}\right)}_{\textbf{neg-term}}.  \nonumber
    \end{align}
    Similar to the arguments in the proofs of \pref{thm: basic corral thm} and \pref{thm: 2/3 corral thm}, we end up bounding $\E\left[\sum_{t=1}^{t'} (\ell_{t,A_t} - \ell_{t,x})\right]$ by the $\textbf{pos-term}$ above for all $x\in\calX$. 
    Finally, we process $\textbf{pos-term}$. Observe that
    \begin{align*}
        &\E_t\left[\sum_{i=1}^2 (\ind[i_t=i]-q_{t,i})^2\xi_{t,A_t} \right] \\
        &= \E_t\left[q_{t,1}(1-q_{t,1})^2\xi_{t,\hatx} + (1-q_{t,1})q_{t,1}^2\xi_{t,\tildeA_t} + q_{t,2}(1-q_{t,2})^2\xi_{t,\tildeA_t} + (1-q_{t,2})q_{t,2}^2 \xi_{t,\hatx}\right] \\
        &= \E_t\left[2q_{t,1}q_{t,2}^2\xi_{t,\hatx} + 2q_{t,1}^2q_{t,2}\xi_{t,\tildeA_t}\right] \\
        &= 2q_{t,1}q_{t,2}^2\xi_{t,\hatx} + 2q_{t,1}^2 \left( \sum_{x\neq \hatx} p_{t,x}\xi_{t,x} \right) \\
        &\leq 2p_{t,\hatx}(1-p_{t,\hatx})\xi_{t,\hatx} + 2\left( \sum_{x\neq \hatx} p_{t,x}\xi_{t,x} \right) \\
        &= 2\left(\sum_x p_{t,x}\xi_{t,x}-p_{t,\hatx}^2\xi_{t,\hatx}\right), 
    \end{align*}
    and that 
    \begin{align*}
        \E_t\left[\xi_{t,\tildeA_t}\ind[i_t=2]\right] = \sum_{x\neq \hatx} p_{t,\hatx} \xi_{t,x} \leq \sum_{x}p_{t,x}\xi_{t,x} - p_{t,\hatx}^2\xi_{t,\hatx}
    \end{align*}
    Thus, for any $u\in\calX$, 
    \begin{align*}
        \E\left[\sum_{t=1}^{t'}(\ell_{t,A_t}-\ell_{t,u})\right]\leq \E\left[\textbf{pos-term}\right] \leq   O\left(\sqrt{c_1\E\left[\sum_{t=1}^{t'} \left(\sum_{x}p_{t,x}\xi_{t,x} - p_{t,\hatx}^2\xi_{t,\hatx}\right)\right]\log T} + (\sqrt{c_1}+c_2)\log T\right), 
    \end{align*}
    which implies that the algorithm satisfies $\frac{1}{2}$-dd-LSB with constants $(O(c_1), O(\sqrt{c_1}+c_2))$. 
\end{proof}

\subsection{$\frac{1}{2}$-LSB to $\frac{1}{2}$-strongly-iw-stable (\pref{alg: corral strong} / \pref{thm: basic corral thm strong})}\label{app: strongly stable}

\begin{algorithm}[H] 
\caption{LSB via Corral (for $\alpha=\frac{1}{2}$, using a $\frac{1}{2}$-strongly-iw-stable algorithm)}\label{alg: corral strong}
    \textbf{Input}:  candidate action $\hatx$, $\frac{1}{2}$-iw-stable algorithm $\calB$ over $\calX$ with constant $c$. \\
    \textbf{Define}: $\psi_t(q)= \frac{-2}{\eta_t}\sum_{i=1}^2 \sqrt{q_i} + \frac{1}{\beta}\sum_{i=1}^2 \ln\frac{1}{q_{i}}$. \\
    $B_0=0$. \\
    \For{$t=1,2,\ldots$}{
        Let $\calB$ generate an action $\tildeA_t$ (which is the action to be chosen if $\calB$ is selected in this round). \\ 
        Let
        \begin{align*}
            &\barq_{t}= \argmin_{q\in\Delta_2}\left\{\left\langle q, \sum_{\tau=1}^{t-1} z_{\tau} - \begin{bmatrix} 0 \\ B_{t-1}
            \end{bmatrix}\right\rangle + \psi_t(q)\right\}, \quad
            q_{t} = \left(1-\frac{1}{2t^2}\right)\barq_{t} + \frac{1}{4t^2}\one, \\
            &\text{where}\ \ \ \eta_{t} = \frac{1}{\sqrt{\sum_{\tau=1}^{t}\ind\{\tildeA_\tau\neq \hatx\}}+8\sqrt{c_1}},\quad \beta = \frac{1}{8c_2}.  
        \end{align*}
        Sample $i_t\sim q_t$. \\
        \lIf{$i_t=1$}{
            draw $A_t=\hatx$ and observe $\ell_{t,A_t}$ 
        }
        \lElse{
            draw $A_t=\tildeA_t$ and observe $\ell_{t,A_t}$  
        }
        Define $z_{t,i} = \frac{\ell_{t,A_t}\ind\{i_t=i\}}{q_{t,i}}\ind\{\tildeA_t\neq \hatx\}$ 
        and
        \begin{align*}
            B_t= 
                \sqrt{c_1\sum_{\tau=1}^t \frac{\ind\left\{\tildeA_t\neq \hatx\right\}}{q_{\tau,2}}} + c_2\max_{\tau\leq t} \frac{1}{ q_{\tau,2}}. 
        \end{align*}
        
    }
\end{algorithm}

\begin{proof}[Proof of \pref{thm: basic corral thm strong}] 
    The proof of this theorem mostly follows that of \pref{thm: basic corral thm}. The difference is that the regret of the base algorithm $\calB$ is now bounded by \begin{align}
        \sqrt{c_1\sum_{t=1}^{t'} \frac{1}{q_{t,2} + q_{t,1}\ind\{\tildeA_t=\hatx\}}} + \frac{c_2\log T}{\min_{t\leq t'} q_{t,2}} \leq \sqrt{c_1\sum_{t=1}^{t'} \frac{\ind\{\tildeA_t\neq \hatx\}}{q_{t,2}}} + \sqrt{c_1 t'} + \frac{c_2\log T}{\min_{t\leq t'} q_{t,2}}  \label{eq: new corral bound}
    \end{align}
    because when $\tildeA_t=\hatx$, the base algorithm $\calB$ is able to receive feedback no matter which side the Corral algorithm chooses.  The goal of adding bonus is now only to cancel the first term and the third term on the right-hand side of \eqref{eq: new corral bound}. 
    
    Similar to \pref{eq: corral decompose}, we have 
    \begin{align}
        &\sum_{t=1}^{t'} \inner{q_t-u, z_t} \leq \underbrace{\sum_{t=1}^{t'} \inner{q_t-\barq_{t}, z_t}}_{\term_1} \nonumber \\
        &\ \  + O\bigg(\sqrt{c_1}+\underbrace{\sum_{t=1}^{t'} \frac{\sqrt{q_{t,2}}\ind\{\tildeA_t\neq \hatx\}}{\sqrt{\sum_{\tau=1}^t \ind\{\tildeA_\tau\neq \hatx\} }} }_{\term_2} + \underbrace{\frac{\log T}{\beta}}_{\term_3} +  \underbrace{\sum_{t=1}^{t'}\eta_t \min_{|\theta_t|\leq 1} q_{t,i}^{\frac{3}{2}}(z_{t,i}-\theta_t)^2}_{\term_4} +  \underbrace{\sum_{t=1}^{t'} q_{t,2}b_t^{\ts}}_{\term_5} + \underbrace{\sum_{t=1}^{t'} q_{t,2}b_t^{\lo}}_{\term_6}\bigg) - u_2\sum_{t=1}^{t'}b_t  \label{eq: corral decompose strong}
    \end{align}
    where 
    \begin{align*}
        b_t^{\ts} &= \sqrt{c_1\sum_{\tau=1}^t \frac{\ind\{\tildeA_\tau\neq \hatx\}}{q_{\tau,2}}} - \sqrt{c_1\sum_{\tau=1}^{t-1} \frac{\ind\{\tildeA_\tau\neq \hatx\}}{q_{\tau,2}}}\leq  \frac{\frac{c_1\ind\{\tildeA_t\neq \hatx\}}{q_{t,2}}}{\sqrt{c_1\sum_{\tau=1}^t \frac{\ind\{\tildeA_\tau\neq \hatx\}}{q_{\tau,2}}}} \leq \sqrt{\frac{c_1}{q_{t,2}}}, \\
        b_t^{\lo} &= \frac{c_2}{\min_{\tau\leq t} q_{\tau,2}} - \frac{c_2}{\min_{\tau\leq t-1} q_{\tau,2}} = \frac{c_2}{q_{t,2}}\left(1-\frac{\min_{\tau\leq t} q_{\tau,2}}{\min_{\tau\leq t-1} q_{\tau,2}}\right) 
    \end{align*}
    satisfying $\eta_t \sqrt{\barq_{t,2}} b_t^{\ts}\leq \frac{1}{2}$ and $\beta \barq_{t,2}b_t^{\lo}\leq \frac{1}{2}$ as in \eqref{eq: satisfy 1} and \eqref{eq: satisfy 2}. 
    
    Below, we bound $\term_1, \ldots, \term_6$. 
    
    \begin{align*}
        \E[\term_1]=\E\left[\sum_{t=1}^{t'} \inner{q_t-\barq_{t}, z_t}\right] \leq O\left(\sum_{t=1}^{t'}\frac{1}{t^2}\right) = O(1). 
    \end{align*}
    \begin{align*}
        \term_2&\leq O\left(\min\left\{\sqrt{\sum_{t=1}^{t'}\ind\{\tildeA_t\neq \hatx\}},\ \ \sqrt{\sum_{t=1}^{t'}q_{t,2}\ind\{\tildeA_t\neq \hatx\}\log T} \right\}\right)\\
        \term_3&=\frac{\log T}{\beta} \leq  O\left(c_2 \log T\right).
    \end{align*}
    
    \begin{align*}
        \E\left[\term_4\right]&=\E\left[\sum_{t=1}^{t'}\eta_t \min_{\theta_t\in[-1,1]} q_{t,i}^{\frac{3}{2}}(z_{t,i}-\theta_t)^2\right] \\
        &\leq \E\left[\sum_{t=1}^{t'} \eta_t \sum_{i=1}^2 q_{t,i}^{\frac{3}{2}}\left(z_{t,i}-\ell_{t,A_t}\right)^2 \ind\{\tildeA_t\neq \hatx\} \right]  \tag{when $\tildeA_t=\hatx$, $z_{t,1}=z_{t,2}=0$}\\
        &= \E\left[\sum_{t=1}^{t'} \eta_t\sum_{i=1}^2\frac{1}{\sqrt{q_{t,i}}}\left( \ind[i_t=i]\ell_{t,A_t}-q_{t,i}\ell_{t,A_t}\right)^2\ind\{\tildeA_t\neq \hatx\} \right] \\
        &\leq \E\left[\sum_{t=1}^{t'}\eta_t \sum_{i=1}^2 \left(\sqrt{q_{t,i}}(1-q_{t,i})^2 + (1-q_{t,i})q_{t,i}^\frac{3}{2}\right)\ind\{\tildeA_t\neq \hatx\} \right] \\
        &\leq \E\left[\sum_{t=1}^{t'} \eta_t \sqrt{q_{t,2}\ind\{\tildeA_t\neq \hatx\}}  \right] \leq O\left(\E[\term_2]\right).
    \end{align*}
    \allowdisplaybreaks
    
    \begin{align}
            \term_5 + \term_6 &= 
            \sum_{t=1}^{t'}q_{t,2}(b_t^{\ts} + b_t^{\lo}) \nonumber \\
            &=\sum_{t=1}^{t'} q_{t,2}\left(\frac{\frac{c_1\ind\{\tildeA_t\neq \hatx\}}{q_{t,2}}}{\sqrt{c_1\sum_{\tau=1}^t \frac{\ind\{\tildeA_\tau\neq \hatx\}}{q_{\tau,2}}}}+c_2\left(\frac{1}{\min_{\tau\leq t}q_{\tau,2}}  - \frac{1}{\min_{\tau\leq t-1 }q_{\tau,2}}\right)\right) \nonumber \\
            &=\sqrt{c_1}\sum_{t=1}^{t'} \frac{\frac{\ind\{\tildeA_t\neq \hatx\}}{\sqrt{q_{t,2}}}}{\sqrt{\sum_{\tau=1}^{t} \frac{\ind\{\tildeA_\tau\neq \hatx\}}{q_{\tau,2}}}}\times\sqrt{q_{t,2}\ind\{\tildeA_t\neq \hatx\}} + c_2\sum_{t=1}^{t'} \left(1-\frac{\min_{\tau\leq t}q_{\tau,2}}{\min_{\tau\leq t-1}q_{\tau,2}}\right) \label{eq: continue from 2} \\
            &\leq  \sqrt{c_1}\sqrt{\sum_{t=1}^{t'} \frac{\frac{\ind\{\tildeA_t\neq \hatx\}}{q_{t,2}}}{\sum_{\tau=1}^t \frac{\ind\{\tildeA_\tau\neq \hatx\}}{q_{\tau,2}}}}\sqrt{ \sum_{t=1}^{t'}q_{t,2}\ind\{\tildeA_t\neq \hatx\}} + c_2 \sum_{t=1}^{t'}\log\left(\frac{\min_{\tau\leq t-1}q_{\tau,2}}{\min_{\tau\leq t}q_{\tau,2}}\right)    \tag{Cauchy-Schwarz}\\ 
            &\leq  \sqrt{c_1}\sqrt{1+\log\left(\sum_{t=1}^{t'}\frac{\ind\{\tildeA_t\neq \hatx\}}{q_{t,2}}\right)}\sqrt{ \sum_{t=1}^{t'}q_{t,2}\ind\{\tildeA_t\neq \hatx\}} + c_2 \log \frac{1}{\min_{\tau\leq t'}q_{\tau,2}}   \nonumber \\
            &\leq O\left(\sqrt{c_1\sum_{t=1}^{t'} q_{t,2}\ind\{\tildeA_t\neq \hatx\} \log T } + c_2\log T\right). 
    \end{align}

    Continuing from \eqref{eq: continue from 2}, we also have $\term_5\leq \sqrt{c_1}\sum_{t=1}^{t'}\frac{\ind\{\tildeA_t\neq \hatx\}}{\sum_{\tau=1}^{t} \ind\{\tildeA_\tau\neq \hatx\} }\leq 2\sqrt{c_1t'}$.

    Using all the inequalities above in \eqref{eq: corral decompose strong}, we can bound $\E\left[\sum_{t=1}^{t'} \langle q_t - u, z_t\rangle\right]$ by 
    \begin{align*}
        &\underbrace{O\left(\min\left\{\sqrt{c_1 \E[t']},\ \ \sqrt{c_1\left[\sum_{t=1}^{t'}q_{t,2}\ind\{\tildeA_t\neq \hatx\}\right]\log T}\right\}+(\sqrt{c_1}+c_2)\log T\right)}_{\textbf{pos-term}} \\
        &\qquad \qquad \qquad - u_2\underbrace{\E\left[ \sqrt{c_1\sum_{t=1}^{t'} \frac{\ind\{\tildeA_t\neq \hatx\}}{q_{t,2}}} + \frac{c_2}{\min_{t\leq t'} q_{t,2}}\right]}_{\textbf{neg-term}}.
    \end{align*}
    
    For comparator $\hatx$, we set $u=\mathbf{e}_1$. Then we have $\E[\sum_{t=1}^{t'}(\ell_{t,A_t}-\ell_{t,\hatx})]\leq \textbf{pos-term}$. 
    Observe that \sloppy$\E\left[\sum_{t=1}^{t'}q_{t,2}\ind\{\tildeA_t\neq \hatx\}\right] = \E\left[\sum_{t=1}^{t'}(1-p_{t,\hatx})\right]$, so this gives the desired property of $\frac{1}{2}$-LSB for action $\hatx$.  For $x\neq \hatx$, we set $u=\mathbf{e}_2$ and get $\E\left[\sum_{t=1}^{t'}(\ell_{t,A_t}-\ell_{t,x})\right] = \E\left[\sum_{t=1}^{t'}(\ell_{t,A_t}-\ell_{t,\tildeA_t})\right] + \E\left[\sum_{t=1}^{t'}(\ell_{t,\tildeA_t}-\ell_{t,x})\right] \leq (\textbf{pos-term} - \textbf{neg-term}) + (\textbf{neg-term} + \sqrt{c_1\E[t']})$ where the additional $\sqrt{c_1 t'}$ term comes from the second term on the right-hand side of \eqref{eq: new corral bound}, which is not cancelled by the negative term. Note, however, that this still satisfies the requirement of $\frac{1}{2}$-LSB for $x\neq\hatx$ because for these actions, we only require the worst-case bound to hold. Overall, we have justified that \pref{alg: corral strong} satisfies $\frac{1}{2}$-LSB with constants $(c_0', c_1',c_2')$ where $c_0'=c_1'=O(c_1)$ and $c_2=(\sqrt{c_1}+c_2)$.  
\end{proof}

\subsection{$\frac{1}{2}$-dd-LSB to $\frac{1}{2}$-dd-strongly-iw-stable (\pref{alg: corral MDP} / \pref{lem: main lemma for MDP})}

\begin{algorithm}[t]
\caption{dd-LSB via Corral (for $\alpha=\frac{1}{2}$, using a  $\frac{1}{2}$-dd-strongly-iw-stable algorithm)}\label{alg: corral MDP}
    \textbf{Input}:  candidate action $\hatx$, $\frac{1}{2}$-iw-stable algorithm $\calB$ over $\calX$ with constant $(c_1,c_2)$. \\
    \textbf{Define}: $\psi(q)= \sum_{i=1}^2 \ln\frac{1}{q_{i}}$. $B_0=0$. \\
    \textbf{Define}: For first-order bound, $\xi_{t,x}=\ell_{t,x}$ and $m_{t,x}=0$; for second-order bound, $\xi_{t,x}=(\ell_{t,x}-m_{t,x})^2$ where $m_{t,x}$ is the loss predictor. \\
    \For{$t=1,2,\ldots$}{
        Receive prediction $m_{t,x}$ for all $x\in\calX$, and set $y_{t,1}=m_{t,\hatx}$ and $y_{t,2}=m_{t,\tildeA_t}$. \\
        Let $\calB$ generate an action $\tildeA_t$ (which is the action to be chosen if $\calB$ is selected in this round). \\ 
        Let
        \begin{align*}
            &\barq_{t}= \argmin_{q\in\Delta_2}\left\{\left\langle q, \sum_{\tau=1}^{t-1} z_{\tau} + y_t - \begin{bmatrix} 0 \\ B_{t-1}
            \end{bmatrix}  \right\rangle + \frac{1}{\eta_{t}}\psi(q)\right\}, \quad
            q_{t} = \left(1-\frac{1}{2t^2}\right)\barq_{t} + \frac{1}{4t^2}\one, \\
            &\text{where}\ \ \ \eta_{t} = \frac{1}{4}(\log T)^{\frac{1}{2}}\left(\sum_{\tau=1}^{t-1}(\ind[i_\tau=i]-q_{\tau,i})^2\xi_{\tau,A_\tau}\ind[\tildeA_\tau\neq \hatx] + (c_1 + c_2^2)\log T\right)^{-\frac{1}{2}}. 
        \end{align*}
        Sample $i_t\sim q_t$. \\
        \lIf{$i_t=1$}{
            draw $A_t=\hatx$ and observe $\ell_{t,A_t}$ 
        }
        \lElse{
            draw $A_t=\tildeA_t$ and observe $\ell_{t,A_t}$  
        }
        Define $z_{t,i} = \left(\frac{(\ell_{t,A_t}-y_{t,i})\ind\{i_t=i\}}{q_{t,i}} + y_{t,i}\right)\ind\{\tildeA_t\neq \hatx\}$ and
        \begin{align}
            B_t= 
                \sqrt{c_1\sum_{\tau=1}^{t} \frac{\xi_{t,A_t} \ind\{i_\tau=2\}\ind\{\tildeA_t\neq \hatx\}}{q_{\tau,2}^2}} + \frac{c_2}{\min_{\tau\leq t} q_{\tau,2}}.   \label{eq: MDP B choice}
        \end{align}
    }
\end{algorithm}

\begin{lemma}\label{lem: main lemma for MDP}
    Let $\calB$ be an algorithm with the following stability guarantee:  
    given an adaptive sequence of weights $q_1,q_2,\dots\in(0,1]^{\calX}$ such that the feedback in round $t$ is observed with probability $q_t(x)$ if $x$ is chosen, and an adaptive sequence $\{m_{t,x}\}_{x\in\calX}$ available at the beginning of round $t$, it obtains the following pseudo regret guarantee for any stopping time~$t'$: 
\begin{align*}
    \E\left[\sum_{t=1}^{t'} \left(\ell_{t,A_t} - \ell_{t,u}\right)\right] \leq 
        \E\left[\sqrt{c_1\sum_{t=1}^{t'}\frac{\upd_t\cdot\xi_{t,A_t}}{q_t(A_t)^2}} + \frac{c_2}{\min_{t\leq t'}\min_x q_t(x)}\right],   
\end{align*}
where $\upd_t=1$ if feedback is observed in round $t$ and $\upd_t=0$ otherwise.  $\xi_{t,x} = (\ell_{t,x}-m_{t,x})^2$ in the second-order bound case, and $\xi_{t,x} = \ell_{t,x}$ in the first-order bound case. Then \pref{alg: corral MDP} with $\calB$ as input satisfies $\frac{1}{2}$-dd-LSB. 
\end{lemma}

\begin{proof}
    The proof of this lemma is a combination of the elements in \pref{thm: basic corral thm data-dep} (data-dependent iw-stable) and \pref{thm: basic corral thm strong} (strongly-iw-stable), so we omit the details here and only provide a sketch. 
    
    In \pref{alg: corral MDP}, since $\calB$ is  $\frac{1}{2}$-dd-strongly-iw-stable, its regret is upper bounded by the order of 
    \begin{align}
       &\sqrt{c_1\sum_{t=1}^{t'}\left(\frac{\ind\{\tildeA_t=\hatx\}\cdot\xi_{t,A_t}}{1} + \frac{\ind\{\tildeA_t\neq \hatx\}\ind\{i_t=2\}\xi_{t,A_t}}{q_{t,2}^2}\right)} + \frac{c_2}{\min_{t\leq t'}\min_x q_t(x)}    \nonumber \\
       &\leq \sqrt{c_1\sum_{t=1}^{t'}\frac{\ind\{\tildeA_t\neq \hatx\}\ind\{i_t=2\}\cdot\xi_{t,A_t}}{q_{t,2}^2}} + \sqrt{c_1\sum_{t=1}^{t'}\xi_{t,A_t} } + \frac{c_2}{\min_{t\leq t'}\min_x q_t(x)}   \label{eq: strong decompose}
    \end{align}
    because if $\calB$ chooses $\hatx$, then the probability of observing the feedback is $1$, and is $q_{t,2}$ otherwise. This motivates the choice of the bonus in \eqref{eq: MDP B choice}. Then we can follow the proof of \pref{thm: basic corral thm data-dep} step-by-step, and show that the regret compared to $\hatx$ is upper bounded by the order of 
    \allowdisplaybreaks
    \begin{align*}
        &\E\left[\sqrt{c_1\sum_{t=1}^{t'}\sum_{i=1}^2 (\ind[i_t=i]-q_{t,i})^2\xi_{t,A_t} \ind\{\tildeA_t\neq \hatx\}\log T}\right] \\
        &\qquad + \E\left[\sqrt{c_1\sum_{t=1}^{t'} \xi_{t,A_t}\ind[i_t=2]\ind\{\tildeA_t\neq \hatx\} \log T } \right]+ (\sqrt{c_1}+c_2)\log T \\
        &\leq \sqrt{c_1\E\left[\sum_{t=1}^{t'}\left(2q_{t,1}q_{t,2}^2\xi_{t,\hatx} + 2q_{t,1}^2q_{t,2}\xi_{t,\tildeA_t}\right)\ind\{\tildeA_t\neq \hatx\}\right]}  \\
        &\qquad + \sqrt{c_1\E\left[\sum_{t=1}^{t'} \xi_{t,\tildeA_t}q_{t,2}\ind\{\tildeA_t\neq \hatx\}\right] \log T } + (\sqrt{c_1}+c_2)\log T   \tag{taking expectation over $i_t$ and following the calculation in the proof of \pref{thm: basic corral thm data-dep}} \\
        &= \sqrt{c_1\E\left[\sum_{t=1}^{t'}\left(2q_{t,1}q_{t,2}^2\xi_{t,\hatx}\Pr[\tildeA_t\neq \hatx] + 2q_{t,1}^2q_{t,2}\sum_{x\neq \hatx} \Pr[\tildeA_t=x] \xi_{t,x}\right)\right]}  \\
        &\qquad + \sqrt{c_1\E\left[\sum_{t=1}^{t'} q_{t,2}\sum_{x\neq \hatx} \Pr[\tildeA_t= x]\xi_{t,x}\right] \log T } + (\sqrt{c_1}+c_2)\log T \tag{taking expectation over $\tildeA_t$} \\
        & \leq  \sqrt{c_1\E\left[\sum_{t=1}^{t'}\left(2q_{t,1}(1-p_{t,\hatx})\xi_{t,\hatx} + 2q_{t,1}^2\sum_{x\neq \hatx} p_{t,x} \xi_{t,x}\right)\right]}  \\
        &\qquad + \sqrt{c_1\E\left[\sum_{t=1}^{t'} \sum_{x\neq \hatx} p_{t,x}\xi_{t,x}\right] \log T } + (\sqrt{c_1}+c_2)\log T \tag{using the property that for $x\neq \hatx$, $p_{t,x}=\Pr[\tildeA_t=x]q_{t,2}$ and thus $1-p_{t,\hatx}=\Pr[\tildeA_t\neq \hatx]q_{t,2}$} \\
        &\leq \sqrt{c_1\E\left[\sum_{t=1}^{t'}\left(2p_{t,\hatx}(1-p_{t,\hatx})\xi_{t,\hatx} + 2\sum_{x\neq \hatx} p_{t,x} \xi_{t,x}\right)\right]}  \tag{$q_{t,1}\leq p_{t,\hatx}\leq 1$} \\
        &\qquad + \sqrt{c_1\E\left[\sum_{t=1}^{t'} \sum_{x\neq \hatx} p_{t,x}\xi_{t,x}\right] \log T } + (\sqrt{c_1}+c_2)\log T \\
        &\leq O\left(\sqrt{c_1\E\left[\sum_{t=1}^{t'}\left(\sum_{x} p_{t,x} \xi_{t,x} - p_{t,\hatx}^2\xi_{t,\hatx} \right)\right]\log T} \right) + (\sqrt{c_1}+c_2)\log T
    \end{align*}
    which satisfies the requirement of $\frac{1}{2}$-dd-LSB for the regret against $\hatx$. For the regret against $x\neq \hatx$, similar to the proof of \pref{thm: basic corral thm strong}, an extra positive regret comes from the second term in \eqref{eq: strong decompose}. Therefore, the regret against $x\neq \hatx$, can be upper bounded by 
    \begin{align*}
        O\left(\sqrt{c_1\E\left[\sum_{t=1}^{t'}\sum_{x} p_{t,x} \xi_{t,x}\right]\log T} \right) + (\sqrt{c_1}+c_2)\log T
    \end{align*}
    which also satisfies the requirement of $\frac{1}{2}$-dd-LSB for $x\neq \hatx$. 
\end{proof}

\section{Analysis for IW-Stable Algorithms}
\label{app: iw-stable}

\begin{algorithm}[H]
     \caption{EXP2}\label{alg: EXP2}
     \textbf{Input}: $\calX$.  \\
     \For{$t=1, 2, \ldots$}{
          Receive update probability $q_t$. \\
          Let 
          \begin{align*}
              \eta_{t} = \min\left\{\sqrt{\frac{\ln |\calX|}{d\sum_{\tau=1}^t \frac{1}{q_\tau}}}, \ \frac{1}{2d}\min_{\tau\leq t}q_\tau\right\}, \qquad 
              P_t(a) \propto \exp\left(  -\eta_t \sum_{\tau=1}^{t-1} \hatell_\tau(a) \right). 
          \end{align*}
          Sample an action $a_t\sim p_t=(1-\frac{d\eta_t}{q_t} )P_t + \frac{d\eta_t}{q_t}\nu$, where $\nu$ is John's exploration.  \\
          With probability $q_t$, receive $\ell_t(a_t)=\inner{a_t,\theta_t} + \text{noise}$\\
          (in this case, set $\upd_t=1$; otherwise, set $\upd_t=0$). \\
          Construct loss estimator: 
          \begin{align*}
              \hatell_t(a) = \frac{\upd_t}{q_t}a^\top \left(\E_{b\sim p_t}[bb^\top ]\right)^{-1} a_t \ell_t(a_t) 
          \end{align*}
     }
\end{algorithm}

\subsection{EXP2 (\pref{alg: EXP2} / \pref{lem: EXP2})}
\begin{proof}[Proof of \pref{lem: EXP2}]
    Consider the EXP2 algorithm (\pref{alg: EXP2}) which corresponds to FTRL with negentropy potential. By standard analysis of FTRL (\pref{lem: FTRL}), for any $\tau$ and any $a^\star$, 
    \begin{align*}
    \sum_{t=1}^\tau \E_t\left[\sum_{a} P_t(a)\hatell_t(a)\right] - \sum_{t=1}^\tau \E_t\left[\hatell_t (a^\star) \right] 
    &\leq \frac{\ln |\calX|}{\eta_\tau} + \sum_{t=1}^\tau\E_t\left[\max_{P}\left(\langle P_t-P, \ell_t\rangle - \frac{1}{\eta_t}D_{\psi}(P, P_t)\right)\right]\,.
\end{align*}
To apply \pref{lem: exp3 stab}, we need to show that $\eta_t\hatell_t(a)\geq -1$. We have by Cauchy Schwarz
$|a^\top\left(\E_{b\sim p_t}[bb^\top]\right)^{-1}a_t|\leq \sqrt{a^\top\left(\E_{b\sim p_t}[bb^\top]\right)^{-1}a}\sqrt{a_t^\top\left(\E_{b\sim p_t}[bb^\top]\right)^{-1}a_t}$. For each term, we have
\begin{align*}
    a^\top\left(\E_{b\sim p_t}[bb^\top]\right)^{-1}a\leq \frac{q_t}{d\eta_t}a^\top\left(\E_{b\sim \nu}[bb^\top]\right)^{-1}a=\frac{q_t}{\eta_t}\,,
\end{align*}
due to the properties of John's exploration. Hence $|\eta_t\hatell_t(a)| \leq |\ell_t(a)|\leq 1$.
We can apply \pref{lem: exp3 stab} for the stability term, resulting in
\begin{align*}
    &\E_t\left[\max_{P}\left(\langle P_t-P, \ell_t\rangle - \frac{1}{\eta_t}D_{\psi}(P, P_t)\right)\right]\\
    &\leq \eta_t\E_t\left[\sum_a P_t(a)\frac{a^\top\left(\E_{b\sim p_t}[bb^\top]\right)^{-1}a_ta_t^\top\left(\E_{b\sim p_t}[bb^\top]\right)^{-1}a}{q_t^2}\right]\\
    &=\eta_t\sum_a P_t(a)\frac{a^\top\left(\E_{b\sim p_t}[bb^\top]\right)^{-1}a}{q_t}\\
    &=2\eta_t\sum_a P_t(a)\frac{a^\top\left(\E_{b\sim P_t}[bb^\top]\right)^{-1}a}{q_t}
    =\frac{2\eta_t d}{q_t}\,.
\end{align*}
By $|\ell_t(a)|\leq 1$ we have furthermore
\begin{align*}
\sum_{t=1}^\tau\E_t\left[\sum_{a}(P_t(a)-p_t(a))\hatell_t(a)\right]=\sum_{t=1}^\tau\E_t\left[\sum_{a}(P_t(a)-p_t(a))\ell_t(a)\right]\leq\sum_{t=1}^\tau \frac{d\eta_t}{q_t}\,.
\end{align*}
Combining everything and taking the expectation on both sides leads to
\begin{align*}
    \E\left[\sum_{t=1}^\tau \sum_{a}p_t(a)\ell_t(a)-\ell_t(a^\star)\right] &\leq \E\left[\frac{\log |\calX|}{\eta_\tau}+\sum_{t=1}^\tau \frac{3d\eta_t}{q_t}\right]\\
    &\leq \E\left[7\sqrt{d\log|\calX|\sum_{t=1}^\tau\frac{1}{q_t}}+2\log |\calX|d \frac{1}{\min_{t\leq\tau}q_t}\right]\,.
\end{align*}
\end{proof}

\subsection{EXP4 (\pref{alg: EXP4} / \pref{lem: EXP4})}
In this section, we use the more standard notation $\Pi$ to denote the policy class. 
\begin{algorithm}[H]
     \caption{EXP4}\label{alg: EXP4}
     \textbf{Input}: $\Pi$ (policy class), $K$ (number of arms)  \\
     \For{$t=1, 2, \ldots$}{
          Receive context $x_t$. \\
          Receive update probability $q_t$. \\
          Let 
          \begin{align*}
              \eta_{t} = \sqrt{\frac{\ln |\Pi|}{K\sum_{\tau=1}^t \frac{1}{q_\tau}}}, \qquad 
              P_t(\pi) \propto \exp\left(  -\eta_t \sum_{\tau=1}^{t-1} \hatell_\tau\left(\pi(x_\tau)\right) \right). 
          \end{align*}
          Sample an arm $a_t\sim p_t$ where $p_t(a)=\sum_{\pi: \pi(x_t)=a} P_t(\pi)\pi(x_t)$. \\
          With probability $q_t$, receive $\ell_t(a_t)$ (in this case, set $\upd_t=1$; otherwise, set $\upd_t=0$). \\
          Construct loss estimator: 
          \begin{align*}
              \hatell_t(a) = \frac{\upd_t\ind[a_t=a]\ell_t(a)}{q_t p_t(a)}. 
          \end{align*}
     }
\end{algorithm}

\begin{proof}[Proof of \pref{lem: EXP4}]
Consider the EXP4 algorithm with adaptive stepsize (\pref{alg: EXP4}), which corresponds to FTRL with negentropy regularization.   By standard analysis of FTRL (\pref{lem: FTRL}) and \pref{lem: exp3 stab}, for any $\tau$ and any $\pi^\star$ 
\begin{align*}
    \sum_{t=1}^\tau \E_t\left[\sum_{\pi} P_t(\pi)\hatell_t(\pi(x_t))\right] - \sum_{t=1}^\tau \E_t\left[\hatell_t (\pi^\star(x_t)) \right] 
    &\leq \frac{\ln |\Pi|}{\eta_\tau} +
    \sum_{t=1}^\tau\frac{\eta_t}{2}\E_t\left[\sum_{\pi}P_t(\pi)\hatell_t^2(\pi(x_t))\right]\\
    &\leq \frac{\ln |\Pi|}{\eta_\tau} +
    \sum_{t=1}^\tau\frac{\eta_t}{2}\sum_{\pi}P_t(\pi)\frac{\ell_t(\pi(x_t))}{q_tp_t(\pi(x_t))}\\
    &\leq \frac{\ln |\Pi|}{\eta_\tau} +K \frac{\eta_t}{2q_t} \leq 2\sqrt{K\ln |\Pi|\sum_{t=1}^\tau \frac{1}{q_t}}.
\end{align*}
Taking expectation on both sides finishes the proof. 
\end{proof}

\subsection{$(1-1/\log(K))$-Tsallis-INF (\pref{alg: Tsallis-Inf} / \pref{lem: strong graph 1})}\label{app: log K tsallis}

\begin{algorithm}[H] 
\caption{$(1-1/\log(K))$-Tsallis-Inf (for strongly observable graphs)}\label{alg: Tsallis-Inf}
    \textbf{Input}:  $\calG\setminus\hatx$.\\
    \textbf{Define}: $\psi(x)= \sum_{i=1}^K \frac{x_i^{\beta}}{\beta(1-\beta)}$, where $\beta=1-1/\log(K)$.  \\ 
    \For{$t=1,2,\ldots$}{
        Receive update probability $q_t$.\\
        Let
        \begin{align*}
            &p_{t}= \argmin_{x\in\Delta([K])}\left\{\sum_{\tau=1}^{t-1}\inner{x,\hatell_\tau}+ \frac{1}{\eta_{t}}\psi(x)\right\} \\
            &\text{where}\ \ \ \eta_{t} = \sqrt{\frac{\log K}{\sum_{s=1}^{t}\frac{1}{q_s}(1+\min\{\baralpha,\alpha\log K\}) } }. 
        \end{align*}
        Sample $A_t\sim p_t$. \\
        With probability $q_t$ receive $\ell_{t,i}$ for all $A_t\in\calN(i)$ (in this case, set $\upd_t=1$; otherwise, set $\upd_t=0$).\\
        Define 
        \begin{align*}
            &\hatell_{t,i} = \upd_t\left(\frac{(\ell_{t,i}-\ind\{i\in\calI_t\})\ind\{A_t\in\calN(i)\}}{\sum_{j\in\calN(i)}p_{t,j}q_t}+\ind\{i\in\calI_t\}\right)\,,\\
            &\text{where }\calI_t = \left\{i\in[K]: i\not\in\calN(i)\land p_{t,i}>\frac{1}{2}\right\}\,\text{ and }\calN(i)=\{j\in \calG\setminus\hatx\,|\,(j,i)\in E\}.
        \end{align*}
        
    }
\end{algorithm}
We require the following graph theoretic Lemmas.
\begin{lemma}
\label{lem: refined graph sum}
    Let $\calG=(V,E)$ be a directed graph with independence number $\alpha$ and vertex weights $p_i>0$, then 
    \begin{align*}
        \exists i\in V:\, \frac{p_i}{p_i+\sum_{j:(j,i)\in E}p_j} \leq \frac{2p_i \alpha}{\sum_{i}p_i}\,.
    \end{align*}
\end{lemma}
\begin{proof}
Without loss of generality, assume that $p\in\Delta(V)$, since the statement is scale invariant.
The statement is equivalent to
\begin{align*}
    \min_i p_i+\sum_{j:(j,i)\in E}p_j \geq \frac{1}{2\alpha}\,.
\end{align*}
 Let
\begin{align*}
    \min_{p\in\Delta(V)}\min_i \left(p_i+\sum_{j:(j,i)\in E}p_j \right)\leq \min_{p\in\Delta(V)}\sum_{i}\left(p_i^2 +\sum_{j:(j,i)\in E}p_ip_j\right)\,.
\end{align*}
Via K.K.T. conditions, there exists $\lambda \in \bbR$ such that for an optimal solution $p^\star$ it holds
\begin{align*}
    \forall i\in V:\, &\text{either \ \ }2p^\star_i +\sum_{j:(j,i)\in E}p^\star_j+\sum_{j:(i,j)\in E}p^\star_j = \lambda\\
    &\text{or\ \ }2p^\star_i +\sum_{j:(j,i)\in E}p^\star_j+\sum_{j:(i,j)\in E}p^\star_j \geq \lambda\text{ and }p^\star_i=0\,.
\end{align*}
Next we bound $\lambda$.
Take the sub-graph over $V_+=\{i:\,p^\star_i >0\}$ and take a maximally independence set $S$ over $V_+$ ($|S|\leq \alpha$, since $V_+\subset V$). We have
\begin{align*}
    1\leq\underbrace{\sum_{j\not\in V_+}p^\star_j}_{=0}+\sum_{i\in S}\left(2p^\star_i+\sum_{j:(j,i)\in E}p^\star_j+\sum_{j:(i,j)\in E}p^\star_j\right)=|S|\lambda\leq \alpha \lambda.
\end{align*}
Hence $\lambda \geq \frac{1}{\alpha}$.
Finally,  
\begin{align*}
    \sum_{i}\left((p^\star_i)^2 +\sum_{j:(j,i)\in E}p^\star_ip^\star_j\right) =\frac{1}{2}\sum_{i}\left(2(p^\star_i)^2 +\sum_{j:(j,i)\in E}p^\star_ip^\star_j+\sum_{j:(i,j)\in E}p^\star_ip^\star_j\right)\geq \sum_{i}\frac{p^\star_i\lambda}{2}\geq \frac{1}{2\alpha}\,.
\end{align*}
\end{proof}
\begin{lemma}(Lemma 10 \cite{alon2013bandits})
\label{lem: alon}
let $\calG = (V, E)$ be a directed graph. Then, for any distribution $p\in\Delta(V)$ we have:
\begin{align*}
    \sum_i \frac{p_i}{p_i+\sum_{j:(j,i)\in E}p_j} \leq \operatorname{mas}(\calG)\,,
\end{align*}
where $\operatorname{mas}(\calG)\leq \baralpha$ is the maximal acyclic sub-graph.
\end{lemma}
With these Lemmas, we are ready to proof the iw-stability.
\begin{proof}[Proof of \pref{lem: strong graph 1}]
Consider the $(1-1/\log(K))$-Tsallis-INF algorithm with adaptive stepsize (\pref{alg: Tsallis-Inf}).
Applying \pref{lem: FTRL}, we have for any stopping time $\tau$ and $a^\star\in[K]$:
\begin{align*}
    \sum_{t=1}^\tau \E_t[\inner{p_t,\hatell_t}-\hatell_{t,a^\star}] &\leq \frac{2e\log K}{\eta_\tau} + \sum_{t=1}^\tau\E_t\left[\max_p \inner{p-p_t,\hatell_t}-\frac{1}{\eta_t}D_\psi(p,p_t)\right]\\
    &= \frac{2e\log K}{\eta_\tau} + \sum_{t=1}^\tau\E_t\left[\max_p \inner{p-p_t,\hatell_t+c_t\bm{1}}-\frac{1}{\eta_t}D_\psi(p,p_t)\right]\,,
\end{align*}
where $c_t=-\upd_t\frac{(\ell_{t,\calI_t}-1)\ind\{A_t\in\calN(\calI_t)\}}{\sum_{j\in\calN(\calI_t)}p_{t,j}q_t}$, which is well defined because $\calI_t$ contains by definition maximally one arm.

We can apply \pref{lem: tsallis stab}, since the losses are strictly positive.
\begin{align*}
    \E_t\left[\max_p \inner{p-p_t,\hatell_t+c_t\bm{1}}-\frac{1}{\eta_t}D_\psi(p,p_t)\right]\leq \frac{\eta_t}{2}\E_t\left[\sum_{i}p_{t,i}^{1+\frac{1}{\log K}}(\hatell_{t,i}+c_t)^2\right]\,.
\end{align*}
We split the vertices in three sets. Let $M_1 = \{i\,|\,i\in\calN(i)\}$ the nodes with self-loops. Let $M_2 = \{i\,|\,i\not\in M_1,\not\in \calI_t\}$ and $M_3=\calI_t$.
we have
\begin{align*}
    &\frac{\eta_t}{2}\E_t\left[\sum_{i}p_{t,i}^{1+\frac{1}{\log K}}(\hatell_{t,i}+c_t)^2\right]\\
    &\leq \eta_t\E_t\left[\upd_t\left(\sum_{i\in M_1\cup M_2} P_{t,i}c_t^2+\sum_{i\in M_1}\frac{p_{t,i}^{1+\frac{1}{\log K}}}{(\sum_{j\in\calN(i)}p_{t,j}q_t)^2}+\sum_{i\in M_2}\frac{4p_{t,i}}{q_t^2}+p_{t,\calI_t} \right)\right]\,,
\end{align*}
since for any $i\in M_2$, we have $\sum_{j\in \calN(i)}p_{t,j}> \frac{1}{2}$.
The first term is in expectation $\E_t\left[\sum_{i\in M_1\cup M_2}p_{t,i}c_t^2\right]\leq \frac{1}{q_t}$, the third term is in expectation $\E_t\left[\sum_{i\in M_2}\frac{4p_{t,i}\upd_t}{q_t^2}\right]\leq \frac{4}{q_t}$, while the second is
\begin{align*}
    \E_t\left[\sum_{i\in M_1}\frac{p_{t,i}^{1+\frac{1}{\log K}}\upd_t}{(\sum_{j\in\calN(i)}p_{t,j}q_t)^2}\right] \leq \frac{2}{q_t}\sum_{i\in M_1}\frac{p_{t,i}^{1+\frac{1}{\log K}}}{\sum_{j\in\calN(i)}p_{t,j}}\,.
\end{align*}
The subgraph consisting of $M_1$ contains sef-loops, hence we can apply \pref{lem: alon} to bound this term by $\frac{2\baralpha}{q_t}$.
If $\alpha\log K  < \baralpha$, we instead take a permutation $\tildep$ of $p$, such that $M_1$ is in position $1,..,|M_1|$ and for any $i\in [|M_1|]$
\begin{align*}
    \frac{\tildep_{t,i}}{\sum_{j\in\calN(i)}\tildep_{t,j}}\leq\frac{\tildep_{t,i}}{\sum_{j\in\calN(i)\cap [i]}\tildep_{t,j}} \leq \frac{2\alpha \tildep_{t,i}}{\sum_{j=1}^i\tildep_{t,j}}\,.
\end{align*}
The existence of such a permutation is guaranteed by \pref{lem: refined graph sum}: Apply \pref{lem: refined graph sum} on the graph $M_1$ to select $\tilde p_{|M_1|}$. Recursively remove the $i$-th vertex from the graph and apply \pref{lem: refined graph sum} on the resulting sub-graph to pick $\tilde p_{i-1}$.
Applying this permutation yields
\begin{align*}
    \frac{2}{q_t}\sum_{i\in M_1}\frac{p_{t,i}^{1+\frac{1}{\log K}}}{\sum_{j\in\calN(i)}p_{t,j}}\leq \frac{4\alpha}{q_t}\sum_{i=1}^{M_1}\frac{\tildep_i^{1+\frac{1}{\log K}}}{\sum_{j=1}^i\tildep_j}\leq  \frac{4\alpha}{q_t}\sum_{i=1}^{M_1}\frac{\tildep_i}{\left(\sum_{j=1}^i\tildep_j\right)^{1-\frac{1}{\log K}}} \leq \frac{4\alpha\log K}{q_t}\,.
\end{align*}
Combining everything
\begin{align*}
    \E_t\left[\max_p \inner{p-p_t,\hatell_t+c_t\bm{1}}-\frac{1}{\eta_t}D_\psi(p,p_t)\right] \leq \frac{\eta_t}{q_t}\left(6+4\min\{\baralpha,\alpha \log K\}\right)\,.
\end{align*}
By the definition of the learning rate, we obtain
\begin{align*}
    \sum_{t=1}^\tau \E_t[\inner{p_t,\hatell_t}-\hatell_{t,a^\star}] &\leq \frac{2e\log K}{\eta_\tau} + \sum_{t=1}^\tau\frac{\eta_t}{q_t}\left(6+4\min\{\baralpha,\alpha \log K\}\right)\\
    &\leq 18\sqrt{(1+\min\{\baralpha,\alpha\log K\})\sum_{t=1}^\tau\frac{1}{q_t}\log K}\,.
\end{align*}

\end{proof}

\subsection{EXP3 for weakly observable graphs (\pref{alg: weak exp3} / \pref{lem: weak graph})}
\label{app: weakly obs exp}

\begin{algorithm}[H] 
\caption{EXP3 (for weakly observable graphs)}\label{alg: weak exp3}
    \textbf{Input}:  $\calG\setminus\hatx$, dominating set $D$ (potentially including $\hatx$).\\
    \textbf{Define}: $\psi(x)= \sum_{i=1}^K x_i\log(x_i)$. $\nu_D$ is the uniform distribution over $D$.  \\
    \For{$t=1,2,\ldots$}{
        Receive update probability $q_t$.\\
        Let
        \begin{align*}
            &P_{t}= \argmin_{x\in\Delta([K])}\left\{\sum_{\tau=1}^{t-1}\inner{x,\hatell_\tau}+ \frac{1}{\eta_{t}}\psi(x)\right\}, p_t = (1-\gamma_t)P_t+\gamma_t \nu_D\\
            &\text{where}\ \ \ \eta_{t} = \left(\left(\frac{\sqrt{\delta}\sum_{s=1}^t\frac{1}{\sqrt{q_s}}}{\log(K)}\right)^\frac{2}{3}+\frac{4\delta}{\min_{s\leq t}q_s}\right)^{-1} \text{ and }\gamma_t=\sqrt{\frac{\eta_t\delta}{q_t}}.
        \end{align*}
        Sample $A_t\sim p_t$. \\
        With probability $q_t$ receive $\ell_{t,i}$ for all $A_t\in\calN(i)$ (in this case, set $\upd_t=1$; otherwise, set $\upd_t=0$).\\
        Define 
        \begin{align*}
            &\hatell_{t,i} = \frac{\ell_{t,i}\upd_t\ind\{A_t\in\calN(i)\}}{\sum_{j\in\calN(i)}p_{t,j}q_t}.
        \end{align*}
        
    }
\end{algorithm}

\begin{proof}[Proof of \pref{lem: weak graph}]
Consider the EXP3 algorithm (\pref{alg: weak exp3}). By standard analysis of FTRL (\pref{lem: FTRL}) and \pref{lem: exp3 stab} by the non-negativity of all losses, for any $\tau$ and any $a^\star$,     \begin{align*}
    &\sum_{t=1}^\tau \E_t\left[\sum_{a} P_t(a)\hatell_t(a)\right] - \sum_{t=1}^\tau \E_t\left[\hatell_t (a^\star) \right] \\
    &\leq \frac{\log K}{\eta_\tau} + \sum_{t=1}^\tau\frac{\eta_t}{2}\E_t\left[\sum_{i}P_{t,i}\hatell_{t,i}^2\right] \\
    &\leq
    \frac{\log K}{\eta_\tau} + \sum_{t=1}^\tau\frac{\eta_t}{2}\sum_i P_{t,i}\frac{1}{\sum_{j\in\calN(i)}p_{t,j}q_t}\\
    &\leq
    \frac{\log K}{\eta_\tau} + \sum_{t=1}^\tau\sqrt{\frac{\eta_t\delta}{4q_t}}\,
\end{align*}
where in the last inequality we use $\frac{\eta_t}{2} \sum_i P_{t,i}\times \frac{\delta}{\gamma_t}\times \frac{1}{q_t}= \frac{1}{2}\sqrt{\frac{\eta_t\delta}{q_t}}$. 
We have due to $\E_t[\inner{P_t-p_t,\hatell_t}]\leq \gamma_t$ 
\begin{align*}
   \sum_{t=1}^\tau \E_t\left[\sum_{a} p_t(a)\hatell_t(a)\right] - \sum_{t=1}^\tau \E_t\left[\hatell_t (a^\star) \right]&\leq \frac{\log K}{\eta_\tau} + \sum_{t=1}^\tau\sqrt{9\frac{\eta_t\delta}{4q_t}}\\
    &\leq 6(\delta\log(K))^\frac{1}{3}\left(\sum_{t=1}^\tau\frac{1}{\sqrt{q_t}}\right)^\frac{2}{3}+\frac{4\delta\log(K)}{\min_{t\leq T}q_t}.
\end{align*}
Taking expectation on both sides finishes the proof.
\end{proof}

\section{Surrogate Loss for Strongly Observable Graph Problems}
\label{sec: no self loop}
When there exist arms such that $i\not\in\calN(i)$, we cannot directly apply \pref{alg: corral}.
To make the algorithm applicable to all strongly observable graphs,
 define the surrogate losses $\tilde\ell_t$ in the following way:

\begin{align*}
&\tilde\ell_{t,\hatx} = \ell_{t,\hatx}\ind\{(\hatx,\hatx)\in E\} - \sum_{j\in[K]\setminus\{\hatx\}}p_{t,j}\ell_{t,j}\ind\{(j,j)\not\in E\}\\
\forall j\in[K]\setminus\{\hatx\}:\,&\tilde\ell_{t,j} = \ell_{t,j}\ind\{(j,j)\in E\} - \ell_{t,\hatx}\ind\{(\hatx,\hatx)\not\in E\}\,,
\end{align*}
where $p_t$ is the distribution of the base algorithm $\calB$ over $[K]\setminus\{\hatx\}$ at round $t$.
By construction and the definition of strongly observable graphs, $\tilde\ell_{t,j}$ is observed when playing arm $j$. (When the player does not have access to $p_t$, one can also sample one action from the current distribution for an unbiased estimate of $\ell_{t,\hatx}$.)
The losses $\tilde\ell_t$ are in range $[-1,1]$ instead of $[0,1]$ and can further be shifted to be strictly non-negative.
Finally observe that
\begin{align*}
     \E[\tilde\ell_{t,A_t}-\tilde\ell_{t,\hatx}]&=q_{t,2}\left(\sum_{j\in[K]\setminus\{\hatx\}}p_{t,j}\tilde\ell_{t,j}-\tilde\ell_{t,\hatx}\right)\\
     &=q_{t,2}\left(\sum_{j\in[K]\setminus\{\hatx\}}p_{t,j}\ell_{t,j}-\ell_{t,\hatx}\right)\\
     &= \E[\ell_{t,A_t}-\ell_{t,\hatx}],
\end{align*}
as well as
\begin{align*}
    \E\left[\tilde\ell_{t,A_t}-\sum_{j\in[K]\setminus\{\hatx\}}p_{t,j}\tilde\ell_{t,j}\right]&=q_{t,1}\left(\tilde\ell_{t,\hatx}-\sum_{j\in[K]\setminus\{\hatx\}}p_{t,j}\tilde\ell_{t,j}\right)\\
     &=q_{t,1}\left(\tilde\ell_{t,\hatx}-\sum_{j\in[K]\setminus\{\hatx\}}p_{t,j}\ell_{t,j}\right)\\
     &= \E\left[\ell_{t,A_t}-\sum_{j\in[K]\setminus\{\hatx\}}p_{t,j}\ell_{t,j}\right]\,.
\end{align*}
That means running \pref{alg: corral}, replacing $\ell_t$ by $\tilde\ell_t$ allows to apply \pref{thm: basic corral thm} to strongly observable graphs where not every arm receives its own loss as feedback.

\section{Tabular MDP (\pref{thm: MDP main theorem bound})}\label{app: tabular MDP}
In this section, we consider using the UOB-Log-barrier Policy Search algorithm in \cite{lee2020bias} (their Algorithm 4) as our base algorithm. To this end, we need to show that it satisfies a  dd-strongly-iw-stable condition specified in \pref{lem: main lemma for MDP}. We consider a variant of their algorithm by incorporating the feedback probability $q_t'=q_t+(1-q_t)\ind[\pi_t=\hatpi]$. The algorithm is \pref{alg: log-barrier MDP}. 
\begin{algorithm}
    \caption{UOB-Log-Barrier Policy Search} \label{alg: log-barrier MDP}
    \textbf{Input}: state space $\calS$, action space $\calA$, candidate policy $\hatpi$.  \\
    \textbf{Define}: $
        \Omega = \left\{\hatw:~ \hatw(s,a,s')\geq \frac{1}{T^3S^2A}\right\}$, $\psi(w) = \sum_h \sum_{(s,a,s')\in\calS_h\times\calA\times \calS_{h+1}}  \ln\frac{1}{w(s,a,s')}$. \\ $\delta=\frac{1}{T^5S^3A}$.  \\
    \textbf{Initialization}: 
    \begin{align*}
        \hatw_1(s,a,s') = \frac{1}{|\calS_h||\calA||\calS_{h+1}|}, \qquad \pi_1 = \pi^{\hatw_1}, \qquad t^\star\leftarrow 1. 
    \end{align*}
    \For{$t=1,2,\ldots$}{
        If $\pi_t=\hatpi$, $\upd_t=1$; otherwise, $\upd_t=1$ w.p. $q_t$ and $\upd_t=0$ w.p. $1-q_t$. 
        
        If $\upd_t=1$, obtain the trajectory $s_h, a_h, \ell_t(s_h,a_h)$ for all $h=1,\ldots, H$. \\
        Construct loss estimators 
        \begin{align*}
            \hatell_t(s,a) = \frac{\upd_t}{q_t'}\cdot\frac{\ell_t(s,a)\ind_t(s,a)}{\phi_t(s,a)},\quad \text{where}\ \ind_t(s,a)=\ind\{(s_{h(s)},a_{h(s)})=(s,a)\} 
        \end{align*}
        where $q_t'=q_t+(1-q_t)\one[\pi_t=\hatpi]$. \\ 
        Update counters: for all $s,a,s'$, 
        \begin{align*}
            N_{t+1}(s, a)\leftarrow N_t(s,a)+\ind_t(s,a), \quad N_{t+1}(s, a, s')\leftarrow N_t(s,a,s') + \ind_t(s,a)\ind\{s_{h(s)+1}=s'\}. 
        \end{align*}
            Compute confidence set 
            \begin{align*}
                \calP_{t+1} = \left\{\hat{P}:~\left| \hat{P}(s'|s,a) - \overline{P}_{t+1}(s'|s,a) \right|\leq \epsilon_{t+1}(s'|s,a), \ \ \forall (s,a,s')\right\}
            \end{align*}
            where $\overline{P}_{t+1}(s'|s,a)=\frac{N_{t+1}(s,a,s')}{\max\{1, N_{t+1}(s,a)\}}$ and 
            \begin{align*}
                \epsilon_{t+1}(s'|s,a)=4\sqrt{\frac{\overline{P}_{t+1}(s'|s,a)\ln(SAT/\delta)}{\max\{1, N_{t+1}(s,a)\}}} + \frac{28\ln(SAT/\delta)}{3\max\{1, N_{t+1}(s,a)\}}
            \end{align*}
        \If{ $\sum_{\tau=t^\star}^{t}\sum_{s,a} \frac{\upd_\tau \ind_\tau(s,a)\ell_\tau(s,a)^2}{q_t'^2} + \max_{\tau\leq t} \frac{H}{q_\tau'^2} \geq \frac{S^2A\ln(SAT)}{\eta_t^2}$
        }{
            $\eta_{t+1}\leftarrow \frac{\eta_t}{2}$ \\
            $\hatw_{t+1}=\argmin_{w\in\Delta(\calP_{t+1})\cap \Omega}  \psi(w)$ \\
            $t^\star\leftarrow t+1$. 
        }
        \Else{
            $\eta_{t+1}=\eta_t$ \\
            $\hatw_{t+1} = \argmin_{w\in\Delta(\calP_{t+1})\cap \Omega} \left\{\inner{w, \hatell_t} + \frac{1}{\eta_{t}}D_{\psi}(w,\hatw_t)\right\}$. 
        }
        Update policy $\pi_{t+1}=\pi^{\hatw_{t+1}}$. 
    }
\end{algorithm}

We refer the readers to Section C.3 of \cite{lee2020bias} for the setting descriptions and notations, since we will follow them tightly in this section.   

\subsection{Base algorithm}
\begin{lemma}\label{lem: base algorithm for MDP}
    \pref{alg: log-barrier MDP} ensures for any $u^\star$, 
    \begin{align*}
        \E\left[\sum_{t=1}^T \inner{w_t - u^\star ,\ell_t} \right] \leq O\left(\sqrt{HS^2A\iota^2\E\left[\sum_{t=1}^T \sum_{s,a}\frac{\upd_t\ind_t(s,a)\ell_t(s,a)}{q_t'^2}\right]} + \E\left[\frac{S^5A^2\iota^2}{\min_t q_t'}\right]\right). 
    \end{align*}
\end{lemma}
\begin{proof}
The regret is decomposed as the following (similar to Section C.3 in \cite{lee2020bias}): 
\begin{align*}
    \sum_{t=1}^T \inner{w_t - u^\star, \ell_t} &= \underbrace{\sum_{t=1}^T \inner{w_t - \hatw_t, \ell_t}}_{\text{Error}} + \underbrace{\sum_{t=1}^T \inner{\hatw_t, \ell_t-\hatell_t}}_{\text{Bias-1}} + \underbrace{\sum_{t=1}^T \inner{\hatw_t-u,\hatell_t}}_{\text{Reg-term}} \\
    &\quad + \underbrace{\sum_{t=1}^T \inner{u,\hatell_t-\ell_t}}_{\text{Bias-2}} + \underbrace{\sum_{t=1}^T \inner{u-u^\star, \ell_t}}_{\text{Bias-3}}, 
\end{align*}
with the same definition of $u$ as in (24) of \cite{lee2020bias}. $\text{Bias-3} \leq H$ trivially (see (25) of \cite{lee2020bias}). For the remaining four terms, we use \pref{lem: error term}, \pref{lem: bias 1}, \pref{lem: bias 2}, and \pref{lem: regterm}, respectively. Combining all terms finishes the proof. 
\end{proof}

\begin{lemma}[Lemma C.2 of \cite{lee2020bias}] 
    With probability $1-O(\delta)$, for all $t,s,a,s'$, 
    \begin{align*}
        \left|P(s'|s,a) - \overline{P}_t(s'|s,a)\right|\leq \frac{\epsilon_t(s'|s,a)}{2}. 
    \end{align*}
\end{lemma}

\begin{lemma}[\textit{cf.} Lemma C.3 of \cite{lee2020bias}]\label{lem: concentration bounds }
    With probability at least $1-\delta$, for all $h$, 
    \begin{align*}
        \sum_{t=1}^T \sum_{s\in\calS_h, a\in\calA} \frac{q_t' \cdot w_t(s,a)}{\max\{1,N_t(s,a)\}} 
        &= O\left(|\calS_h|A\log(T) + \ln(H/\delta)\right)
    \end{align*}
\end{lemma}
\begin{proof}
    The proof of this lemma is identical to the original one --- only need to notice that in our case, the probability of obtaining a sample of $(s,a)$ is $q_t'\cdot w_t(s,a)$.  
\end{proof}

\begin{lemma}[\textit{cf.} Lemma C.6 of \cite{lee2020bias}]\label{lem: general lemma MDP}
   With probability at least $1-O(\delta)$, for any $t$ and any collection of transition functions $\{P^s_t\}_{s\in\calS}$ such that $P^s_t\in\calP_{t}$ for all $s$, we have 
   \begin{align*}
       &\sum_{t=1}^T \sum_{s\in\calS, a\in\calA} \left| w^{P^s_t, \pi_t}(s,a) - w_t(s,a) \right|\ell_t(s,a) = O\left(S\sqrt{HA \sum_{t=1}^T \frac{\inner{w_t,\ell_t^2}}{q_t'} \iota^2} + \frac{S^5A^2 \iota^2}{\min_t q_t'}\right). 
   \end{align*}
   where $\iota\triangleq \log(SAT/\delta)$ and $\ell_t^2(s,a)\triangleq (\ell_t(s,a))^2$.  
\end{lemma}
\begin{proof}
    Following the proof of Lemma C.6 in \cite{lee2020bias}, we have 
    \begin{align*}
        \sum_{t=1}^T \sum_{s\in\calS, a\in\calA} \left| w^{P^s_t, \pi_t}(s,a) - w_t(s,a) \right|\ell_t(s,a)\leq B_1 + SB_2
    \end{align*}
    where 
    \begin{align}
        B_2 &\leq O\left( \sum_{t=1}^T  \sum_{s,a,s'}\sum_{x,y,x'} \sqrt{\frac{P(s'|s,a)\iota}{\max\{1,N_{t}(s,a)\}}}w_t(s,a)\sqrt{\frac{P(x'|x,y)\iota}{\max\{1,N_{t}(x,y)\}}}w_t(x,y|s')\right) \nonumber \\
        &\qquad + O\left(\sum_{t=1}^T \sum_{s,a,s'}\sum_{x,y,x'} \frac{w_t(s,a)\iota}{\max\{1,N_{t}(s,a)\}}\right)\label{eq: complicated term} 
    \end{align}
    The first term in \eqref{eq: complicated term} can be upper bounded by the order of 
    \begin{align*}
        &\sum_{t=1}^T \sum_{s,a,s'}\sum_{x,y,x'}\sqrt{\frac{w_t(s,a)P(x'|x,y)w_t(x,y|s')\iota}{\max\{1,N_{t}(s,a)\}}}\sqrt{\frac{w_t(s,a) P(s'|s,a)w_t(x,y|s')\iota}{\max\{1,N_{t}(x,y)\}}}\\
        &\leq \sum_{t=1}^T \sqrt{\sum_{s,a,s'}\sum_{x,y,x'}\frac{w_t(s,a)P(x'|x,y)w_t(x,y|s')\iota}{\max\{1,N_{t}(s,a)\}}}\sqrt{\sum_{s,a,s'}\sum_{x,y,x'}\frac{w_t(s,a) P(s'|s,a)w_t(x,y|s')\iota}{\max\{1,N_{t}(x,y)\}}} \\
        &\leq \sum_{t=1}^T \sqrt{H\sum_{s,a,s'}\frac{w_t(s,a)\iota}{\max\{1,N_t(s,a)\}}} \sqrt{H\sum_{x,y,x'}\frac{w_t(x,y)\iota}{\max\{1,N_t(x,y)\}}} \\
        &\leq HS \sum_{t=1}^T \sum_{s,a,s'}\frac{w_t(s,a)\iota}{\max\{1, N_t(s,a)\}} \\
        &\leq O\left(\frac{HS^2\iota}{\min_{t}q_t'}\left(SA\ln(T)+H\log(H/\delta)\right)\right).   \tag{by \pref{lem: concentration bounds }}
    \end{align*} 
    The second term in \eqref{eq: complicated term} can be upper bounded by the order of 
    \begin{align*}
        S^2A \sum_{t=1}^T \sum_{s,a,s'}\frac{w_t(s,a)\iota}{\max\{1, N_t(s,a)\}} = O\left(\frac{S^3A\iota}{\min_{t}q_t'}\left(SA\ln(T)+H\log(H/\delta)\right)\right).  \tag{by \pref{lem: concentration bounds }}
    \end{align*}
    Combining the two parts, we get 
    \begin{align*}
        B_2 \leq O\left(\frac{1}{\min_{t}q_t'}\left(S^4A^2\ln(T)\iota+HS^3A\iota\log(H/\delta)\right)\right)\leq O\left(\frac{S^4A^2\iota^2}{\min_t q_t'}\right). 
    \end{align*}
    Next, we bound $B_1$. By the same calculation as in Lemma C.6 of \cite{lee2020bias}, 
    \begin{align*}
        &B_1 \\
        &\leq O\left( \sum_{t=1}^T \sum_{s,a,s'} \sum_{x,y} w_t(s,a)\sqrt{\frac{P(s'|s,a)\iota}{\max\{1,N_{t}(s,a)\}}}w_t(x,y|s')\ell_t(x,y) + H\sum_{t=1}^T \sum_{s,a,s'}\frac{w_t(x,a)\iota}{\max\{1, N_{t}(s,a)\}}\right) \\
        &\leq O\left(\sum_{t=1}^T \sum_{s,a,s'} \sum_{x,y} w_t(s,a)\sqrt{\frac{P(s'|s,a)\iota}{\max\{1,N_{t}(s,a)\}}}w_t(x,y|s')\ell_t(x,y) + \frac{HS^2A\iota^2}{\min_t q_t'}\right) 
    \end{align*}
    For the first term above, we consider the summation over $(s,a,s')\in\calT_h\triangleq \calS_h\times \calA\times \calS_{h+1}$.  We continue to bound it by the order of
    \begin{align*}
        & \sum_{t=1}^T \sum_{s,a,s'\in\calT_h} \sum_{x,y} w_t(s,a)\sqrt{\frac{P(s'|s,a)\iota}{\max\{1,N_{t}(s,a)\}}}w_t(x,y|s')\ell_t(x,y) \\
        &\leq \alpha\sum_{t=1}^T \frac{1}{q_t'} \sum_{s,a,s'\in\calT_h}\sum_{x,y}w_t(s,a)P(s'|s,a)w_t(x,y|s')\ell_t(x,y)^2 \\
        &\qquad + \frac{1}{\alpha}
        \sum_{t=1}^T  \sum_{s,a,s'\in\calT_h}\sum_{x,y} q_t'w_t(s,a)w_t(x,y|s')\cdot \frac{\iota}{\max\{1,N_{t}(s,a)\}} \tag{by AM-GM, holds for any $\alpha>0$}\\
        &\leq \alpha \sum_{t=1}^T \frac{1}{q_t'} \sum_{x,y} w_t(x,y)\ell_t(x,y)^2 + \frac{H|\calS_{h+1}|}{\alpha} \sum_{t=1}^T \sum_{s,a\in\calS_h\times \calA}\frac{q_t'  w_t(s,a)\iota}{N_{t}(s,a)} \\
        &\leq \alpha \sum_{t=1}^T \frac{\inner{w_t,\ell_t^2}}{q_t'} + \frac{H|\calS_{h+1}||\calS_h| A\iota^2}{\alpha} \tag{by \pref{lem: concentration bounds }}\\
        &\leq O\left(\sqrt{H|\calS_h||\calS_{h+1}|A \sum_{t=1}^T \frac{\inner{w_t,\ell_t^2}\iota^2}{q_t'}}\right)
        \tag{choose the optimal $\alpha$}\\
        &\leq O\left((|\calS_h|+|\calS_{h+1}|)\sqrt{HA \sum_{t=1}^T \frac{\inner{w_t,\ell_t^2}\iota^2}{q_t'}}\right)
    \end{align*}
    Therefore, 
    \begin{align*}
        B_1 
        &\leq O\left(\sum_h (|\calS_h|+|\calS_{h+1}|)\sqrt{HA \sum_{t=1}^T \frac{\inner{w_t,\ell_t^2}\iota^2}{q_t'}}+ \frac{HS^2A\iota^2}{\min_t q_t'}\right) \\
        &\leq O\left(S\sqrt{HA \sum_{t=1}^T \frac{\inner{w_t,\ell_t^2}\iota^2}{q_t'}}+ \frac{HS^2A\iota^2}{\min_t q_t'}\right). 
    \end{align*}
    Combining the bounds finishes the proof. 
\end{proof}

\begin{lemma}[\textit{cf.} Lemma C.7 of \cite{lee2020bias}]\label{lem: error term}
    \begin{align*}
        \E\left[\textup{Error}\right]=\E\left[\sum_{t=1}^T \inner{\hatw_t-w_t, \ell_t}\right]\leq O\left(\E\left[S\sqrt{HA \sum_{t=1}^T \frac{\inner{w_t,\ell_t^2}\iota^2}{q_t'}} + \frac{S^5A^2\iota^2}{\min_t q_t'}\right]\right). 
    \end{align*}
    
\end{lemma}
\begin{proof}
By \pref{lem: general lemma MDP}, with probability at least $1-O(\delta)$,  
\begin{align*}
        \text{Error} 
        &= \sum_{t=1}^T \inner{\hatw_t-w_t, \ell_t} \leq \sum_{t=1}^T \sum_{s,a} |\hatw_t(s,a)-w_t(s,a)|\ell_t(s,a) \\
        &\leq O\left(S\sqrt{HA \sum_{t=1}^T \frac{\inner{w_t,\ell_t^2}\iota^2}{q_t'}} + \frac{S^5A^2\iota^2}{\min_t q_t'}\right) 
    \end{align*}
    Furthermore, $|\textup{Error}|\leq O(SAT)$ with probability $1$. Therefore, 
    \begin{align*}
        \E[\textup{Error}] 
        &\leq O\left(\E\left[S\sqrt{HA \sum_{t=1}^T \frac{\inner{w_t,\ell_t^2}\iota^2}{q_t'}} + \frac{S^5A^2\iota^2}{\min_t q_t'}\right] + \delta SAT\right) \\
        &\leq O\left(\E\left[S\sqrt{HA \sum_{t=1}^T \frac{\inner{w_t,\ell_t^2}\iota^2}{q_t'}} + \frac{S^5A^2\iota^2}{\min_t q_t'} \right]\right)
    \end{align*}
    by our choice of $\delta$. 
\end{proof}

\begin{lemma}[\textit{cf.} Lemma C.8 of \cite{lee2020bias}]\label{lem: bias 1}
    \begin{align*}
        \E[\textup{Bias-1}] 
        &= \E\left[\sum_{t=1}^T \inner{\hatw_t, \ell_t-\hatell_t} \right] \leq O\left(\E\left[S\sqrt{HA \sum_{t=1}^T \frac{\inner{w_t,\ell_t}}{q_t'}} + \frac{S^5A^2\iota^2}{\min_t q_t'}\right]\right). 
    \end{align*}
\end{lemma}

\begin{proof}
    Let $\calE_t$ be the event that $P\in \calP_\tau$ for all $\tau\leq t$. 
    \begin{align*}
        \E_t\left[ \inner{\hatw_t, \ell_t - \hatell_t} ~\bigg|~\calE_t \right] 
        &= \sum_{s,a} \hatw_t(s,a)\ell_t(s,a)\left(1-\frac{w_t(s,a)}{\phi_t(s,a)}\right) \\
        &\leq \sum_{s,a} |\phi_t(s,a) - w_t(s,a)|\ell_t(s,a)
    \end{align*}
    Thus, 
    \begin{align*}
        \E_t\left[ \inner{\hatw_t, \ell_t - \hatell_t} \right] 
        &\leq \sum_{s,a} |\phi_t(s,a) - w_t(s,a)|\ell_t(s,a) + O(H\ind[\overline{\calE}_t])  
    \end{align*}
    Summing this over $t$, 
    \begin{align*}
        \sum_{t=1}^T \E_t\left[ \inner{\hatw_t, \ell_t - \hatell_t} \right]  \leq \underbrace{\sum_{t=1}^T \sum_{s,a} |\phi_t(s,a) - w_t(s,a)|\ell_t(s,a)}_{(\star)} + O\left(H\sum_{t=1}^T \ind[\overline{\calE}_t]\right)  
    \end{align*}
    By \pref{lem: general lemma MDP}, $(\star)$ is upper bounded by  $O\left(S\sqrt{HA \sum_{t=1}^T \frac{\inner{w_t,\ell_t^2}}{q_t'} \iota^2} + \frac{S^5A^2 \iota^2}{\min_t q_t'}\right)$ with probability $1-O(\delta)$. Taking expectations on both sides and using that $\Pr[\overline{\calE}_t]\leq \delta$, we get 
    \begin{align*}
        \E\left[\sum_{t=1}^T  \inner{\hatw_t, \ell_t - \hatell_t} \right] 
        &\leq  O\left(\E\left[S\sqrt{HA \sum_{t=1}^T \frac{\inner{w_t,\ell_t^2}}{q_t'} \iota^2} + \frac{S^5A^2 \iota^2}{\min_t q_t'}\right]\right) + O\left(\delta SAT\right) \\
        &\leq O\left(\E\left[S\sqrt{HA \sum_{t=1}^T \frac{\inner{w_t,\ell_t^2}}{q_t'} \iota^2} + \frac{S^5A^2 \iota^2}{\min_t q_t'}\right]\right). 
    \end{align*}
    
\end{proof}

\begin{lemma}[\textit{cf.}  Lemma C.9 of \cite{lee2020bias}]\label{lem: bias 2}
    \begin{align*}
        \E[\textup{Bias-2}] = \E\left[\sum_{t=1}^T \inner{u,\hatell_t-\ell_t}\right]\leq O(1). 
    \end{align*}
\end{lemma}
\begin{proof}
     Let $\calE_t$ be the event that $P\in \calP_\tau$ for all $\tau\leq t$. By the construction of the loss estimator, we have 
     \begin{align*}
         \E_t\left[\inner{u, \hatell_t - \ell_t}~\big|~\calE_t\right]\leq 0
     \end{align*}
     and thus 
     \begin{align*}
         \E_t\left[\inner{u, \hatell_t - \ell_t}\right]\leq \ind[\overline{\calE}_t]\cdot H\cdot T^3S^2A \tag{by Lemma C.5 of \cite{lee2020bias}, $\hatell_t(s,a)\leq T^3S^2A$}
     \end{align*}
     and 
     \begin{align*}
         \E\left[\sum_{t=1}^T \inner{u, \hatell_t-\ell_t} \right]\leq HT^4S^2A \E\left[\sum_{t=1}^T \ind[\overline{\calE}_t] \right]\leq O(\delta HT^5S^2A) \leq O(1), 
     \end{align*}
     where the last inequality is by our choice of $\delta$. 
\end{proof}

\begin{lemma}[\textit{cf. } Lemma C.10 of \cite{lee2020bias}]\label{lem: regterm}
    \begin{align*}
        \E[\textup{Reg-term}]\leq O\left(\sqrt{S^2A\ln(SAT)\E\left[\sum_{t=1}^T \sum_{s,a}\frac{\upd_t\ind_t(s,a)\ell_t(s,a)^2}{q_t'^2} + \max_{t} \frac{H}{q_t'^2} \right]} \right). 
    \end{align*}
\end{lemma}
\begin{proof}
    By the same calculation as in the proof of Lemma C.10 in \cite{lee2020bias}, 
    \begin{align*}
        \inner{\hatw_t - u, \hatell_t} 
        &\leq \frac{D_{\psi}(u,\hatw_t) - D_{\psi}(u,\hatw_{t+1})}{\eta_t}  +  \eta_t \sum_{s,a} \hatw_t(s,a)^2 \hatell_t(s,a)^2 \\
        &\leq \frac{D_{\psi}(u,\hatw_t) - D_{\psi}(u,\hatw_{t+1})}{\eta_t}  +  \eta_t \sum_{s,a} \frac{\upd_t \ind_t(s,a)\ell_t(s,a)^2  }{q_t'^2}. 
    \end{align*}
    Let $t_1, t_2, \ldots$ be the time indices when $\eta_t=\frac{\eta_{t-1}}{2}$, and let $t_{i^\star}$ be the last time this happens. Summing the inequalities above and using telescoping, we get 
    \begin{align*}
        \sum_{t=1}^T  \inner{\hatw_t-u,\hatell_t} 
        &\leq \sum_{i} \left[\frac{D_{\psi}(u,\hatw_{t_i})}{\eta_{t_i}} + \eta_{t_i}\sum_{t=t_i}^{t_{i+1}-1}\sum_{s,a} \frac{\upd_t\ind_t(s,a)\ell_t(s,a)^2}{q_t'^2} \right] \\
        &\leq \sum_i \left[\frac{O(S^2A\ln(SAT))}{\eta_{t_i}} + \eta_{t_i}\sum_{t=t_i}^{t_{i+1}-1}\sum_{s,a} \frac{\upd_t\ind_t(s,a)\ell_t(s,a)^2}{q_t'^2} \right]  \tag{computed in the proof of Lemma C.10 in \cite{lee2020bias}}\\
        &\leq\sum_i \frac{O(S^2A\ln(SAT))}{\eta_{t_i}}  \tag{by the timing we halve the learning rate}\\
        &= O\left(\frac{S^2A\ln(SAT)}{\eta_{t_{i^\star}}}\right) \\
        &\leq O\left(\sqrt{S^2A\ln(SAT)\left(\sum_{t=1}^T \sum_{s,a}\frac{\upd_t\ind_t(s,a)\ell_t(s,a)^2}{q_t'^2} + \max_{t} \frac{H}{q_t'^2} \right)}\right). 
    \end{align*}
    
\end{proof}

\subsection{Corraling}

We use \pref{alg: corral MDP} as the corral algorithm for the MDP setting, with \pref{alg: log-barrier MDP} being the base algorithm. The guarantee of \pref{alg: corral MDP} is provided in \pref{lem: main lemma for MDP}. 

\begin{proof}[Proof of \pref{thm: MDP main theorem bound}] 
   To apply \pref{lem: main lemma for MDP}, we have to perform re-scaling on the loss because for MDPs, the loss of a policy in one round is $H$. Therefore, scale down all losses by a factor of $\frac{1}{H}$. Then by \pref{lem: base algorithm for MDP}, the base algorithm satisfies the condition in \pref{lem: main lemma for MDP} with $c_1 = S^2A\iota^2$ and $c_2=S^5A^2\iota^2/H$ where $\xi_{t,\pi}$ is defined as $\ell_{t,\pi}'=\ell_{t,\pi}/H$, the expected loss of policy $\pi$ after scaling. Therefore, by \pref{lem: main lemma for MDP}, the we can transform it to an  $\frac{1}{2}$-dd-LSB algorithm with $c_1 = S^2A\iota^2$ and $c_2=S^5A^2\iota^2/H$. Finally, using \pref{thm: data-dependent reduction } and scaling back the loss range, we can get 
   \begin{align*}
       O\left(\sqrt{S^2A  HL_\star\log^2(T)\iota^2} + S^5A^2\log^2(T)\iota^2\right)
   \end{align*}
   regret in the adversarial regime, and 
   \begin{align*}
       O\left(\frac{H^2S^2A\iota^2\log T}{\Delta} + \sqrt{\frac{H^2S^2A\iota^2\log T}{\Delta}C} + S^5A^2\iota^2\log(T)\log(C\Delta^{-1})\right)
   \end{align*}
   regret in the corrupted stochastic regime. 
\end{proof}

\end{document}